%% file: main.tex
\pdfoutput=1
\documentclass[letterpaper]{article}
\usepackage[totalwidth=480pt, totalheight=680pt]{geometry}

\usepackage{algorithm}
\usepackage{algorithmicx}
\usepackage[noend]{algpseudocode}
\usepackage{subcaption}
\usepackage{listings}
\usepackage{enumerate}
\usepackage{bm}
\usepackage{amsmath}
\usepackage{amssymb}
\usepackage{amsthm}
\usepackage{pifont}
\usepackage{booktabs}
\usepackage{xcolor}
\definecolor{darkgreen}{rgb}{0,0.5,0}
\usepackage{hyperref}
\hypersetup{
    unicode=false,          
    colorlinks=true,        
    linkcolor=red,          
    citecolor=darkgreen,    
    filecolor=magenta,      
    urlcolor=blue           
}
\usepackage[capitalize, nameinlink]{cleveref}

\usepackage{thm-restate}
\theoremstyle{plain}
\newtheorem{theorem}{Theorem}

\newtheorem{lemma}[theorem]{Lemma}
\newtheorem{fact}[theorem]{Fact}

\theoremstyle{definition}
\newtheorem{definition}[theorem]{Definition}

\theoremstyle{remark}
\newtheorem{remark}[theorem]{Remark}

\usepackage{tikz}
\usetikzlibrary{calc, graphs, graphs.standard, shapes, arrows, arrows.meta, positioning, decorations.pathreplacing, decorations.markings, decorations.pathmorphing, fit, matrix, patterns, shapes.misc, tikzmark}

\newcommand{\ba}{\mathbf{a}}
\newcommand{\bb}{\mathbf{b}}

\newcommand{\bg}{\mathbf{g}}
\newcommand{\bu}{\mathbf{u}}
\newcommand{\bv}{\mathbf{v}}
\newcommand{\bx}{\mathbf{x}}
\newcommand{\by}{\mathbf{y}}
\newcommand{\bz}{\mathbf{z}}

\newcommand{\bA}{\mathbf{A}}
\newcommand{\bB}{\mathbf{B}}
\newcommand{\bC}{\mathbf{C}}
\newcommand{\bD}{\mathbf{D}}

\newcommand{\bI}{\mathbf{I}}

\newcommand{\bM}{\mathbf{M}}

\newcommand{\bS}{\mathbf{S}}
\newcommand{\bT}{\mathbf{T}}
\newcommand{\bU}{\mathbf{U}}
\newcommand{\bV}{\mathbf{V}}

\newcommand{\bX}{\mathbf{X}}
\newcommand{\bY}{\mathbf{Y}}

\newcommand{\cB}{\mathcal{B}}

\newcommand{\cE}{\mathcal{E}}
\newcommand{\cF}{\mathcal{F}}

\newcommand{\cM}{\mathcal{M}}

\newcommand{\cO}{\mathcal{O}}
\newcommand{\cP}{\mathcal{P}}
\newcommand{\cQ}{\mathcal{Q}}

\newcommand{\cS}{\mathcal{S}}
\newcommand{\cT}{\mathcal{T}}

\newcommand{\eps}{\varepsilon}

\newcommand{\R}{\mathbb{R}}
\newcommand{\wt}{\widetilde}
\newcommand{\wh}{\widehat}
\newcommand{\tv}{\mathrm{d_{TV}}}
\newcommand{\kl}{\mathrm{d_{KL}}}
\newcommand{\Tr}{\mathrm{Tr}}
\newcommand{\poly}{\mathrm{poly}}

\DeclareMathOperator*{\argmin}{argmin}
\renewcommand{\vec}{\mathrm{vec}}
\DeclareMathOperator{\mat}{\mathrm{mat}}
\newcommand{\diag}{\mathrm{diag}}
\DeclareMathOperator{\E}{\mathbb{E}} 
\DeclareMathOperator{\Var}{\mathsf{Var}} 
\DeclareMathOperator{\Fail}{\mathsf{Fail}}
\DeclareMathOperator{\Accept}{\mathsf{Accept}}
\DeclareMathOperator{\maj}{\mathsf{majority}}
\DeclareMathOperator{\Reject}{\mathsf{Reject}}
\DeclareMathOperator{\OK}{\mathsf{OK}}

\allowdisplaybreaks

\title{Learning multivariate Gaussians with imperfect advice}
\author{
Arnab Bhattacharyya\thanks{Part of work done while the authors were affiliated with the National University of Singapore, Singapore.}\\
University of Warwick\\
\texttt{arnab.bhattacharyya@warwick.ac.uk}
\and
Davin Choo$^*$\\
Harvard University\\
\texttt{davinchoo@seas.harvard.edu}
\and
Philips George John\\
CNRS-CREATE \& National University of Singapore\\
\texttt{philips.george.john@u.nus.edu}
\and
Themis Gouleakis$^*$\\
Nanyang Technological University\\
\texttt{themis.gouleakis@ntu.edu.sg}
}
\date{}

\begin{document}

\maketitle

\begin{abstract}
We revisit the problem of distribution learning within the framework of learning-augmented algorithms.
In this setting, we explore the scenario where a probability distribution is provided as potentially inaccurate advice on the true, unknown distribution.
Our objective is to develop learning algorithms whose sample complexity decreases as the quality of the advice improves, thereby surpassing standard learning lower bounds when the advice is sufficiently accurate.

Specifically, we demonstrate that this outcome is achievable for the problem of learning a multivariate Gaussian distribution $N(\bm{\mu}, \bm{\Sigma})$ in the PAC learning setting.
Classically, in the advice-free setting, $\wt{\Theta}(d^2/\varepsilon^2)$ samples are sufficient and worst case necessary to learn $d$-dimensional Gaussians up to TV distance $\varepsilon$ with constant probability.
When we are additionally given a parameter $\wt{\bm{\Sigma}}$ as advice, we show that $\wt{\cO}(d^{2-\beta}/\varepsilon^2)$ samples suffices whenever $\| \wt{\bm{\Sigma}}^{-1/2} \bm{\Sigma} \wt{\bm{\Sigma}}^{-1/2} - \bm{I_d} \|_1 \leq \varepsilon d^{1-\beta}$ (where $\|\cdot\|_1$ denotes the entrywise $\ell_1$ norm) for any $\beta > 0$, yielding a polynomial improvement over the advice-free setting.
\end{abstract}

\input{introduction}
\input{preliminaries}
\input{identity-covariance}
\input{general-covariance}
\input{lower-bound}
\input{experiments}

\subsubsection*{Acknowledgements}
\begin{itemize}
\item AB: This research was supported by the National Research Foundation (NRF), Prime Minister’s Office, Singapore under its Campus for Research Excellence and Technological Enterprise (CREATE) programme, by the NRF-AI Fellowship R-252-100-B13-281, Amazon Faculty Research Award, and Google South \& Southeast Asia Research Award. 
\item DC: This research/project is supported by the National Research Foundation, Singapore under its AI Singapore Programme (AISG Award No: AISG-PhD/2021-08-013).
\item TG: This research was partially supported by MoE AcRF Tier 1 A-8000980-00-00 while at NUS and by an NTU startup grant while at NTU.
\end{itemize}

\bibliographystyle{alpha}
\bibliography{refs}

\newpage
\appendix

\input{appendix-additional-results}
\input{appendix-identity-covariance}
\input{appendix-general-covariance}
\input{appendix-python-code}

\end{document}

%% file: introduction.tex
\section{Introduction}
\label{sec:introduction}

The problem of approximating an underlying distribution from its observed samples is a fundamental scientific problem.
The \emph{distribution learning} problem has been studied for more than a century in statistics, and it is the underlying engine for much of applied machine learning.
The emphasis in modern applications is on high-dimensional distributions, with the goal being to understand when one can escape the curse of dimensionality.
The survey by \cite{diakonikolas2016learning} gives an excellent overview of classical and modern techniques for distribution learning, especially when there is some underlying structure to be exploited. 

In this work, we investigate how to go beyond worst case sample complexities for learning distributions by considering situations where one is also given the aid of possibly imperfect advice regarding the input distribution.
We position our study in the context of \emph{algorithms with predictions}, where the usual problem input is supplemented by ``predictions'' or ``advice'' (potentially drawn from modern machine learning models). The algorithm's goal is to incorporate the advice in a way that improves performance if the advice is of high quality, but if the advice is inaccurate, there should not be degradation below the performance in the no-advice setting.
Most previous works in this setting are in the context of online algorithms, e.g.\ for the ski-rental problem \cite{gollapudi2019online,wang2020online,angelopoulos2020online}, non-clairvoyant scheduling \cite{purohit2018improving}, scheduling \cite{lattanzi2020online,bamas2020learning,antoniadis2022novel}, augmenting classical data structures with predictions (e.g.\ indexing \cite{kraska2018case} and Bloom filters \cite{mitzenmacher2018model}), online selection and matching problems \cite{antoniadis2020secretary,dutting2021secretaries, choo2024online}, online TSP \cite{bernardiniuniversal,gouleakis2023learning}, and a more general framework of online primal-dual algorithms \cite{bamas2020primal}.
However, there have been some recent applications to other areas, e.g.\ graph algorithms \cite{chen2022faster,dinitz2021faster}, causal learning \cite{choo2023active}, and mechanism design \cite{gkatzelis2022improved,agrawal2022learning}.

We apply the algorithms with predictions perspective to the classical problem of learning high-dimensional Gaussian distributions.
For a $d$-dimensional Gaussian $N(\bm{\mu}, \bm{\Sigma})$, it is known (e.g.\ see Appendix C of \cite{ashtiani2020gaussian}) that
\begin{enumerate}
    \item When $\bm{\Sigma} = \bI_d$, $\wt{\Theta}(d / \eps^2)$ i.i.d.\ samples suffice to learn a $\wh{\bm{\mu}} \in \R^d$ such that $\tv(N(\bm{\mu}, \bI_d), N(\wh{\bm{\mu}}, \bI_d)) \leq \eps$.
    \item In general, $\wt{\Theta}(d^2 / \eps^2)$ i.i.d.\ samples suffice to learn $\wh{\bm{\mu}}$ and $\wh{\bm{\Sigma}}$ such that $\tv(N(\bm{\mu}, \bm{\Sigma}), N(\wh{\bm{\mu}}, \wh{\bm{\Sigma}})) \leq \eps$.
\end{enumerate}
Here, $\tv$ denotes the \emph{total variation distance}, and the algorithm for both cases is the most natural one: compute the empirical mean and empirical covariance.
Meanwhile, note that if one is given as advice the correct mean $\wt{\bm{\mu}} = \bm{\mu}$, then using distribution testing, one can certify that $\|\wt{\bm{\mu}}-\bm{\mu}\|_2 \leq \eps$ using only $\wt{\Theta}(\sqrt{d}/\eps^2)$ samples, quadratically better than without advice; see Appendix C of \cite{DiakonikolasKaneStewart2017}.
This observation motivates the object of our study.

\renewenvironment{quote}
{\list{}{\leftmargin=10pt\rightmargin=10pt}\item[]}
{\endlist}
\begin{quote}
\textbf{\textsc{Gaussian Learning with Advice}}: Given samples from a Gaussian $N(\bm{\mu}, \bm{\Sigma})$, as well as advice $\wt{\bm{\mu}}$ and $\wt{\bm{\Sigma}}$, how many samples are required to recover $\wh{\bm{\mu}}$ and $\wh{\bm{\Sigma}}$ such that $\tv(N(\bm{\mu}, \bm{\Sigma}), N(\wh{\bm{\mu}}, \wh{\bm{\Sigma}})\leq \eps$ with probability at least $1-\delta$? The sample complexity should be a function of the dimension, $\eps, \delta$, as well as a measure of how close $\wt{\bm{\mu}}$ and $\wt{\bm{\Sigma}}$ are to $\bm{\mu}$ and $\bm{\Sigma}$ respectively.
\end{quote}

\textbf{Notation.}
We use \emph{lowercase letters} for scalars, set elements, random variable instantiations, \emph{uppercase letters} for random variables, \emph{bolded lowercase letters} for vectors and sets, \emph{bolded uppercase letters} for set of random variables and matrices, \emph{calligraphic letters} for probability distributions and sets of sets, and \emph{small caps} for algorithm names.
Intuitively, we use non-bolded versions for singletons, bolded versions for collections of items, and calligraphic for more complicated objects.
The context should be clear enough to distinguish between various representations.

\subsection{Our main results}

We give the first known results in distribution learning with imperfect advice.
Our techniques are piecewise elementary and easy to follow.
Furthermore, we provide polynomial algorithms for producing the estimates $\wh{\bm{\mu}}$ and $\wh{\bm{\Sigma}}$ based on LASSO and SDP formulations.

Given a mean $\wt{\bm{\mu}} \in \R^d$ and covariance matrix $\wt{\bm{\Sigma}} \in \R^{d \times d}$ as advice, we present two algorithms \textsc{TestAndOptimizeMean} and \textsc{TestAndOptimizeCovariance} that provably improve on the sample complexities of $\wt{\Theta}(d / \eps^2)$ and $\wt{\Theta}(d^2 / \eps^2)$ for identity and general covariances respectively when given high quality advice.

\begin{restatable}{theorem}{identitycovariance}
\label{thm:main-result-identity-covariance}
For any given $\eps, \delta \in (0,1)$, $\eta \in [0,\frac{1}{4}]$, and $\wt{\bm{\mu}} \in \R^d$, the \textsc{TestAndOptimizeMean} algorithm uses
$n \in \wt{\cO} \left( \frac{d}{\eps^2} \cdot \left( d^{- \eta} + \min\{ 1, f(\bm{\mu}, \wt{\bm{\mu}}, d, \eta, \eps) \} \right) \right)$,
where
\[
f(\bm{\mu}, \wt{\bm{\mu}}, d, \eta, \eps) = \frac{\|\bm{\mu} - \wt{\bm{\mu}} \|_1^2}{d^{1 - 4 \eta}\eps^2} \;,
\]
i.i.d.\ samples from $N(\bm{\mu}, \bI_d)$ for some unknown mean $\bm{\mu}$ and identity covariance $\bI_d$, and can produce $\wh{\bm{\mu}}$ in $\poly(n, d)$ time such that $\tv(N(\bm{\mu}, \bI_d), N(\wh{\bm{\mu}}, \bI_d)) \leq \eps$ with success probability at least $1 - \delta$.
\end{restatable}

\begin{restatable}{theorem}{generalcovariance}
\label{thm:main-result-general-covariance}
For any given $\eps, \delta \in (0,1)$, $\eta \in [0, 1]$ and $\wt{\bm{\Sigma}} \in \R^{d \times d}$, \textsc{TestAndOptimizeCovariance} uses
$n \in \wt{\cO} \left( \frac{d^2}{\eps^2} \cdot \left( d^{- \eta} + \min \left\{ 1, f(\bm{\Sigma}, \wt{\bm{\Sigma}}, d, \eta, \eps) \right\} \right) \right)$,
where
\[
f(\bm{\Sigma}, \wt{\bm{\Sigma}}, d, \eta, \eps) = \frac{\| \vec( \wt{\bm{\Sigma}}^{-1/2} \bm{\Sigma} \wt{\bm{\Sigma}}^{-1/2} - \bI_d) \|_1^2}{d^{2 - \eta} \eps^2} \;,
\]
i.i.d.\ samples from $N(\bm{\mu}, \bm{\Sigma})$ for some unknown mean $\bm{\mu}$ and unknown covariance $\bm{\Sigma}$, and can produce $\wh{\bm{\mu}}$ and $\wh{\bm{\Sigma}}$ in $\poly(n, d, \log(1/\eps))$ time such that $\tv(N(\bm{\mu}, \bm{\Sigma}), N(\wh{\bm{\mu}}, \wh{\bm{\Sigma}})) \leq \eps$ with success probability at least $1 - \delta$.
\end{restatable}

In particular, the \textsc{TestAndOptimizeMean} algorithm uses only $\wt{\cO}(\frac{d^{1-\eta}}{\eps^2})$ samples when $\| \bm{\mu} - \wt{\bm{\mu}} \|_1 < \eps d^{(1 - 5 \eta)/2} = \eps \sqrt{d} \cdot d^{- 5 \eta/2}$, for any $\eta \in [0, \frac{1}{4}]$.
Similarly, \textsc{TestAndOptimizeCovariance} algorithm uses only $\wt{\cO}(\frac{d^{2 - \eta}}{\eps^2})$ samples when $\| \vec( \wt{\bm{\Sigma}}^{-1/2} \bm{\Sigma} \wt{\bm{\Sigma}}^{-1/2} - \bI_d ) \|_1 < \eps d^{1 - \eta} = \eps d \cdot d^{- \eta}$, for any $\eta \in [0, 1]$.
Moreover, both algorithms \textsc{TestAndOptimizeMean} and \textsc{TestAndOptimizeCovariance} have polynomial runtime.

The choice of representing the quality of the advice in terms of the $\ell_1$-norm is well-motivated.
It is known, e.g.\ see Theorem 2.5 of \cite{10.5555/2526243}, that if a vector $\bm{x}$ satisfies $\|\bm{x}\|_1 \leq \tau$, then for any positive integer $s$, $\sigma_s(\bm{x})\leq \tau/(2\sqrt{s})$, where $\sigma_s(\bm{x})$ is the $\ell_2$-error of the best $s$-sparse approximation to $\bm{x}$.
Thus, if $\|\wt{\bm{\mu}}-\bm{\mu}\|_1\leq 2\eps d^{(1- \eta)/2}$, then $\sigma_{d^{1-\eta}}(\wt{\bm{\mu}}-\bm{\mu}) \leq \eps$.
The latter may be very reasonable, as one may have good predictions for most of the coordinates of the mean with the error in the advice concentrated on a sublinear ($d^{1-\eta}$) number of coordinates.
Algorithmically, we employ sublinear property testing algorithms to evaluate the quality of the given advice before deciding how to produce a final estimate, similar in spirit to the \textsc{TestAndMatch} approach in \cite{choo2024online}.
The idea of incorporating  property testing as a way to verify whether certain distributional assumptions are satisfied that enable efficient subsequent learning has also been explored in recent works on testable learning \cite{rubinfeld2023testing,klivans24testable,vasilyan2024enhancing}.

We supplement our algorithmic upper bounds with information-theoretic lower bounds.
Here, we say that an algorithm $(\eps,1-\delta)$-PAC learns a distribution $\cP$ if it can produce another distribution $\wh{\cP}$ such that $\tv(\cP, \wh{\cP}) \leq \eps$ with success probability at least $1-\delta$.
Our lower bounds tell us that $\wt{\Omega}(d/\eps^2)$ and $\wt{\Omega}(d^2/\eps^2)$ samples are unavoidable for PAC-learning $N(\bm{\mu}, \bI_d)$ and $N(\bm{\mu}, \bm{\Sigma})$ respectively when given low quality advice.

\begin{restatable}{theorem}{mainresultlowerbound}\label{thm:mainresultlowerbound}
Suppose we are given $\wt{\bm{\mu}} \in \R^d$ as advice with only the guarantee that $\|\bm{\mu} - \wt{\bm{\mu}}\|_1 \leq \Delta$.
Then, any algorithm that $(\eps,\frac{2}{3})$-PAC learns $N(\bm{\mu}, \bI_d)$ requires $\Omega\left(\frac{\min\{d, \Delta^2/\eps^2\}}{\eps^2 \log (1/\eps)}\right)$ samples in the worst case.
\end{restatable}

\begin{restatable}{theorem}{mainresultlowerboundcovariance}\label{thm:mainresultlowerboundcovariance}
Suppose we are given a symmetric and positive-definite $\wt{\bm{\Sigma}} \in \R^{d \times d}$ as advice with only the guarantee that $\|\vec\left(\wt{\bm{\Sigma}}^{-\frac{1}{2}} \bm{\Sigma} \wt{\bm{\Sigma}}^{-\frac{1}{2}} -\bI_d\right)\|_1 \leq \Delta$. Then, any algorithm that $(\eps,\frac{2}{3})$-PAC learns $N(\bm{0}, \bm{\Sigma})$ requires $\Omega \left( \frac{\min\{d^2, \Delta^2/\eps^2\}}{\eps^2 \log(1/\eps)} \right)$ samples in the worst case.
\end{restatable}

Both of our lower bounds are tight in the following sense. Our algorithm \textsc{TestAndOptimizeMean} gives a polynomially-smaller sample complexity compared to $\wt{\cO}(d/\varepsilon^2)$ when the advice quality (measured in terms of the $\ell_1$-norm) is polynomially smaller compared to $\eps \sqrt{d}$.
\cref{thm:mainresultlowerbound} shows that this is the best we can do; there is a hard instance where the advice quality is $\leq \eps \sqrt{d}$ and we need $\wt{\Omega}(d/\eps^2)$ samples.
A similar situation happens between \textsc{TestAndOptimizeCovariance} and \cref{thm:mainresultlowerboundcovariance}, when the guarantee on the advice quality is at most $\eps d$.

The lower bounds in \cref{thm:mainresultlowerbound} and \cref{thm:mainresultlowerboundcovariance} apply when the parameter $\Delta$ is known to the algorithm. Our algorithms are stronger since they do not need to know $\Delta$ beforehand. In case $\Delta$ is known, the sample complexity of the distribution learning component of our algorithms match the above lower bounds up to log factors. 

\subsection{Technical overview}

To obtain our upper bounds, we first show that the existing test statistics for non-tolerant testing can actually be used for tolerant testing with the same asymptotic sample complexity bounds and then use these new tolerant testers to test the advice quality. 
The tolerance is with respect to the $\ell_2$-norm for mean testing and with respect to the Frobenius norm for covariance testing.
These results are folklore, but we did not manage to find formal proofs for them.
As these may be of independent interest, we present their proofs in \cref{sec:appendix-tolerant-testing} for completeness.

\begin{restatable}[Tolerant mean tester]{lemma}{tolerantmeantester}
\label{lem:tolerant-mean-tester}
Given $\eps_2 > \eps_1 > 0$, $\delta \in (0,1)$, and $d \geq \left( \frac{16 \eps_2^2}{\eps_2^2 - \eps_1^2} \right)^2$, there is a tolerant tester that uses $\cO \left( \frac{\sqrt{d}}{\eps_2^2 - \eps_1^2} \log \left( \frac{1}{\delta} \right) \right)$ i.i.d.\ samples from $N(\bm{\mu}, \bI_d)$ and satisfies both conditions below:\\
1. If $\| \bm{\mu} \|_2 \leq \eps_1$, then the tester outputs $\Accept$,\\
2. If $\| \bm{\mu} \|_2 \geq \eps_2$, then the tester outputs $\Reject$,\\
each with success probability at least $1 - \delta$.
\end{restatable}

\begin{restatable}[Tolerant covariance tester]{lemma}{tolerantcovariancetester}
\label{lem:tolerant-covariance-tester}
Given $\eps_2 > \eps_1 > 0$, $\delta \in (0,1)$, and $d \geq \eps_2^2$, there is a tolerant tester that uses $\cO \left( d \cdot \max \left\{ \frac{1}{\eps_1^2}, \left( \frac{\eps_2^2}{\eps_2^2 - \eps_1^2} \right)^2, \left( \frac{\eps_2}{\eps_2^2 - \eps_1^2} \right)^2 \right\} \log \left( \frac{1}{\delta} \right) \right)$ i.i.d.\ samples from $N(\mathbf{0}, \bm{\Sigma})$ and satisfies both conditions below:\\
1. If $\| \bm{\Sigma} - \bI_d \|_F \leq \eps_1$, then the tester outputs $\Accept$,\\
2. If $\| \bm{\Sigma} - \bI_d \|_F \geq \eps_2$, then the tester outputs $\Reject$,\\
each with success probability at least $1 - \delta$.
\end{restatable}

We will first explain how to obtain our result for \textsc{TestAndOptimizeMean} before explaining how a similar approach works for \textsc{TestAndOptimizeCovariance}.

\subsubsection{Approach for \textsc{TestAndOptimizeMean}}
\label{sec:technical-overview-identity-covariance}

Without loss of generality, we may assume henceforth that $\wt{\bm{\mu}} = \bm{0}$ since one can always pre-process samples by subtracting $\wt{\bm{\mu}}$ and then add $\wt{\bm{\mu}}$ back to the estimated $\wh{\bm{\mu}}$. Our overall approach is quite natural: (i) use the tolerant testing algorithm in \cref{lem:tolerant-mean-tester} to get an upper bound on the ``advice quality'', and (ii) enforce the constraint on the ``advice quality'' when learning $\wh{\bm{\mu}}$. 

The most immediate notion of advice quality one may posit is $\|\bm{\mu}-\bm{0}\|_2 = \|\bm{\mu}\|_2$.
Let us see what issues arise.
Using an exponential search process, we can invoke \cref{lem:tolerant-mean-tester} directly to find some $r > 0$, such that $r/2 \leq \| \bm{\mu} - \wt{\bm{\mu}} \|_2  =  \| \bm{\mu} \|_2 \leq r$.
To argue about the sample complexity for learning $\wh{\bm{\mu}}$, and ignoring computational efficiency, one can invoke the Scheff\'{e} tournament approach for density estimation.
Let $\mathcal{N}$ be an $\eps$-cover in $\ell_2$ of the the $\ell_2$-ball of radius $r$ around $\bm{0}$.
Clearly, $\bm{\mu}$ is $\eps$-close in $\ell_2$ to one of the points in $\mathcal{N}$. It is known (e.g.\ see Chapter 4 of \cite{devroye2001combinatorial}) that the sample complexity of the Scheff\'{e} tournament algorithm scales as $\log |\mathcal{N}|$.
However, we have that $\log |\mathcal{N}| = \Omega(d)$; e.g.\ see Proposition 4.2.13 of \cite{vershynin2018high}.
Indeed, one can get a formal lower bound showing that the sample complexity cannot be made sublinear in $d$ for non-trivial values of $r$.
To get around this barrier, we will instead take the notion of advice quality to be $\|\bm{\mu}\|_1$ instead of $\|\bm{\mu}\|_2$. It is known that $d^{\frac{c r^2}{\eps^2}}$ $\ell_2$ balls of radius $\eps$ suffice to cover an $\ell_1$-ball of radius $r$, for some absolute constant $c > 0$; e.g.\ see Chapter 4, Example 2.8 of \cite{vershynin2012lectures}.
Using this modified approach, the Scheff\'{e} tournament only requires $\cO(\frac{r^2}{\eps^4} \log d)$ samples which could be $o(d/\eps^2)$ for a wide range of values of $r$.

There are still two issues to address: (i) how to obtain an $\ell_1$ estimate $r$ of $\bm{\mu}$, i.e., $r/2 \leq \| \bm{\mu} \|_1 \leq r$, and (ii) how to get a computationally efficient learning algorithm. 

To address (i), we can apply the standard inequality $\| \bm{\mu} \|_2 \leq \| \bm{\mu} \|_1 \leq \sqrt{d} \| \bm{\mu} \|_2$ bound to transform our $\ell_2$ estimate from \cref{lem:tolerant-mean-tester} into an $\ell_1$ one.
However, since the number of samples has a quadratic relation with $r$, we need a better approximation than $\sqrt{d}$ to achieve sample complexity that is sublinear in $d$. To achieve this, we partition the $\bm{\mu}$ vector into blocks of size at most $k \leq d$ and approximate the $\ell_1$ norm of each smaller dimension vector separately and then add them up to obtain an $\ell_1$ estimate of the overall $\bm{\mu}$.
Doing so improves the resulting multiplicative error to $\approx \sqrt{d/k}$ instead of $\sqrt{d}$. In effect, we devise a tolerant tester for a mixed $\ell_{1,2}$ norm instead of the $\ell_1$ or $\ell_2$ norms directly. 

To address (ii), observe that the Scheff\'{e} tournament approach requires time at least linear in the size of the $\eps$-cover.
In order to do better, we observe that we can formulate our task as an optimization problem with an $\ell_1$-constraint. Specifically, given samples $\by_1, \dots, \by_n$, we solve the following program:
\[
\wh{\bm{\mu}}= \argmin_{\|\bm{\beta}\|_1 \leq r} \frac{1}{n} \sum_{i=1}^n \|\by_i - \bm{\beta}\|_2^2
\]
The error $\|\bm{\mu}-\bm{\wh{\mu}}\|_2$ can be analyzed by similar techniques as those used for analyzing $\ell_1$-regularization in the context of LASSO or compressive sensing; e.g.\ see \cite{tibshirani1996regression, tibshirani1997lasso, hastie2015statistical}.

\subsubsection{Approach for \textsc{TestAndOptimizeCovariance}}
\label{sec:technical-overview-general-covariance}

As before, we may assume without loss of generality that $\wt{\bm{\Sigma}} = \bI_d$ by pre-processing the samples appropriately.
Furthermore, we can invest $\Omega(d/\eps^2)$ samples up-front to ensure that the empirical mean $\wh{\bm{\mu}}$ will be an $\eps$-good estimate of $\bm{\mu}$.
Then, it will suffice to obtain an estimate $\wh{\bm{\Sigma}}$ of $\bm{\Sigma}$ such that $\| \bm{\Sigma}^{-1} \wh{\bm{\Sigma}} - \bI_d \|_F \leq \cO(\eps)$ suffices.
At a high level, the approach for \textsc{TestAndOptimizeCovariance} is the same as \textsc{TestAndOptimizeMean} after three key adjustments to adapt the approach from vectors to matrices.

The first adjustment is that we perform a suitable preconditioning process using an additional $\cO(d)$ samples so that we can subsequently argue that $\| \bm{\Sigma}^{-1} \|_2 \leq 1$.
This will then allow us to argue that $\| \bm{\Sigma}^{-1} \wh{\bm{\Sigma}} - \bI_d \|_F \leq \| \bm{\Sigma}^{-1} \|_2 \| \wh{\bm{\Sigma}} - \bm{\Sigma} \|_F \in \cO(\eps)$.
Our preconditioning technique is inspired by \cite{kamath2019privately}; while they use $\cO(d)$ samples to construct a preconditioner to control the maximum eigenvalue, we use a similar approach to control the minimum eigenvalue.

The second adjustment pertains to the partitioning idea used for multiplicatively approximating $\|\vec(\bm{\Sigma}-\bI_d)\|_1$.
Observe that the covariance matrix of a marginal of a multivariate Gaussian is precisely the principal submatrix of the original covariance $\bm{\Sigma}$ on the corresponding projected coordinates.
For example, if one focuses on coordinates $\{i,j\} \subseteq [d]$ of each sample, then the corresponding covariance matrix is
$\begin{bmatrix}
\bm{\Sigma}_{i,i} & \bm{\Sigma}_{i,j}\\
\bm{\Sigma}_{j,i} & \bm{\Sigma}_{j,j}
\end{bmatrix}$, for $i < j$.
To this end, we generalize the partitioning scheme described for \textsc{TestAndOptimizeMean} to higher ordered objects.

\begin{definition}[Partitioning scheme]
\label{def:partitioning-scheme}
Fix $q \geq 1$, $d \geq 1$, and a $q$-ordered $d$-dimensional tensor $\cT \in \R^{d^{\otimes q}}$.
Let $\bB \subseteq [d]$ be a subset of indices and define $\cT_{\cB}$ as the principal subtensor of $\cT$ indexed by $\bB$.
A collection of subsets $\bB_1, \ldots, \bB_w \subseteq [d]$ is called an $(q,d,k,a,b)$-partitioning of the tensor $\cT$ if the following three properties hold:
\begin{itemize}
    \item $|\bB_1| \leq k, \ldots, |\bB_w| \leq k$
    \item For every cell of $\cT$ appears in \emph{at least} $a$ of the $w$ principal subtensors $\cT_{\bB_1}, \ldots, \cT_{\bB_w}$.
    \item For every cell of $\cT$ appears in \emph{at most} $b$ of the $w$ principal subtensors $\cT_{\bB_1}, \ldots, \cT_{\bB_w}$.
\end{itemize}
\end{definition}

For example, when $q = 2$, $\bT \in \R^{d \times d}$ is just a $d \times d$ matrix.
Observe one can always obtain a partitioning with $k \leq d^q$ by letting the index sets $\bB_1, \ldots, \bB_w$ encode every possible index, but this results in a large $w = \binom{d}{q}$ which can be undesirable for downstream analysis.
The partitioning used in \textsc{TestAndOptimizeMean} is a special case of \cref{def:partitioning-scheme} with $q = a = b = 1$, $k = \lceil d/w \rceil$.
For \textsc{TestAndOptimizeCovariance}, we are interested in the case where $q = 2$ and $a = 1$.
Ideally, we want to minimize $k$ and $b$ as well.
\cref{fig:partitioning-visualization} illustrates an example of a $(q=2, d=5, k=3, a=1, b=3)$-partitioning.

\begin{figure}[htb]
\centering
\resizebox{0.9\linewidth}{!}{
    \input{tikz/partitioning}
}
\caption{Consider partitioning a $d \times d$ matrix (i.e.\ $d = 5$, $q = 2$) with $w = 4$ blocks $\{(1,2,3), (1,4,5), (2,4,5), (3,4,5)\}$, each of size $k = 3$. Every cell in the original $5 \times 5$ matrix appears in at least $a = 1$ and at most $b = 3$ times across all the induced submatrices.}
\label{fig:partitioning-visualization}
\end{figure}
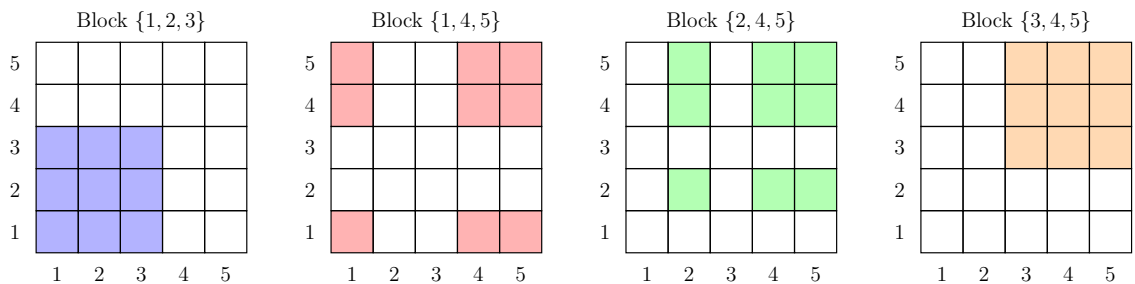

The last change is to the optimization program for learning $\wh{\bm{\Sigma}}$. Given samples $\by_1, \dots, \by_n$ from $N(\bm{\mu}, \bm{\Sigma})$, we define:
\[
\wh{\bm{\Sigma}} = \argmin_{\substack{\text{$\bA \in \R^{d \times d}$ is p.s.d.}\\ \| \vec(\bA - \bI_d) \|_1 \leq r\\ \| \bA^{-1} \|_2 \leq 1}} \sum_{i=1}^n \| \bA - \by_i \by_i^\top \|_F^2
\]
Observe that $\bm{\Sigma}$ is a feasible solution to the above program.
The optimization problem can be solved efficiently since it can be written as an SDP with convex constraints; see \cref{sec:appendix-cov-estimation-program}.
We finally bound $\|\bm{\Sigma} - \wh{\bm{\Sigma}}\|_F$ using an analysis that mirrors that for \textsc{TestAndOptimizeMean} but is in terms of matrix algebra.

\subsubsection{Lower bound}

To prove our lower bound results (\cref{thm:mainresultlowerbound} and \cref{thm:mainresultlowerboundcovariance}), we make use of a lemma in \cite{ashtiani2020gaussian} that informally says the following: If we can construct a cover $f_1,\ldots,f_M$ of distributions such that the pairwise KL divergence is at most $\kappa$ and the pairwise TV distance is $> 2\varepsilon$, then, given sample access to an unknown $f_i$, the sample complexity of learning a distribution which is $\varepsilon$-close to $f_i$ in total variation with probability $\geq \frac23$ over the samples (which is referred to as $(\varepsilon, \frac23)$-PAC learning in total variation) is $\geq \wt{\Omega}\left(\frac{\log M}{\kappa}\right)$. This lemma gives an information-theoretic lower bound and is a consequence of the generalized Fano's inequality.

To apply this lemma in the context of learning with advice, we need to fix an advice $\bm{a}$ (mean or covariance, in the case of our problem) and find a large cover of distributions $f_1,\ldots,f_M$ that satisfy the conditions of the lemma (pairwise KL $\leq \kappa$ and pairwise TV $> 2\eps$), while also satisfying a guarantee on the advice quality with respect to \emph{all} $f_1,\ldots,f_M$ (say, the quality of $\bm{a}$ is $Q$). Then, applying the lemma will show a sample complexity lower bound for learning a distribution given advice with quality $Q$, since an adversary can choose an $f_i$ in the cover set and give $\bm{a}$ (fixed) as the advice in each case while still satisfying the advice quality requirement. Since, in this context, we know that the underlying ground truth is one of $f_1,\ldots,f_M$, the advice $\bm{a}$ is immaterial. The lemma asserts that we still need $\wt{\Omega}\left(\frac{\log M}{\kappa}\right)$ samples to learn a distribution close to the given $f_i$ (where the pairwise TV separation of $> 2\eps$ is crucial in ensuring that the learning algorithm would need to identify the correct $f_i$ to succeed, since no distribution $f$ will be $\varepsilon$-close in TV to $f_i$ and $f_j$ for $i \neq j$ due to the triangle inequality).

In the context of learning a Gaussian with unknown mean, the advice quality that we consider is $\|\wt{\bm{\mu}} - \bm{\mu}\|_1$, where $\wt{\bm{\mu}}$ is the advice and $\bm{\mu}$ is the ground truth.
To show \cref{thm:mainresultlowerbound}, we construct a cover of $M$ distributions $N(\bm{\mu}_i, \bI_d)$ such that $\|\wt{\bm{\mu}} - \bm{\mu}_i\|_1$ is precisely the same for all $\bm{\mu}_i$'s.
Then, we ensure that the pairwise TV and KL requirements are satisfied by controlling the $\ell_2$ distance $\|\bm{\mu}_i - \bm{\mu}_j\|_2$ for each pair $i \neq j$. This enables us to use a construction where we set the first $k$ coordinates of each $\bm{\mu}_i$ based on the codewords of an error correcting code with distance $\geq \Omega(k)$, and we can show the existence of such a code with $2^{\Omega(k)}$ codewords using the Gilbert-Varshamov bound.

In the context of learning Gaussians with unknown covariance, we consider the advice quality $\|\wt{\bm{\Sigma}}^{-\frac{1}{2}} \bm{\Sigma} \wt{\bm{\Sigma}}^{-\frac{1}{2}} - \bI_d\|_1$ where $\bm{\Sigma}$ is the ground truth and $\wt{\bm{\Sigma}}$ is the advice. To prove a lower bound on the sample complexity of learning given good advice, we follow a similar strategy where again, we want to construct a cover of $M$ distributions $N(\bm{0}, \bm{\Sigma}_i)$ which all satisfy a bound on the advice quality and also satisfy the pairwise TV and KL requirements. \cite{ashtiani2020gaussian} also pursue the same goal but without the advice quality constraint. We adapt their construction by defining a family of block-diagonal orthogonal matrices such that the size of the submatrices can be used to control the entrywise $\ell_1$-norm distance to the identity. Quantifying the KL divergences and TV distances between the constructed gaussians then gives the desired lower bound.

\paragraph{Outline of the paper}

We begin with preliminary materials and related work in \cref{sec:preliminaries}.
Then, we present \textsc{TestAndOptimizeMean} and \textsc{TestAndOptimizeCovariance} in \cref{sec:identity-covariance} and \cref{sec:general-covariance} respectively.
Our hardness results are given in \cref{sec:lower-bounds} and some experimental results illustrating the savings in sample complexity are shown in \cref{sec:experiments}.

%% file: tikz/partitioning.tex
\begin{tikzpicture}
\def\rows{5}
\def\cols{5}


\def\offsetx{0}
\def\offsety{0}
\foreach \x in {1,2,3} {
    \foreach \y in {1,2,3} {
        \fill[blue!30!white] (\offsetx + \x, \offsety + \y) rectangle (\offsetx + \x + 1, \offsety + \y + 1);
    }
}
\foreach \x in {1, ..., \cols} {
    \foreach \y in {1, ..., \rows} {
        \draw[black, thick] (\offsetx + \x, \offsety + \y) rectangle (\offsetx + \x + 1, \offsety + \y + 1);
    }
}
\foreach \x in {1, ..., \cols} {
    \node at (\offsetx + \x + 0.5, \offsety + 0.5) {\Large \x};
}
\foreach \y in {1, ..., \cols} {
    \node at (\offsetx + 0.5, \offsety + \y + 0.5) {\Large \y};
}
\node at (\offsetx + 3.5, \offsety + 1.5 + \rows) {\Large Block $\{1,2,3\}$};

\def\offsetx{7}
\def\offsety{0}
\foreach \x in {1,4,5} {
    \foreach \y in {1,4,5} {
        \fill[red!30!white] (\offsetx + \x, \offsety + \y) rectangle (\offsetx + \x + 1, \offsety + \y + 1);
    }
}
\foreach \x in {1, ..., \cols} {
    \foreach \y in {1, ..., \rows} {
        \draw[black, thick] (\offsetx + \x, \offsety + \y) rectangle (\offsetx + \x + 1, \offsety + \y + 1);
    }
}
\foreach \x in {1, ..., \cols} {
    \node at (\offsetx + \x + 0.5, \offsety + 0.5) {\Large \x};
}
\foreach \y in {1, ..., \cols} {
    \node at (\offsetx + 0.5, \offsety + \y + 0.5) {\Large \y};
}
\node at (\offsetx + 3.5, \offsety + 1.5 + \rows) {\Large Block $\{1,4,5\}$};

\def\offsetx{14}
\def\offsety{0}
\foreach \x in {2,4,5} {
    \foreach \y in {2,4,5} {
        \fill[green!30!white] (\offsetx + \x, \offsety + \y) rectangle (\offsetx + \x + 1, \offsety + \y + 1);
    }
}
\foreach \x in {1, ..., \cols} {
    \foreach \y in {1, ..., \rows} {
        \draw[black, thick] (\offsetx + \x, \offsety + \y) rectangle (\offsetx + \x + 1, \offsety + \y + 1);
    }
}
\foreach \x in {1, ..., \cols} {
    \node at (\offsetx + \x + 0.5, \offsety + 0.5) {\Large \x};
}
\foreach \y in {1, ..., \cols} {
    \node at (\offsetx + 0.5, \offsety + \y + 0.5) {\Large \y};
}
\node at (\offsetx + 3.5, \offsety + 1.5 + \rows) {\Large Block $\{2,4,5\}$};

\def\offsetx{21}
\def\offsety{0}
\foreach \x in {3,4,5} {
    \foreach \y in {3,4,5} {
        \fill[orange!30!white] (\offsetx + \x, \offsety + \y) rectangle (\offsetx + \x + 1, \offsety + \y + 1);
    }
}
\foreach \x in {1, ..., \cols} {
    \foreach \y in {1, ..., \rows} {
        \draw[black, thick] (\offsetx + \x, \offsety + \y) rectangle (\offsetx + \x + 1, \offsety + \y + 1);
    }
}
\foreach \x in {1, ..., \cols} {
    \node at (\offsetx + \x + 0.5, \offsety + 0.5) {\Large \x};
}
\foreach \y in {1, ..., \cols} {
    \node at (\offsetx + 0.5, \offsety + \y + 0.5) {\Large \y};
}
\node at (\offsetx + 3.5, \offsety + 1.5 + \rows) {\Large Block $\{3,4,5\}$};
\end{tikzpicture}

%% file: preliminaries.tex
\section{Preliminaries}
\label{sec:preliminaries}

\paragraph{Notation}
We use \emph{lowercase letters} for scalars, set elements, random variable instantiations, \emph{uppercase letters} for random variables, \emph{bolded lowercase letters} for vectors and sets, \emph{bolded uppercase letters} for set of random variables and matrices, \emph{calligraphic letters} for probability distributions and sets of sets, and \emph{small caps} for algorithm names.
Intuitively, we use non-bolded versions for singletons, bolded versions for collections of items, and calligraphic for more complicated objects.
The context should be clear enough to distinguish between various representations.

For any integer $d \geq 1$, we write $[d]$ to mean the set of integers $\{1, \ldots, d\}$.
We will write $\bv \sim N(\bm{\mu}, \bm{\Sigma})$ to mean drawing a multivariate Gaussian sample and $\cM = \{\bv_1, \ldots, \bv_{|\cM|}\}$ to mean a collection of $|\cM|$ independently drawn such vectors.

In the rest of this section, we will state some basic facts and lemmas that would be useful for our work.
Most of them are folklore results and we supplement proofs in \cref{sec:appendix-preliminaries-proofs} for them when we could not nail down a direct reference.

\subsection{Matrix facts}

\begin{fact}[e.g.\ see Exercise 5.4.P3 of \cite{horn2012matrix}]
\label{lem:l1-l2-inequality}
Let $\bx \in \R^d$ be an arbitrary $d$-dimensional real vector.
Then, the $\ell_1$ and $\ell_2$ norms of $\bx$ are defined as $\| \bx \|_1 = \sum_{i=1}^d |\bx_i|$ and $ \| \bx \|_2 = \sqrt{\sum_{i=1}^d \bx_i^2}$ respectively. They satisfy the inequality: $\|\bx\|_2 \leq \|\bx\|_1 \leq \sqrt{d}\cdot \|\bx\|_2$. 
\end{fact}

For a real matrix $\bM \in \R^{d \times d}$, we define its vectorized form $\vec(M) \in \R^{d^2}$ by $\vec(\bM) = (\bM_{1,1}, \ldots, \bM_{d,d})$ and we see that $\| \bM \|_F^2 = \| \vec(\bM) \|_2^2$.
We recover a matrix given its vectorized form via $\bM = \mat(\vec(\bM))$.
For any matrix $\bA$, we use $\sigma_{\min}(\bA)$ to denote its smallest eigenvalue.
Note that for any full rank matrix $\bA \in \R^{d \times d}$, we have $\frac{1}{\| \bA \|_2} \leq \| \bA^{-1} \|_2$, $\| \bA \|_2 \leq \| \bA \|_F \leq \sqrt{d} \cdot \| \bA \|_2$ (e.g.\ see Exercise 5.6.P23 of \cite{horn2012matrix}), and $\| \bA \|_F = \| \vec(\bA) \|_2 \leq \| \vec(\bA) \|_1 \leq \sqrt{d} \cdot \| \vec(\bA) \|_2$.
For any two matrices $\bA$ and $\bB$ of the same dimension, we also know that $\| \bA \bB \|_F \leq \min\{\| \bA \|_2 \| \bB \|_F, \| \bA \|_F \| \bB \|_2\}$.

\begin{restatable}[Chapter 5.6 of \cite{horn2012matrix}]{lemma}{rotatingnorm}
\label{lem:rotating-norm}
Let $\bA$ and $\bB$ be two square real matrices where $\bA$ is an invertible matrix.
Then, $\| \bA \bB \| = \| \bB \bA \|$.
\end{restatable}

\begin{restatable}{lemma}{applyingrotatingnorm}
\label{lem:applying-rotating-norm}
Let $\bA$ and $\bB$ be two square $d \times d$ matrices where $\bA$ is an invertible matrix with a square root.
Then, $\| \bA^{-1/2} \bB \bA^{-1/2} - I \| = \| \bA^{-1} \bB - \bI_d \|$
\end{restatable}

\begin{definition}[Projected vector]
\label{def:projected-vector}
Let $\bm{\bx} = (\bm{\bx}_1, \ldots, \bm{\bx}_d) \in \R^d$ be a $d$-dimensional vector and $\bB = \{i_1, \ldots, i_w\} \subseteq [d]$ be a subset of $1 \leq w \leq d$ indices, where $i_1 < \ldots < i_w$.
Then, we define $\bm{\bx}_{\bB} = (\bm{\bx}_{i_1}, \ldots, \bm{\bx}_{i_w}) \in \R^{w}$ as the projection of the vector $\bm{\bx}$ to the coordinates indicated by $\bB$.
\end{definition}

\begin{restatable}[Trace inequality]{lemma}{traceinequality}
\label{lem:trace-inequality}
For any three matrices $\bA, \bB, \bC \in \R^{d \times d}$, we have $\Tr(\bA \bB \bC) \leq \| \vec(\bB \bA) \|_1 \cdot \| \bC \|_2$.
\end{restatable}

\begin{restatable}{lemma}{vectorizedinequalities}
\label{lem:vectorized-inequalities}
For any two matrices $\bA, \bB \in \R^{d \times d}$, we have $\| \vec(\bA + \bB) \|_1 \leq \| \vec(\bA) \|_1 + \| \vec(\bB) \|_1$ and $\| \vec(\bA \bB) \|_1 \leq \| \vec(\bA) \|_1 \cdot \| \vec(\bB) \|_1$.
\end{restatable}

\subsection{Distance measures between distributions}

\begin{definition}[Kullback–Leibler (KL) divergence]
\hspace{0pt}\\
For two continuous distributions $\cP$ and $\cQ$ over $\bX$,
\[
\kl(\cP, \cQ)
= \int_{\bx \in \bX} \cP(\bx) \log \left( \frac{\cP(\bx)}{\cQ(\bx)} \right) \,d \bx
\]
\end{definition}

Note that KL divergence is not symmetric in general.

\begin{restatable}[Known fact about KL divergence]{lemma}{klknownfact}
\label{lem:kl-known-fact}
Given two $d$-dimensional multivariate Gaussian distributions $\cP \sim N(\bm{\mu}_{\cP}, \bm{\Sigma}_{\cP})$ and $\cQ \sim N(\bm{\mu}_{\cQ}, \bm{\Sigma}_{\cQ})$ where $\bm{\Sigma}_{\cP}$ and $\bm{\Sigma}_{\cQ}$ are invertible, we have
\begin{align*}
\kl(\cP, \cQ)
&= \frac{1}{2} \cdot \left( \Tr(\bm{\Sigma}_{\cQ}^{-1} \bm{\Sigma}_{\cP}) - d + (\bm{\mu}_{\cQ} - \bm{\mu}_{\cP})^\top \bm{\Sigma}_{\cQ}^{-1} (\bm{\mu}_{\cQ} - \bm{\mu}_{\cP}) + \ln \left( \frac{\det \bm{\Sigma}_{\cQ}}{\det \bm{\Sigma}_{\cP}} \right) \right)\\
&\leq \frac{1}{2} \cdot \left( (\bm{\mu}_{\cQ} - \bm{\mu}_{\cP})^\top \bm{\Sigma}_{\cQ}^{-1} (\bm{\mu}_{\cQ} - \bm{\mu}_{\cP}) + \| \bX \|_F^2 \right)
\end{align*}
where $\bX = \bm{\Sigma}_{\cQ}^{-1/2} \bm{\Sigma}_{\cP} \bm{\Sigma}_{\cQ}^{-1/2} - \bI_d$ with eigenvalues $\lambda_1, \ldots, \lambda_d$.
In particular, $\kl(\cP, \cQ) = \frac{1}{2} \| \bm{\mu}_{\cQ} - \bm{\mu}_{\cP} \|_2^2$ when $\bm{\Sigma}_{\cP} = \bm{\Sigma}_{\cQ} = \bI_d$ and $\kl(\cP, \cQ) \leq \frac{1}{2} \| \bX \|_F^2$ when $\bm{\mu}_{\cP} = \bm{\mu}_{\cQ}$.
\end{restatable}

\begin{definition}[Total variation (TV) distance]
\hspace{0pt}
For two continuous distributions $\cP$ and $\cQ$ over domain $\bX$, with density functions $f$ and $g$ respectively, $\tv(\cP, \cQ) = \frac{1}{2} \int_{\bx \in \bX} | f(\bx) - g(\bx) | \,dx$.
\end{definition}

\begin{theorem}[Pinsker's inequality]
\label{thm:pinsker}
If $\cP$ and $\cQ$ are two probability distributions on the same measurable space, then $\tv(\cP, \cQ) \leq \sqrt{\kl(\cP, \cQ)/2}$.
\end{theorem}

\subsection{Properties of Gaussians}

The following are standard results about empirical statistics of Gaussian samples.

\begin{lemma}[Lemma C.4 in \cite{ashtiani2020gaussian}; Corollary 5.50 in \cite{vershynin2010introduction}]
\label{lem:concentration-of-empirical-covariance}
Let $\bg_1, \ldots, \bg_n \sim N(\bm{0}, \bI_d)$ and let $0 < \eps < 1 < t$.
If $n \geq c_0 \cdot \frac{t^2 d}{\eps^2}$, for some absolute constant $c_0$, then
\[
\Pr \left( \left\| \frac{1}{n} \sum_{i=1}^n \bg_i \bg_i^\top - \bI_d \right\|_2 > \eps \right) \leq 2 \exp(-t^2 d)
\]
\end{lemma}

\begin{lemma}[Folklore; e.g.\ see Appendix C of \cite{ashtiani2020gaussian}]
\label{lem:empirical-is-good-with-linear-samples}
\label{lem:0-preliminaries-empirical-is-good-with-linear-samples}
Fix $\eps, \delta \in (0,1)$.
Given $2n$ i.i.d.\ samples $\bx_1, \ldots, \bx_{2n} \sim N(\bm{\mu}, \bm{\Sigma})$ for some unknown mean $\bm{\mu}$ and unknown covariance $\bm{\Sigma}$, define empirical mean and covariance as
\[
\wh{\bm{\mu}} = \frac{1}{2n} \sum_{i=1}^{2n} \bm{x}_i
\quad \text{and} \quad
\wh{\bm{\Sigma}} = \frac{1}{2n} \sum_{i=1}^n (\bm{x}_{2i} - \bm{x}_{2i-1}) (\bm{x}_{2i} - \bm{x}_{2i-1})^\top
\]
Then,
\begin{itemize}
    \item When $n \in \cO \left( \frac{d^2 + d \log(1/\delta)}{\eps^2} \right)$, we have $\Pr\left( \tv(N(\bm{\mu}, \bm{\Sigma}), N(\wh{\bm{\mu}}, \wh{\bm{\Sigma}})) \leq \eps \right) \geq 1 - \delta$
    \item When $n \in \cO \left( \frac{d + \sqrt{d \log(1/\delta)}}{\eps^2} \right)$, we have $\Pr\left( (\wh{\bm{\mu}} - \bm{\mu})^\top \bm{\Sigma}^{-1} (\wh{\bm{\mu}} - \bm{\mu}) \leq \eps^2 \right) \geq 1 - \delta$
\end{itemize}
\end{lemma}

\begin{restatable}[Properties of empirical covariance]{lemma}{knownpropertiesofempiricalcovariance}
\label{lem:known-properties-of-empirical-covariance}
Let $\wh{\bm{\Sigma}} \in \R^{d \times d}$ be the empirical covariance constructed from $n$ i.i.d.\ samples from $N(\bm{0}, \bm{\Sigma})$ for some unknown covariance $\bm{\Sigma}$.
Then,
\begin{itemize}
    \item When $n = d$, with probability 1, we have that $\wh{\bm{\Sigma}}$ and $\bm{\Sigma}$ share the same eigenspace.
    \item Let $\lambda_1 \leq \ldots \leq \lambda_d$ and $\wh{\lambda}_1 \leq \ldots \leq \wh{\lambda}_d$ be the eigenvalues of $\bm{\Sigma}$ and $\wh{\bm{\Sigma}}$ respectively.
    With probability at least $1 - \delta$, we have $\frac{\wh{\lambda}_1}{\lambda_1} \leq 1 + \cO \left( \sqrt{\frac{d + \log 1/\delta}{n}} \right)$.
\end{itemize}
\end{restatable}

\begin{restatable}{lemma}{noncentralchisquare}
\label{lem:non-central-chi-square}
Fix $n \geq 1$ and $d \geq 1$.
Suppose $\bm{\mu} \in \R^d$ is a hidden mean vector and we draw $n$ samples $\bx_1, \ldots, \bx_n \sim N(\bm{\mu}, \bI_d)$.
Define $\bz_n = \frac{1}{\sqrt{n}} \sum_{i=1}^{n} \bx_i$ and $y_n = \|\bz_n \|_2^2$.
Then,
\begin{enumerate}
    \item $y_n$ follows the \emph{non-central chi-squared distribution} $\chi^{\prime^2}_d(\lambda)$ for $\lambda = n \| \bm{\mu} \|_2^2$.
    This also implies that $\E[y_n] = d + \lambda$ and $\Var(y_n) = 2d + 4\lambda$.
    \item For any $t > 0$,
    \begin{align*}
    \Pr(y_n > d + \lambda + t)
    &\leq \exp\left(-\frac{d}{2} \left(\frac{t}{d + 2\lambda} - \log\left(1 + \frac{t}{d + 2\lambda}\right)\right)\right)\\
    &\leq \exp\left(-\frac{dt^2}{4(d + 2\lambda)(d + 2\lambda + t)}\right)
    \end{align*}
    \item For any $t \in (0, d + \lambda)$,
    \begin{align*}
    \Pr(y_n < d + \lambda - t)
    &\leq \exp\left(\frac{d}{2}\left(\frac{t}{d+2\lambda} + \log\left(1 - \frac{t}{d+2\lambda}\right)\right)\right)
    \\
    &\leq \exp\left(-\frac{dt^2}{4(d + 2\lambda)^2}\right)
    \end{align*}
\end{enumerate}
\end{restatable}

\begin{restatable}{lemma}{gaussianmaxconcentration}
\label{lem:gaussian-max-concentration}
Suppose $\bg_1, \ldots, \bg_n \sim N(0, \bI_d)$.
Then,
\[
\Pr \left( \left\| \sum_{i=1}^n \bg_i \right\|_\infty \geq \sqrt{2n \log \left( \frac{2d}{\delta} \right)} \right) \leq \delta
\]
\end{restatable}

%% file: identity-covariance.tex
\section{Identity covariance setting}
\label{sec:identity-covariance}

We begin by defining a parameterized sample count $m(d, \eps, \delta)$.
Then, we will state our \textsc{ApproxL1} algorithm and show how to use it according to the strategy outlined in \cref{sec:technical-overview-identity-covariance}.

\begin{definition}
\label{def:m-d-alpha-delta}
Fix any $d \geq 1$, $\eps > 0$, and $\delta \in (0,1)$.
We define $m(d, \eps, \delta) = n_{d,\eps} \cdot r_{\delta}$, where
\[
n_{d,\eps} = \left\lceil \frac{16 \sqrt{d}}{3 \eps^2} \right\rceil
\qquad \text{and} \qquad
r_{\delta} = 1 + \left\lceil \log \left( \frac{12}{\delta} \right) \right\rceil
\]
\end{definition}

Given samples from a $d$-dimensional isotropic Gaussian $N(\bm{\mu}, \bI_d)$ with unknown mean $\bm{\mu}$ and identity covariance, the \textsc{ApproxL1} algorithm partitions the $d$ coordinates into $w = \lceil d/k \rceil$ buckets each of length at most $k \in [d]$ and separately perform an exponential search to find the 2-approximation of the $\ell_2$ norm of each bucket by repeatedly invoking the tolerant tester from \cref{lem:tolerant-mean-tester}.
In the terminology of \cref{def:partitioning-scheme}, this is a partitioning scheme with $q=1$, $a=1$, and $b=1$.
Crucially, projecting the samples in $\R^{d}$ of $N(\bm{\mu}, \bI_d)$ into the subcoordinates of $\bB \subseteq [d]$ yields samples in $\R^{|\bB|}$ from $N(\bm{\mu}_{\bB}, \bI_{|\bB|})$ so we can obtain valid estimates using each of these marginals.
After obtaining the $\ell_2$ estimate of each bucket, we use \cref{lem:l1-l2-inequality} to obtain bounds on the $\ell_1$ and then combine them by summing up these estimates: if we have an $\eps$-multiplicative approximation of each bucket's $\ell_1$, then their sum will be an $\cO(\eps)$-multiplicative approximation of the entire $\bm{\mu}$ vector whenever the partition overlap parameters $a$ and $b$ of \cref{def:partitioning-scheme} are constants.

\begin{algorithm}[htb]
\begin{algorithmic}[1]
\caption{The \textsc{ApproxL1} algorithm.}
\label{alg:approxl1}
    \Statex \textbf{Input}: Error rate $\eps > 0$, failure rate $\delta \in (0,1)$, block size $k \in [d]$, lower bound $\alpha > 0$, upper bound $\zeta > 2 \alpha$, and i.i.d.\ samples $\cS$ from $N(\bm{\mu}, \bI_d)$
    \Statex \textbf{Output}: $\Fail$, $\OK$, or $\lambda \in \R$
    \State Define $w = \left\lceil {d}/{k} \right\rceil$ and $\delta' = \frac{\delta}{w \cdot \lceil \log_2 \zeta/\alpha \rceil}$
    \State Partition the index set $[d]$ into $w$ blocks:
    \[
    \bB_1 = \{1, \ldots, k\}, \bB_2 = \{k+1, \ldots, 2k\}, \ldots, \bB_w = \{k(w-1)+1, \ldots, d\}
    \]
    \For{$j \in \{1, \ldots, w\}$}
        \State Define $\cS_{j} = \{ \bx_{\bB_j} \in \R^{|\bB_j|} : \bx \in \cS \}$ as the samples projected to $\bB_j$
        \Comment{See \cref{def:projected-vector}}
        \State Initialize $o_j = \Fail$
        \For{$i = 1, 2, \ldots, \lceil \log_2 \zeta/\alpha \rceil$}
            \State Define $l_i = 2^{i-1} \cdot \alpha$
            \State Let \texttt{Outcome} be the output of the tolerant tester of \cref{lem:tolerant-mean-tester} using sample
            \Statex\hspace{\algorithmicindent}\hspace{\algorithmicindent}set $\cS_j$ with parameters $\eps_1 = l_i$, $\eps_2 = 2l_i$, and $\delta = \delta'$
            \If {\texttt{Outcome} is $\Accept$}
                \State Set $o_j = l_i$ and \textbf{break} \Comment{Escape inner loop for block $j$}
            \EndIf
        \EndFor
    \EndFor
    \If{there exists a $\Fail$ amongst $\{o_1, \ldots, o_w\}$}
        \State\Return $\Fail$
    \ElsIf{$4 \sum_{j=1}^w o_j^2 \leq \alpha^2$}
        \State\Return $\OK$ \Comment{Note: $o_j$ is an estimate for $\| \bm{\mu}_{\bB_j} \|_2$}
    \Else
        \Return $\lambda = 2 \sum_{j=1}^w \sqrt{|\bB_j|} \cdot o_j$ \Comment{$\lambda$ is an estimate for $\| \bm{\mu} \|_1$}
    \EndIf
\end{algorithmic}
\end{algorithm}

In \cref{sec:appendix-approxl1-guarantees}, we show that the \textsc{ApproxL1} algorithm has the following guarantees.

\begin{restatable}{lemma}{guaranteesofapproxlone}
\label{lem:guarantees-of-approxl1}
Let $\eps$, $\delta$, $k$, $\alpha$, and $\zeta$ be the input parameters to the \textsc{ApproxL1} algorithm (\cref{alg:approxl1}).
Given $m(k, \alpha, \delta')$ i.i.d.\ samples from $N(\bm{\mu}, \bI_d)$, the \textsc{ApproxL1} algorithm succeeds with probability at least $1 - \delta$ and has the following properties:
\begin{itemize}
    \item If \textsc{ApproxL1} outputs $\Fail$, then $\| \bm{\mu} \|_2 > \zeta / 2$.
    \item If \textsc{ApproxL1} outputs $\OK$, then $\| \bm{\mu} \|_2 \leq \alpha$.
    \item If \textsc{ApproxL1} outputs $\lambda \in \R$, then $\| \bm{\mu} \|_1 \leq \lambda \leq 2 \sqrt{k} \cdot \left( \lceil d/k \rceil \cdot \alpha + 2 \| \bm{\mu} \|_1 \right)$.
\end{itemize}
\end{restatable}

Now, suppose \textsc{ApproxL1} tells us that $\| \bm{\mu} \|_1 \leq r$.
We can then perform a constrained version of LASSO to search for a candidate $\wh{\bm{\mu}} \in \R^d$ using $\cO \left( \frac{r^2}{\eps^4} \log \frac{d}{\delta} \right)$ samples from $N(\bm{\mu}, \bI_d)$.

\begin{lemma}
\label{lem:constrained-LASSO-given-l1-upper-bound}
Fix $d \geq 1$, $r \geq 0$, and $\eps, \delta > 0$.
Given $\cO \left( \frac{r^2}{\eps^4} \log \frac{d}{\delta} \right)$ samples from $N(\bm{\mu}, \bI_d)$ for some unknown $\bm{\mu} \in \R^d$ with $\| \bm{\mu} \|_1 \leq r$, one can produce an estimate $\wh{\bm{\mu}} \in \R^d$ in $\poly(n, d)$ time such that $\tv(N(\bm{\mu}, \bI_d), N(\wh{\bm{\mu}}, \bI_d)) \leq \eps$ with success probability at least $1 - \delta$.
\end{lemma}
\begin{proof}
Suppose we get $n$ samples $\by_1, \ldots, \by_n \sim N(\bm{\mu}, \bI_d)$.
For $i \in [n]$, we can re-express each $\by_i$ as $\by_i = \bm{\mu} + \bg_i$ for some $\bg_i \sim N(\bm{0}, \bI_d)$.
Let us define $\wh{\bm{\mu}} \in \R^d$ as follows:
\begin{equation}
\label{eq:mean-estimation-program}
\wh{\bm{\mu}} = \argmin_{\| \bm{\beta} \|_1 \leq r} \frac{1}{n} \sum_{i=1}^n \| \by_i - \bm{\beta} \|_2^2
\end{equation}

By optimality of $\wh{\bm{\mu}}$ in \cref{eq:mean-estimation-program}, we have
\begin{equation}
\label{eq:mean-estimation-program-optimality}
\frac{1}{n} \sum_{i=1}^n \| \by_i - \wh{\bm{\mu}} \|_2^2 \leq \frac{1}{n} \sum_{i=1}^n \| \by_i - \bm{\mu} \|_2^2
\end{equation}

By expanding and rearranging \cref{eq:mean-estimation-program-optimality}, one can show (see \cref{sec:appendix-identity-covariance-deferred-derivation})
\begin{equation}
\label{eq:mean-estimation-program-optimality-after-manipulation}
\| \wh{\bm{\mu}} - \bm{\mu} \|_2^2
\leq \frac{2}{n} \langle \sum_{i=1}^n \bg_i, \wh{\bm{\mu}} - \bm{\mu} \rangle
\end{equation}

Therefore, with probability at least $1 - \delta$,
\begin{align*}
\| \wh{\bm{\mu}} - \bm{\mu} \|_2^2
&\leq \frac{2}{n} \langle \sum_{i=1}^n \bg_i, \wh{\bm{\mu}} - \bm{\mu} \rangle \tag{From \cref{eq:mean-estimation-program-optimality-after-manipulation}}\\
&\leq \frac{2}{n} \cdot \left\| \sum_{i=1}^n \bg_i \right\|_\infty \cdot \| \wh{\bm{\mu}} - \bm{\mu} \|_1 \tag{H\"{o}lder's inequality}\\
&\leq \frac{2}{n} \cdot \left\| \sum_{i=1}^n \bg_i \right\|_\infty \cdot \left( \| \wh{\bm{\mu}} \|_1 + \| \bm{\mu} \|_1 \right) \tag{Triangle inequality}\\
&\leq 4r \cdot \sqrt{\frac{2 \log \left( \frac{2d}{\delta} \right)}{n}} \tag{From \cref{lem:gaussian-max-concentration}, $\| \wh{\bm{\mu}} \|_1 \leq r$, and $\| \bm{\mu} \|_1 \leq r$}
\end{align*}
When $n = \frac{2 r^2 \log \frac{2d}{\delta}}{\eps^4} \in \cO \left( \frac{r^2}{\eps^4} \log \frac{d}{\delta} \right)$, we have $\| \wh{\bm{\mu}} - \bm{\mu} \|_2^2 \leq 4r \cdot \sqrt{\frac{2 \log \left( \frac{2d}{\delta} \right)}{n}} = 4 \eps^2$.
So, by \cref{thm:pinsker} and \cref{lem:kl-known-fact}, we see that
\[
\tv(N(\bm{\mu}, \bI_d), N(\wh{\bm{\mu}}, \bI_d))
\leq \sqrt{\frac{1}{2} \kl(N(\bm{\mu}, \bI_d), N(\wh{\bm{\mu}}, \bI_d))}
\leq \sqrt{\frac{1}{4} \| \bm{\mu} - \wh{\bm{\mu}} \|_2^2}
\leq \sqrt{\frac{4 \eps^2}{4}}
= \eps
\]
Finally, it is well-known that LASSO runs in $\poly(n, d)$ time.
\end{proof}

\begin{algorithm}[htb]
\begin{algorithmic}[1]
\caption{The \textsc{TestAndOptimizeMean} algorithm.}
\label{alg:testandoptimizemean}
    \Statex \textbf{Input}: Error rate $\eps > 0$, failure rate $\delta \in (0,1)$, parameter $\eta \in [0, \frac{1}{4}]$, and sample access to $N(\bm{\mu}, \bI_d)$
    \Statex \textbf{Output}: $\wh{\bm{\mu}} \in \R^d$
    
    \State Define $k = \lceil d^{4\eta}\rceil$, $\alpha = \eps\cdot d^{-(1-3\eta)/2}$, $\zeta = 4 \eps \cdot \sqrt{d}$, and $\delta' = \frac{\delta}{\lceil d/k \rceil \cdot \lceil \log_2 \zeta/\alpha \rceil}$ \Comment{Note: $\zeta > 2 \alpha$}
    \State Draw $m(k, \alpha, \delta')$ i.i.d.\ samples from $N(\bm{\mu}, \bI_d)$ and store it into a set $\cS$ \Comment{See \cref{def:m-d-alpha-delta}}

    \State Let \texttt{Outcome} be the output of the \textsc{ApproxL1} algorithm given $k$, $\alpha$, $\zeta$, and $\bS$ as inputs

    \If{\texttt{Outcome} is $\lambda \in \R$ and $\lambda < \eps \sqrt{d}$}
        \State Draw $n \in \wt{\cO}(\lambda^2/\eps^4)$ i.i.d.\ samples $\by_1, \ldots, \by_n \in \R^d$ from $N(\bm{\mu}, \bI_d)$

        \State \textbf{return} $\wh{\bm{\mu}} = \argmin_{\| \bm{\beta} \|_1 \leq \lambda} \frac{1}{n} \sum_{i=1}^n \| \by_i - \bm{\beta} \|_2^2$
        \Comment{See \cref{eq:mean-estimation-program}}
    \Else
        \State Draw $n \in \wt{\cO}(d/\eps^2)$ i.i.d.\ samples $\by_1, \ldots, \by_n \in \R^d$ from $N(\bm{\mu}, \bI_d)$

        \State \textbf{return} $\wh{\bm{\mu}} = \frac{1}{n} \sum_{i=1}^n \by_i$ \Comment{Empirical mean}
    \EndIf
\end{algorithmic}
\end{algorithm}

\identitycovariance*
\begin{proof}
Without loss of generality, we may assume that $\wt{\bm{\mu}} = \bm{0}$.
This is because we can pre-process all samples by subtracting $\wt{\bm{\mu}}$ to yield i.i.d.\ samples from $N(\bm{\mu}', \bI_d)$ where $\bm{\mu}' = \bm{\mu} - \wt{\bm{\mu}}$.
Suppose we solved this problem to produce $\wh{\bm{\mu}}'$ where $\tv(N(\bm{\mu}', \bI_d), N(\wh{\bm{\mu}}', \bI_d)) \leq 10 \eps$, we can then output $\wh{\bm{\mu}} = \wh{\bm{\mu}}' + \wt{\bm{\mu}}$ and see from data processing inequality that $\tv(N(\bm{\mu}, \bI_d), N(\wh{\bm{\mu}}, \bI_d)) = \tv(N(\bm{\mu}', \bI_d), N(\wh{\bm{\mu}}', \bI_d)) \leq 10 \eps$; see the coupling characterization of TV in \cite{devroye2018total}.

\paragraph{Correctness of $\wh{\bm{\mu}}$ output.}
Consider the \textsc{TestAndOptimizeMean} algorithm given in \cref{alg:testandoptimizemean}.
There are three possible outputs for $\wh{\bm{\mu}}$:
\begin{enumerate}
    \item $\wh{\bm{\mu}} = \bm{0}$, which can only happen when \texttt{Outcome} is $\OK$
    \item $\wh{\bm{\mu}} = \argmin_{\| \bm{\beta} \|_1 \leq \lambda} \frac{1}{n} \sum_{i=1}^n \| \by_i - \bm{\beta} \|_2^2$, which can only happen when \texttt{Outcome} is $\lambda \in \R$
    \item $\wh{\bm{\mu}} = \frac{1}{n} \sum_{i=1}^n \by_i$
\end{enumerate}
Conditioned on \textsc{ApproxL1} succeeding, with probability at least $1 - \delta$, we will show that $\tv(N(\bm{\mu}, \bI_d), N(\wh{\bm{\mu}}, \bI_d)) \leq \eps$ and failure probability at most $\delta$ in each of these cases, which implies the theorem statement.

\begin{enumerate}
    \item When \texttt{Outcome} is $\OK$, \cref{lem:guarantees-of-approxl1} tells us that $\| \bm{\mu} \|_2 \leq \alpha \leq \eps$, with failure probability at most $\delta$.
    So, by \cref{thm:pinsker} and \cref{lem:kl-known-fact}, we see that
    \[
    \tv(N(\bm{\mu}, \bI_d), N(\wh{\bm{\mu}}, \bI_d))
    \leq \sqrt{\frac{1}{2} \cdot \kl(N(\bm{\mu}, \bI_d), N(\wh{\bm{\mu}}, \bI_d))}
    = \sqrt{\frac{1}{4} \cdot \| \bm{\mu} - \bm{0} \|_2^2}
    \leq \sqrt{\frac{\eps^2}{4}}
    \leq \eps
    \]

    \item Using $r = \lambda$ as the upper bound, \cref{lem:constrained-LASSO-given-l1-upper-bound} tells us that $\tv(N(\bm{\mu}, \bI_d), N(\wh{\bm{\mu}}, \bI_d)) \leq \eps$ with failure probability at most $\delta$ when $\wt{\cO}(\lambda^2/\eps^4)$ i.i.d.\ samples are used.

    \item With $\wt{\cO}(d/\eps^2)$ samples, \cref{lem:empirical-is-good-with-linear-samples} tells us that $\tv(N(\bm{\mu}, \bI_d), N(\wh{\bm{\mu}}, \bI_d)) \leq \eps$ with failure probability at most $\delta$.
\end{enumerate}

\paragraph{Sample complexity used.}
By \cref{def:m-d-alpha-delta}, \textsc{ApproxL1} uses $|\bS| = m(k, \alpha, \delta') \in \wt{\cO}(\sqrt{k}/\alpha^2)$ samples to produce \texttt{Outcome}.
Then, \textsc{ApproxL1} further uses $\wt{\cO}(\lambda^2/\eps^4)$ samples or $\wt{\cO}(d/\eps^2)$ samples depending on whether $\lambda < \eps \sqrt{d}$.
So, \textsc{TestAndOptimizeMean} has a total sample complexity of
\begin{equation}
\label{eq:identitycovariance-samplecomplexity}
\wt{\cO} \left( \frac{\sqrt{k}}{\alpha^2} + \min \left\{ \frac{\lambda^2}{\eps^4}, \frac{d}{\eps^2} \right\} \right)
\end{equation}
Meanwhile, \cref{lem:guarantees-of-approxl1} states that $\| \bm{\mu} \|_1 \leq \lambda \leq 2 \sqrt{k} \cdot \left( \lceil d/k \rceil \cdot \alpha + 2 \| \bm{\mu} \|_1 \right)$ whenever \texttt{Outcome} is $\lambda \in \R$.
Since $(a+b)^2 \leq 2a^2 + 2b^2$ for any two real numbers $a,b \in \R$, we see that
\begin{equation}
\label{eq:identitycovariance-samplecomplexity2}
\frac{\lambda^2}{\eps^4}
\in \cO \left( \frac{k}{\eps^4} \cdot \left( \frac{d^2 \alpha^2}{k^2} + \| \bm{\mu} \|_1^2 \right) \right)
\subseteq \cO \left( \frac{d}{\eps^2} \cdot \left( \frac{d \alpha^2}{\eps^2 k} + \frac{k \cdot \| \bm{\mu} \|_1^2}{d \eps^2} \right) \right)
\end{equation}
Putting together \cref{eq:identitycovariance-samplecomplexity} and \cref{eq:identitycovariance-samplecomplexity2}, we see that the total sample complexity is
\[
\wt{\cO} \left( \frac{\sqrt{k}}{\alpha^2} + \frac{d}{\eps^2} \cdot \min \left\{ 1, \frac{d \alpha^2}{\eps^2 k} + \frac{k \cdot \| \bm{\mu} \|_1^2}{d \eps^2} \right\} \right)
\]
Recalling that $\bm{\mu}$ in the analysis above actually refers to the pre-processed $\bm{\mu} - \wt{\bm{\mu}}$, and that \textsc{TestAndOptimizeMean} sets $k = \lceil d^{4\eta} \rceil$ and $\alpha = \eps d^{-(1-3\eta)/2}$, with $0 \leq \eta \leq \frac{1}{4}$, the above expression simplifies to
\[
\wt{\cO} \left( \frac{d}{\eps^2} \cdot \left( d^{- \eta} + \min\{ 1, f(\bm{\mu}, \wt{\bm{\mu}}, d, \eta, \eps) \} \right) \right)
\]
where $f(\bm{\mu}, \wt{\bm{\mu}}, d, \eta, \eps) = \frac{\|\bm{\mu} - \wt{\bm{\mu}} \|_1^2}{d^{1 - 4 \eta}\eps^2}$.
\end{proof}

\paragraph{Remark on setting upper bound $\zeta$.}
As $\zeta$ only affects the sample complexity logarithmically, one may be tempted to use a larger value than $\zeta = 4 \eps \sqrt{d}$.
However, observe that running \textsc{ApproxL1} with a larger upper bound than $\zeta = 4 \eps \sqrt{d}$ would not be helpful since $\| \bm{\mu} \|_2 > \zeta/4$ whenever \textsc{ApproxL1} currently returns $\Fail$ and we have $\| \bm{\mu} \|_1 \leq \lambda$ whenever \textsc{ApproxL1} returns $\lambda \in \R$.
So,
$
\eps \sqrt{d}
= \zeta/4
< \| \bm{\mu} \|_2
\leq \| \bm{\mu} \|_1
\leq \lambda
$ and \textsc{TestAndOptimizeMean} would have resorted to using the empirical mean anyway.

%% file: general-covariance.tex
\section{General covariance setting}
\label{sec:general-covariance}

We will later define analogs of $m(d, \alpha, \delta)$ and \textsc{ApproxL1} from \cref{sec:identity-covariance} to the unknown covariance setting: $m'(d, \alpha, \delta)$ and \textsc{VectorizedApproxL1} respectively.
Then, after stating the guarantees of \textsc{VectorizedApproxL1}, we show how to use them according to the strategy outlined in \cref{sec:technical-overview-general-covariance}.
For the rest of this section, we assume that we get i.i.d.\ samples from $N(\bm{0}, \bm{\Sigma})$ and also that $\bm{\Sigma}$ is full rank.
These are without loss of generality for the following reasons:
\begin{itemize}
    \item Instead of a single sample from $N(\bm{\mu}, \bm{\Sigma})$, we will draw two samples $\bm{x}_1, \bm{x}_2 \sim N(\bm{\mu}, \bm{\Sigma})$ and consider $\bm{x}' = \frac{\bm{x}_1 + \bm{x}_2}{\sqrt{2}}$.
    One can check that $\bm{x}'$ is distributed according to $N(\bm{0}, \bm{\Sigma})$ and we only use a multiplicative factor of 2 additional samples, which is subsumed in the big-O.
    \item By \cref{lem:known-properties-of-empirical-covariance}, the empirical covariance constructed from $d$ i.i.d.\ samples of $N(\bm{0}, \bm{\Sigma})$ will have the same rank as $\bm{\Sigma}$ itself, with probability at least $1-\delta$.
    So, we can simply project and solve the problem on the full rank subspace of the empirical covariance matrix.
\end{itemize}

\subsection{The adjustments}
\label{sec:the-adjustments}

To begin, we elaborate on the adjustments mentioned in \cref{sec:technical-overview-general-covariance} to adapt the approach from the identity covariance setting to the unknown covariance setting.
The formal proofs of the following two adjustment lemmas are deferred to \cref{sec:appendix-the-adjustments}.

The first adjustment relates to performing a suitable preconditioning process using an additional $d$ samples so that we can subsequently argue that $\lambda_{\min}(\bm{\Sigma}) \geq 1$.
The idea is as follows: we will compute a preconditioning matrix $\bA$ using $d$ i.i.d.\ samples such that $\bA \bm{\Sigma} \bA$ has eigenvalues at least $1$, i.e. $\lambda_{\min}(\bA \bm{\Sigma} \bA) \geq 1$.
That is, $\| (\bA \bm{\Sigma} \bA)^{-1} \|_2 = \frac{1}{\lambda_{\min}(\bA \bm{\Sigma} \bA)} \leq 1$.
Then, we solve the problem treating $\bA \bm{\Sigma} \bA$ as our new $\bm{\Sigma}$.
This adjustment succeeds with probability at least $1 - \delta$ for any given $\delta \in (0,1)$ and is possible because, with probability 1, the empirical covariance $\wh{\bm{\Sigma}}$ formed by using $d$ i.i.d.\ samples would have the same eigenspace as $\bm{\Sigma}$, and so we would have a bound on the ratios between the minimum eigenvalues between $\wh{\bm{\Sigma}}$ and $\bm{\Sigma}$; see \cref{lem:known-properties-of-empirical-covariance}.

\begin{restatable}{lemma}{preconditioningadjustment}
\label{lem:preconditioning-adjustment}
For any $\delta \in (0,1)$, there is an explicit preconditioning process that uses $d$ i.i.d.\ samples from $N(\bm{0}, \bm{\Sigma})$ and succeeds with probability at least $1 - \delta$ in constructing a matrix $\bA \in \R^{d \times d}$ such that $\lambda_{\min}(\bA \bm{\Sigma} \bA) \geq 1$.
Furthermore, for any full rank PSD matrix $\wt{\bm{\Sigma}} \in \R^{d \times d}$, we have $\| (\bA \wt{\bm{\Sigma}} \bA)^{-1/2} \bA \bm{\Sigma} \bA (\bA \wt{\bm{\Sigma}} \bA)^{-1/2} - \bI_d \| = \| \wt{\bm{\Sigma}}^{-1/2} \bm{\Sigma} \wt{\bm{\Sigma}}^{-1/2} - \bI_d \|$.
\end{restatable}

The matrix $\bA$ in \cref{lem:preconditioning-adjustment} is essentially constructed by combining the eigenspace corresponding to ``large eigenvalues'' with a suitably upscaled eigenspace corresponding to ``small eigenvalues'' in the empirical covariance matrix obtained by $d$ i.i.d.\ samples and relying on \cref{lem:known-properties-of-empirical-covariance} for correctness arguments.

The second adjustment relates to showing that the partitioning idea also works for obtaining sample efficient $\ell_1$ estimates of $\vec(\bm{\Sigma} - \bI_d)$.
While an existence result suffices, we show that a simple probabilistic construction will in fact succeed with high probability.

\begin{restatable}{lemma}{probabilisticpartitioningconstruction}
\label{lem:probabilistic-partitioning-construction}
Fix dimension $d \geq 2$ and group size $k \leq d$.
Consider the $q = 2$ setting where $\bT \in \R^{d \times d}$ is a matrix.
Define $w = \frac{10 d(d-1) \log d}{k(k-1)}$.
Pick sets $\bB_1, \ldots, \bB_w$ each of size $k$ uniformly at random (with replacement) from all the possible $\binom{d}{k}$ sets.
With high probability in $d$, this is a $(q=2, d, k, a=1, b=\frac{30 (d-1) \log d}{(k-1)})$-partitioning scheme.
\end{restatable}

We can obtain a $(q=2, d, k, a=1, b = \cO(\frac{d \log d}{k}))$-partitioning scheme by repeating the construction of \cref{lem:probabilistic-partitioning-construction} until it satisfies required conditions.
Since it succeeds with high probability in $d$, we should not need many tries.
The key idea behind utilizing partitioning schemes is that the marginal over a subset of indices $\bB \subseteq [d]$ of a $d$-dimensional Gaussian with covariance matrix $\bm{\Sigma}$ has covariance matrix that is the principal submatrix $\bm{\Sigma}_{\bB}$ of $\bm{\Sigma}$.
So, if we can obtain a multiplicative $\alpha$-approximation of a collection of principal submatrices $\bm{\Sigma}_{\bB_1}, \ldots \bm{\Sigma}_{\bB_w}$ such that all cells of $\bm{\Sigma}$ are present, then we can obtain a multiplicative $\alpha$-approximation of $\bm{\Sigma}$ just like in \cref{sec:identity-covariance}.
Meanwhile, the $b$ parameter allows us to upper bound the overestimation factor due to repeated occurrences of any cell of $\bm{\Sigma}$.

\subsection{Following the approach from the identity covariance setting}

We begin by defining a parameterized sample count $m'(d, \eps, \delta)$, similar to \cref{def:m-d-alpha-delta}.

\begin{definition}
\label{def:m-prime-d-alpha-delta}
Fix any $d \geq 1$, $\eps > 0$, and $\delta \in (0,1)$.
We define $m'(d, \eps, \delta) = n'_{d,\eps} \cdot r_{\delta}$, where
\[
n'_{d,\eps} = \left\lceil 3200 d \cdot \max \left\{ \frac{1}{\eps^2}, \frac{1}{\eps}, 1 \right\} \right\rceil
\qquad \text{and} \qquad
r_{\delta} = 1 + \left\lceil \log \left( \frac{12}{\delta} \right) \right\rceil
\]
\end{definition}

The \textsc{VectorizedApproxL1} algorithm corresponds to \textsc{ApproxL1} in \cref{sec:identity-covariance}: it performs an exponential search to find the 2-approximation of the $\| \bm{\Sigma} - \bI_d \|_F^2$ by repeatedly invoking the tolerant tester from \cref{lem:tolerant-covariance-tester} and then utilize a suitable partitioning scheme to bound $\| \vec(\bm{\Sigma} - \bI_d) \|_1$; see \cref{lem:probabilistic-partitioning-construction} and the discussions below it.

\begin{algorithm}[htb]
\begin{algorithmic}[1]
\caption{The \textsc{VectorizedApproxL1} algorithm.}
\label{alg:vectorizedapproxl1}
    \Statex \textbf{Input}: Error rate $\eps > 0$, failure rate $\delta \in (0,1)$, block size $k \in [d]$, lower bound $\alpha > 0$, upper bound $\zeta > 2 \alpha$, and i.i.d.\ samples $\cS$ from $N(\bm{0}, \bm{\Sigma})$
    \Statex \textbf{Output}: $\Fail$ or $\lambda \in \R$

    \State Define $w = \frac{10 d(d-1) \log d}{k(k-1)}$, $\delta' = \frac{\delta}{w \cdot \lceil \log_2 \zeta/\alpha \rceil}$, and let $\bB_1, \ldots, \bB_w \subseteq [d]^2$ be a $(q=2, d, k, a=1, b = \cO(\frac{d \log d}{k}))$-partitioning scheme as per \cref{lem:probabilistic-partitioning-construction}

    \For{$j \in \{1, \ldots, w\}$}
        \State Define $\bS_{\bB_j} = \{ \bm{x}_{\bB_j} \in \R^{|\bB_j|} : \bm{x} \in \bS \}$ as the projected samples \Comment{See \cref{def:projected-vector}}
    
        \State Initialize $o_j = \Fail$
    
        \For{$i = 1, 2, \ldots, \lceil \log_2 \zeta/\alpha \rceil$}
            \State Define $l_i = 2^{i-1} \cdot \alpha$
    
            \State Let \texttt{Outcome} be the output of the tolerant tester of \cref{lem:tolerant-covariance-tester} using sample set $\cS_{\bB_j}$ with $\eps_1 = l_i$,
            \Statex\hspace{\algorithmicindent}\hspace{\algorithmicindent}$\eps_2 = 2l_i$, and $\delta = \delta'$
            
            \If {\texttt{Outcome} is $\Accept$}
                \State Set $o_j = l_i$ and \textbf{break} \Comment{Escape inner loop for block $j$}
            \EndIf
        \EndFor
    \EndFor

    \If{there exists a $\Fail$ amongst $\{o_1, \ldots, o_w\}$}
        \State \Return $\Fail$
    \Else
        \State \Return $\lambda = 2 \sum_{j=1}^w \sqrt{|\bB_j|} \cdot o_j$ \Comment{$\lambda$ is an estimate for $\| vec(\Sigma - \bB_d) \|_1$}
    \EndIf
\end{algorithmic}
\end{algorithm}

In \cref{sec:appendix-vectorizedapproxl1-guarantees}, we show that the \textsc{VectorizedApproxL1} algorithm has the guarantees given in \cref{lem:guarantees-of-vectorizedapproxl1}.

\begin{restatable}{lemma}{guaranteesofvectorizedapproxlone}
\label{lem:guarantees-of-vectorizedapproxl1}
Let $\eps$, $\delta$, $k$, $\alpha$, and $\zeta$ be the input parameters to the \textsc{VectorizedApproxL1} algorithm (\cref{alg:vectorizedapproxl1}).
Given $m(k, \alpha, \delta')$ i.i.d.\ samples from $N(\bm{\mu}, \bI_d)$, the \textsc{VectorizedApproxL1} algorithm succeeds with probability at least $1 - \delta$ and has the following properties:
\begin{itemize}
    \item If \textsc{VectorizedApproxL1} outputs $\Fail$, then $\| \bm{\Sigma} - \bI_d \|_F^2 > \zeta / 2$.
    \item If \textsc{VectorizedApproxL1} outputs $\lambda \in \R$, then
    \[
    \| \vec(\bm{\Sigma} - \bI_d) \|_1 \leq \lambda \leq 2 \sqrt{k} \cdot \left( \frac{10 d(d-1) \log d}{k(k-1)} \cdot \alpha + 2 \| \vec(\bm{\Sigma} - \bI_d) \|_1 \right)
    \]
\end{itemize}
\end{restatable}

Now, suppose \textsc{VectorizedApproxL1} tells us that $\| \vec(\bm{\Sigma} - \bI_d) \|_1 \leq r$.
We can then construct a SDP to search for a candidate $\wh{\bm{\Sigma}} \in \R^{d \times d}$ using $\cO \left( \frac{r^2}{\eps^4} \log \frac{1}{\delta} \right)$ samples from $N(\bm{0}, \bm{\Sigma})$.

\begin{lemma}
\label{lem:SDP-given-l1-upper-bound}
Fix $d \geq 1$, $r \geq 0$, and $\eps, \delta > 0$.
Given $\cO \left( \frac{r^2}{\eps^4} \log \frac{1}{\delta} + \frac{d + \sqrt{d \log(1/\delta)}}{\eps^2} \right)$ samples from $N(\bm{0}, \bm{\Sigma})$ for some unknown $\bm{\Sigma} \in \R^{d \times d}$ with $\| \vec(\bm{\Sigma} - \bI_d) \|_1 \leq r$, one can produce estimates $\wh{\bm{\mu}} \in \R^d$ and $\wh{\bm{\Sigma}} \in \R^{d \times d}$ in $\poly(n, d, \log(1/\eps))$ time such that $\tv(N(\bm{\mu}, \bm{\Sigma}), N(\wh{\bm{\mu}}, \wh{\bm{\Sigma}})) \leq \eps$ with success probability at least $1 - \delta$.
\end{lemma}
\begin{proof}
Suppose we get $n$ samples $\by_1, \ldots, \by_n \sim N(\bm{0}, \bm{\Sigma})$.
For $i \in [n]$, we can re-express each $\by_i$ as $\by_i = \bm{\Sigma}^{1/2} \bg_i$, for some $\bg_i \sim N(\bm{0}, \bI_d)$.
Let us define $\bT = \frac{1}{n} \sum_{i=1}^n \bg_i \bg_i^\top$ and $\bS = \frac{1}{n} \sum_{i=1}^n \by_i \by_i^\top = \bm{\Sigma}^{1/2} \left( \frac{1}{n} \sum_{i=1}^n \bg_i \bg_i^\top \right) \bm{\Sigma}^{1/2} = \bm{\Sigma}^{1/2} \bT \bm{\Sigma}^{1/2}$.

Let us define $\wh{\bm{\Sigma}} \in \R^{d \times d}$ as follows:
\begin{equation}
\label{eq:cov-estimation-program}
\wh{\bm{\Sigma}} = \argmin_{\substack{\text{$\bA \in \R^{d \times d}$ is p.s.d.}\\ \| \vec(\bA - \bI_d) \|_1 \leq r\\ \lambda_{\min}(\bA) \geq 1}} \sum_{i=1}^n \| \bA - \by_i \by_i^\top \|_F^2
\end{equation}
Observe that $\bm{\Sigma}$ is a feasible solution to \cref{eq:cov-estimation-program}.
We show in \cref{sec:appendix-cov-estimation-program} that \cref{eq:cov-estimation-program} is a semidefinite program (SDP) that is polynomial time solvable.

Since $\bm{\Sigma}$ and $\wh{\bm{\Sigma}}$ are symmetric p.s.d.\ matrices, observe that
\begin{align*}
\sum_{i=1}^n \| \wh{\bm{\Sigma}} - \by_i \by_i^\top \|_F^2
&= \sum_{i=1}^n \| \wh{\bm{\Sigma}} - \bm{\Sigma}^{1/2} \bg_i \bg_i^\top \bm{\Sigma}^{1/2} \|_F^2 \tag{Since $\by_i = \bm{\Sigma}^{1/2} \bg_i$}\\
&= \sum_{i=1}^n \Tr \left( \left( \wh{\bm{\Sigma}} - \bm{\Sigma}^{1/2} \bg_i \bg_i^\top \bm{\Sigma}^{1/2} \right)^\top \left( \wh{\bm{\Sigma}} - \bm{\Sigma}^{1/2} \bg_i \bg_i^\top \bm{\Sigma}^{1/2} \right) \right) \tag{Since $\| \bA \|_F^2 = \Tr(\bA^\top \bA)$ for any matrix $\bA$}\\
&= \sum_{i=1}^n \Tr \left( \wh{\bm{\Sigma}}^2 - 2 \bg_i \bg_i^\top \bm{\Sigma}^{1/2} \wh{\bm{\Sigma}} \bm{\Sigma}^{1/2} + \bg_i \bg_i^\top \bm{\Sigma} \bg_i \bg_i^\top \bm{\Sigma} \right) \tag{Expanding and applying cyclic property of trace}
\end{align*}

Similarly, by replacing $\wh{\bm{\Sigma}}$ with $\bm{\Sigma}$, we see that
\[
\sum_{i=1}^n \| \bm{\Sigma} - \by_i \by_i^\top \|_F^2
= \sum_{i=1}^n \Tr \left( \bm{\Sigma}^2 - 2 \bg_i \bg_i^\top \bm{\Sigma}^2 + \bg_i \bg_i^\top \bm{\Sigma} \bg_i \bg_i^\top \bm{\Sigma} \right)
\]

By standard SDP results (e.g.\ see \cite{vandenberghe1996semidefinite,freund2004introduction,gartner2012approximation}), \cref{eq:cov-estimation-program} can be solved optimally up to up to additive $\eps$ in the objective function.
We show explicitly in \cref{sec:appendix-cov-estimation-program} that our problem can be transformed into a SDP and be solved in $\poly(n, d, \log(1/\eps))$ time.
Since we solve up to additive $\eps$ in the objective function, we have
\begin{equation}
\label{eq:sdp-additive-error}
\sum_{i=1}^n \| \wh{\bm{\Sigma}} - \by_i \by_i^\top \|_F^2
\leq \eps + \sum_{i=1}^n \| \bm{\Sigma} - \by_i \by_i^\top \|_F^2    
\end{equation}
which implies that
\[
\sum_{i=1}^n \Tr \left( \wh{\bm{\Sigma}}^2 - 2 \bg_i \bg_i^\top \bm{\Sigma}^{1/2} \wh{\bm{\Sigma}} \bm{\Sigma}^{1/2} + \bg_i \bg_i^\top \bm{\Sigma} \bg_i \bg_i^\top \bm{\Sigma} \right)
\leq
\eps +
\sum_{i=1}^n \Tr \left( \bm{\Sigma}^2 - 2 \bg_i \bg_i^\top \bm{\Sigma}^2 + \bg_i \bg_i^\top \bm{\Sigma} \bg_i \bg_i^\top \bm{\Sigma} \right)
\]

Cancelling the common $\bg_i \bg_i^\top \bm{\Sigma} \bg_i \bg_i^\top \bm{\Sigma}$ term and rearranging, we get
\begin{equation}
\label{eq:program-trace-inequality}
\Tr \left( \wh{\bm{\Sigma}}^2 - \bm{\Sigma}^2 \right)
\leq
\frac{\eps}{n} +
\frac{2}{n} \sum_{i=1}^n \Tr \left( \bg_i \bg_i^\top \left( \bm{\Sigma}^{1/2} \wh{\bm{\Sigma}} \bm{\Sigma}^{1/2} - \bm{\Sigma}^2 \right) \right)
\end{equation}

Therefore,
\begin{align*}
\| \wh{\bm{\Sigma}} - \bm{\Sigma} \|_F^2
&= \Tr \left( \left( \wh{\bm{\Sigma}} - \bm{\Sigma} \right)^\top \left( \wh{\bm{\Sigma}} - \bm{\Sigma} \right) \right)\\
&= \Tr \left( \wh{\bm{\Sigma}}^2 - 2 \wh{\bm{\Sigma}} \bm{\Sigma} + \bm{\Sigma}^2  \right)\\
&\leq \frac{\eps}{n} +\frac{2}{n} \sum_{i=1}^n \Tr \left( \bg_i \bg_i^\top \left( \bm{\Sigma}^{1/2} \wh{\bm{\Sigma}} \bm{\Sigma}^{1/2} - \bm{\Sigma}^2 \right) - \wh{\bm{\Sigma}} \bm{\Sigma} + \bm{\Sigma}^2 \right) \tag{Add $2 \bm{\Sigma}^2 - 2 \wh{\bm{\Sigma}} \bm{\Sigma}$ to both sides of \cref{eq:program-trace-inequality}}\\
&= \frac{\eps}{n} +\frac{2}{n} \sum_{i=1}^n \Tr \left( \left( \bg_i \bg_i^\top - \bI_d \right) \cdot \left( \bm{\Sigma}^{1/2} \wh{\bm{\Sigma}} \bm{\Sigma}^{1/2} - \bm{\Sigma}^2 \right) \right) \tag{Since $\Tr(\wh{\bm{\Sigma}} \bm{\Sigma}) = \Tr(\bm{\Sigma}^{1/2} \wh{\bm{\Sigma}} \bm{\Sigma}^{1/2})$}\\
&= \frac{\eps}{n} +2 \cdot \Tr \left( \left( \bm{\Sigma}^{1/2} \wh{\bm{\Sigma}} - \bm{\Sigma}^{1/2} \bm{\Sigma} \right) \cdot \bm{\Sigma}^{1/2} \cdot \left( \left( \frac{1}{n} \sum_{i=1}^n \bg_i \bg_i^\top \right) - \bI_d \right) \right) \tag{Rearranging with cyclic property of trace}\\
&\leq \frac{\eps}{n} +2 \cdot \left\| \vec \left( \bm{\Sigma} \wh{\bm{\Sigma}} - \bm{\Sigma}^2 \right) \right\|_1 \cdot \left\| \left( \frac{1}{n} \sum_{i=1}^n \bg_i \bg_i^\top \right) - \bI_d \right\|_2 \tag{By \cref{lem:trace-inequality} with $\bA = \bm{\Sigma}^{1/2} \wh{\bm{\Sigma}} - \bm{\Sigma}^{1/2} \bm{\Sigma}$, $\bB = \bm{\Sigma}^{1/2}$, and $\bC = \left( \frac{1}{n} \sum_{i=1}^n \bg_i \bg_i^\top \right) - \bI_d$}
\end{align*}

Recall that $\bT = \frac{1}{n} \sum_{i=1}^n \bg_i \bg_i^\top$ and \cref{lem:concentration-of-empirical-covariance} tells us that $\Pr \left( \left\| \bT - \bI_d \right\|_2 > \eps \right) \leq 2 \exp(-t^2 d)$ when the number of samples $n = \frac{c_0}{\eps^2} \log \frac{2}{\delta}$, for some absolute constant $c_0$.
So, to complete the proof, it suffices to upper bound $\left\| \vec \left( \bm{\Sigma} \wh{\bm{\Sigma}} - \bm{\Sigma}^2 \right) \right\|_1$.
Consider the following:

\begin{align*}
\left\| \vec \left( \bm{\Sigma} \wh{\bm{\Sigma}} - \bm{\Sigma}^2 \right) \right\|_1
&= \left\| \vec \left( ( \bI_d - \bm{\Sigma} ) ( \bm{\Sigma} - \wh{\bm{\Sigma}} ) - \bm{\Sigma} + \wh{\bm{\Sigma}} \right) \right\|_1\\
&\leq \left\| \vec ( \bI_d - \bm{\Sigma} ) \right\|_1 \cdot \left\| \vec ( \bm{\Sigma} - \wh{\bm{\Sigma}} ) \right\|_1 + \left\| \vec ( \wh{\bm{\Sigma}} - \bm{\Sigma} ) \right\|_1 \tag{By \cref{lem:vectorized-inequalities}}\\
&= \left( \left\| \vec ( \bI_d - \bm{\Sigma} ) \right\|_1 + 1 \right) \cdot \left\| \vec ( \wh{\bm{\Sigma}} - \bI_d + \bI_d - \bm{\Sigma} ) \right\|_1 \tag{Rearranging and adding 0}\\
&\leq \left( \left\| \vec \left( \bI_d - \bm{\Sigma} \right) \right\|_1 + 1 \right) \cdot \left( \| \vec ( \wh{\bm{\Sigma}} - \bI_d ) \|_1 + \left\| \vec ( \bI_d - \bm{\Sigma} ) \right\|_1 \right) \tag{By \cref{lem:vectorized-inequalities}}\\
&\leq (r + 1) \cdot 2r \tag{Since $\| \vec ( \bI_d - \bm{\Sigma} ) \|_1 \leq r$ and $\left\| \vec ( \wh{\bm{\Sigma}} - \bI_d ) \right\|_1 \leq r$}
\end{align*}

When $\frac{2}{\eps} \leq n$ and $n \in \cO \left( \frac{r^2}{\eps^4} \log \frac{1}{\delta} \right)$, the following holds with probability at least $1 - \delta$:
\[
\| \wh{\bm{\Sigma}} - \bm{\Sigma} \|_F^2
\leq \frac{\eps}{n} + 2 \cdot \left\| \vec \left( \bm{\Sigma} \wh{\bm{\Sigma}} - \bm{\Sigma}^2 \right) \right\|_1 \cdot \left\| \bT - \bI_d \right\|_2
\leq \frac{\eps}{n} + 4r(r+1) \cdot \left\| \bT - \bI_d \right\|_2
\leq \frac{\eps}{n} + \frac{\eps^2}{2}
\leq \eps^2
\]

Now, \cref{lem:empirical-is-good-with-linear-samples} tells us that the empirical mean $\wh{\bm{\mu}}$ formed using $\cO \left( \frac{d + \sqrt{d \log(1/\delta)}}{\eps^2} \right)$ samples satisfies $(\wh{\bm{\mu}} - \bm{\mu})^\top \bm{\Sigma}^{-1} (\wh{\bm{\mu}} - \bm{\mu}) \leq \eps^2$, with failure probability at most $\delta$.
So,
\begin{align*}
&\; \kl(N(\wh{\bm{\mu}}, \wh{\bm{\Sigma}}), N(\bm{\mu}, \bm{\Sigma}))\\
= &\; \frac{1}{2} \cdot \left( \Tr(\bm{\Sigma}^{-1} \wh{\bm{\Sigma}}) - d + (\bm{\mu} - \wh{\bm{\mu}})^\top \bm{\Sigma}^{-1} (\bm{\mu} - \wh{\bm{\mu}}) + \ln \left( \frac{\det \bm{\Sigma}}{\det \wh{\bm{\Sigma}}} \right) \right)\\
\leq &\; \frac{1}{2} \cdot \left( (\bm{\mu} - \wh{\bm{\mu}})^\top \bm{\Sigma}^{-1} (\bm{\mu} - \wh{\bm{\mu}}) + \| \bm{\Sigma}^{-1/2} \wh{\bm{\Sigma}} \bm{\Sigma}^{-1/2} - \bI_d \|_F^2 \right) \tag{By \cref{lem:kl-known-fact}}\\
= &\; \frac{1}{2} \cdot \left( (\bm{\mu} - \wh{\bm{\mu}})^\top \bm{\Sigma}^{-1} (\bm{\mu} - \wh{\bm{\mu}}) + \| \wh{\bm{\Sigma}} \bm{\Sigma}^{-1} - \bI_d \|_F^2 \right) \tag{By \cref{lem:rotating-norm}}\\
\leq &\; \frac{1}{2} \cdot \left( \eps^2 + \| \wh{\bm{\Sigma}} \bm{\Sigma}^{-1} - \bI_d \|_F^2 \right) \tag{Since $(\wh{\bm{\mu}} - \bm{\mu})^\top \bm{\Sigma}^{-1} (\wh{\bm{\mu}} - \bm{\mu}) \leq \eps$, with probability at least $1 - \delta$}\\
\leq &\; \frac{1}{2} \cdot \left( \eps^2 + \| \bm{\Sigma}^{-1} \|_2^2 \cdot \| \wh{\bm{\Sigma}} - \bm{\Sigma} \|_F^2 \right) \tag{Submultiplicativity of Frobenius norm}\\
\leq &\; \frac{1}{2} \cdot \left( \eps^2 + \| \wh{\bm{\Sigma}} - \bm{\Sigma} \|_F^2 \right) \tag{Since $\| \bm{\Sigma}^{-1} \|_2 = \frac{1}{\lambda_{\min}(\bm{\Sigma})} \leq 1$}\\
\leq &\; \frac{1}{2} \cdot \left( \eps^2 + \eps^2 \right) \tag{From above, with probability at least $1 - \delta$}\\
= &\; \eps^2
\end{align*}
By union bound, the above events jointly hold with probability at least $1 - 2 \delta$.
Thus, by symmetry of TV distance and \cref{thm:pinsker}, we see that
\[
\tv(N(\bm{\mu}, \bI_d), N(\wh{\bm{\mu}}, \bI_d))
= \tv(N(\wh{\bm{\mu}}, \bI_d), N(\bm{\mu}, \bI_d))
\leq \sqrt{\frac{1}{2} \kl(N(\wh{\bm{\mu}}, \bI_d), N(\bm{\mu}, \bI_d))}
\leq \sqrt{\eps^2}
= \eps
\]
The claim holds by repeating the same argument after scaling $\delta$ by an appropriate constant.
\end{proof}

\begin{algorithm}[htb]
\begin{algorithmic}[1]
\caption{The \textsc{TestAndOptimizeCovariance} algorithm.}
\label{alg:testandoptimizecovariance}
    \Statex \textbf{Input}: Error rate $\eps > 0$, failure rate $\delta \in (0,1)$, parameter $\eta \in [0, 1]$, and sample access to $N(\bm{0}, \bm{\Sigma})$
    \Statex \textbf{Output}: $\wh{\bm{\Sigma}} \in \R^{d \times d}$

    \State Define $k = \lceil d^{\eta} \rceil$, $\alpha = \eps d^{-(2 - \eta)/2}$, $\zeta = 4 \eps d$, and $\delta' = \frac{\delta}{w \cdot \lceil \log_2 \zeta/\alpha \rceil}$ \Comment{Note: $\zeta > 2 \alpha$}

    \State Draw $m'(k, \alpha, \delta')$ i.i.d.\ samples from $N(\bm{0}, \bm{\Sigma})$ and store it into a set $\cS$ \Comment{See \cref{def:m-prime-d-alpha-delta}}

    \State Let \texttt{Outcome} be the output of the \textsc{VectorizedApproxL1} algorithm given $\eps$, $\delta$, $k$, $\alpha$, $\zeta$, and $\bS$ as inputs

    \If{\texttt{Outcome} is $\lambda \in \R$ and $\lambda < \eps d$}
        \State Draw $n \in \wt{\cO}(\lambda^2/\eps^4)$ i.i.d.\ samples $\by_1, \ldots, \by_n \in \R^d$ from $N(\bm{0}, \bI_d)$

        \State \Return $\wh{\bm{\Sigma}} = \argmin_{\substack{\text{$\bA \in \R^{d \times d}$ is p.s.d.}\\ \| \vec(\bA - \bI_d) \|_1 \leq \lambda\\ \lambda_{\min}(\bA) \geq 1}} \sum_{i=1}^n \| \bA - \by_i \by_i^\top \|_F^2$ \Comment{See \cref{eq:cov-estimation-program}}
    \Else
        \State Draw $2n \in \wt{\cO}(d^2/\eps^2)$ i.i.d.\ samples $\by_1, \ldots, \by_{2n} \in \R^d$ from $N(\bm{0}, \bI_d)$

        \State \Return $\wh{\bm{\Sigma}} = \frac{1}{2n} \sum_{i=1}^{2n} (\by_{2i} - \by_{2i-1}) (\by_{2i} - \by_{2i-1})^\top$ \Comment{Empirical covariance}
    \EndIf
\end{algorithmic}
\end{algorithm}

\generalcovariance*
\begin{proof}
Without loss of generality, we may assume that $\wt{\bm{\Sigma}} = \bI_d$.
This is because we can pre-process all samples by pre-multiplying $\wt{\bm{\Sigma}}^{-1/2}$ each of them to yield i.i.d.\ samples from $N(\bm{\mu}, \wt{\bm{\Sigma}}^{-1/2} \bm{\Sigma} \wt{\bm{\Sigma}}^{-1/2})$ and then post-process the estimated $\wh{\bm{\Sigma}}$ by outputting $\wt{\bm{\Sigma}}^{1/2} \wh{\bm{\Sigma}} \wt{\bm{\Sigma}}^{1/2}$ instead.

\paragraph{Correctness of $\wh{\bm{\Sigma}}$ output.}
Consider the \textsc{TestAndOptimizeCovariance} algorithm given in \cref{alg:testandoptimizecovariance}.
Using the empirical mean $\wh{\bm{\mu}} = \frac{1}{n} \sum_{i=1}^n \by_i$ formed by $\cO \left( \frac{d + \sqrt{d \log(1/\delta)}}{\eps^2} \right) \subseteq \wt{\cO}(d/\eps^2)$ samples, \cref{lem:empirical-is-good-with-linear-samples} tells us that $(\wh{\bm{\mu}} - \bm{\mu})^\top \bm{\Sigma}^{-1} (\wh{\bm{\mu}} - \bm{\mu}) \leq \eps$ with probability at least $1 - \delta$.
There are two possible outputs for $\wh{\bm{\Sigma}}$:
\begin{enumerate}
    \item $\wh{\bm{\Sigma}} = \argmin_{\substack{\text{$\bA \in \R^{d \times d}$ is p.s.d.}\\ \| \vec(\bA - \bI_d) \|_1 \leq r\\ \lambda_{\min}(\bA) \geq 1 \leq 1}} \sum_{i=1}^n \| \bA - \by_i \by_i^\top \|_F^2$, which can only happen when \texttt{Outcome} is $\lambda \in \R$
    \item $\wh{\bm{\Sigma}} = \frac{1}{2n} \sum_{i=1}^{2n} (\by_{2i} - \by_{2i-1}) (\by_{2i} - \by_{2i-1})^\top$
\end{enumerate}
Conditioned on \textsc{VectorizedApproxL1} succeeding, with probability at least $1 - \delta$, we will now show that $\tv(N(\bm{\mu}, \bm{\Sigma}), N(\wh{\bm{\mu}}, \wh{\bm{\Sigma}})) \leq \eps$ and failure probability at most $2 \delta$ in each of these cases, which implies the theorem statement as we can repeat the argument by scaling $\eps$ and $\delta$ by appropriate constants.

\textbf{Case 1:}
Using $r = \lambda$ as the upper bound, \cref{lem:SDP-given-l1-upper-bound} tells us that $\tv(N(\bm{\mu}, \bm{\Sigma}), N(\wh{\bm{\mu}}, \wh{\bm{\Sigma}})) \leq \eps$ with failure probability at most $\delta$ when $\wt{\cO}(\frac{\lambda^2}{\eps^4} + \frac{d}{\eps^2})$ i.i.d.\ samples are used.

\textbf{Case 2:}
With $\wt{\cO}(d^2/\eps^2)$ samples, \cref{lem:empirical-is-good-with-linear-samples} tells us that $\tv(N(\bm{\mu}, \bm{\Sigma}), N(\wh{\bm{\mu}}, \wh{\bm{\Sigma}})) \leq \eps$ with failure probability at most $\delta$.

\paragraph{Sample complexity used.}
By \cref{def:m-prime-d-alpha-delta}, \textsc{VectorizedApproxL1} uses $|\bS| = m'(k, \alpha, \delta') \in \wt{\cO}(k/\alpha^2)$ samples to produce \texttt{Outcome}.
Then, \textsc{VectorizedApproxL1} further uses $\wt{\cO}(\lambda^2/\eps^4)$ samples or $\wt{\cO}(d^2/\eps^2)$ samples depending on whether $\lambda < \eps d$.
So, \textsc{TestAndOptimizeCovariance} has a total sample complexity of
\begin{equation}
\label{eq:generalcovariance-samplecomplexity}
\wt{\cO} \left( \frac{k}{\alpha^2} + \min \left\{ \frac{\lambda^2}{\eps^4} + \frac{d}{\eps^2}, \frac{d^2}{\eps^2} \right\} \right)
\subseteq \wt{\cO} \left( \frac{k}{\alpha^2} + \frac{d}{\eps^2} + \min \left\{ \frac{\lambda^2}{\eps^4}, \frac{d^2}{\eps^2} \right\} \right)
\end{equation}
Meanwhile, \cref{lem:guarantees-of-vectorizedapproxl1} states that
\[
\| \vec(\bm{\Sigma} - \bI_d) \|_1 \leq \lambda \leq 2 \sqrt{k} \cdot \left( \frac{10 d(d-1) \log d}{k(k-1)} \cdot \alpha + 2 \| \vec(\bm{\Sigma} - \bI_d) \|_1 \right)
\]
whenever \texttt{Outcome} is $\lambda \in \R$.
Since $(a+b)^2 \leq 2a^2 + 2b^2$ for any two real numbers $a,b \in \R$, we see that
\begin{equation}
\label{eq:generalcovariance-samplecomplexity2}
\frac{\lambda^2}{\eps^4}
\in \cO \left( \frac{k}{\eps^4} \cdot \left( \frac{d^4 \alpha^2}{k^4} + \| \vec(\bm{\Sigma} - \bI_d) \|_1^2 \right) \right)
\subseteq \cO \left( \frac{d^2}{\eps^2} \cdot \left( \frac{d^2 \alpha^2}{\eps^2 k^3} + \frac{k \cdot \| \vec(\bm{\Sigma} - \bI_d) \|_1^2}{d^2 \eps^2} \right) \right)
\end{equation}
Putting together \cref{eq:generalcovariance-samplecomplexity} and \cref{eq:generalcovariance-samplecomplexity2}, we see that the total sample complexity is
\[
\wt{\cO} \left( \frac{k}{\alpha^2} + \frac{d}{\eps^2} + \frac{d^2}{\eps^2} \cdot \min \left\{ 1, \frac{d^2 \alpha^2}{\eps^2 k^3} + \frac{k \cdot \| \vec(\bm{\Sigma} - \bI_d) \|_1^2}{d^2 \eps^2} \right\} \right)
\]
Recalling that $\bm{\Sigma}$ in the analysis above actually refers to the pre-processed $\wt{\bm{\Sigma}}^{-1/2} \bm{\Sigma} \wt{\bm{\Sigma}}^{-1/2}$, and that \textsc{TestAndOptimizeCovariance} sets $k = \lceil d^{\eta} \rceil$, $\alpha = \eps d^{-(2 - \eta)/2}$, with $0 \leq \eta \leq 1$, the above expression simplifies to
\[
\wt{\cO} \left( \frac{d^2}{\eps^2} \cdot \left( d^{- \eta} + \min \left\{ 1, f(\bm{\Sigma}, \wt{\bm{\Sigma}}, d, \eta, \eps) \right\} \right) \right)
\]
where $f(\bm{\Sigma}, \wt{\bm{\Sigma}}, d, \eta, \eps) = \frac{\| \vec( \wt{\bm{\Sigma}}^{-1/2} \bm{\Sigma} \wt{\bm{\Sigma}}^{-1/2} - \bI_d) \|_1^2}{d^{2 - \eta} \eps^2}$.
\end{proof}

\paragraph{Remark on setting upper bound $\zeta$.}
As $\zeta$ only affects the sample complexity logarithmically, one may be tempted to use a larger value than $\zeta = 4 \eps d$.
However, observe that running \textsc{VectorizedApproxL1} with a larger upper bound than $\zeta = 4 \eps \sqrt{d}$ would not be helpful since $\| \bm{\Sigma} - \bI_d \|_F^2 > \zeta / 2$ whenever \textsc{VectorizedApproxL1} currently returns $\Fail$ and we have $\| \vec(\bm{\Sigma} - \bI_d) \|_1 \leq \lambda$ whenever \textsc{VectorizedApproxL1} returns $\lambda \in \R$.
So,
$
\eps d
= \zeta/4
< \| \bm{\Sigma} - \bI_d \|_F^2
= \| \vec(\bm{\Sigma} - \bI_d) \|_2
\leq \| \vec(\bm{\Sigma} - \bI_d) \|_1
\leq \lambda
$ and \textsc{TestAndOptimizeMean} would have resorted to using the empirical mean anyway.

\paragraph{Remark about early termination without the optimization step.}
If there is no $\Fail$ amongst $\{o_1, \ldots, o_w\}$ and $4b \sum_{j=1}^w o_j^2 \leq \eps^2$ after Line 9 of \textsc{VectorizedApproxL1}, then we could have just output $\wh{\bm{\Sigma}} = \bI_d$ without running the optimization step.
This ie because since $4b \sum_{j=1}^w o_j^2 \leq \eps^2$ would imply $\| \bm{\Sigma} - \bI_d \|_F^2 \leq \eps^2$ via
\[
\| \bm{\Sigma} - \bI_d \|_F^2
\leq b \cdot \sum_{j=1}^w \| \bm{\Sigma}_{\bB_j} - \bI_d \|_F^2
\leq b \cdot \sum_{j=1}^w (2 o_j)^2
\leq \eps^2
\]
Meanwhile, \cref{lem:empirical-is-good-with-linear-samples} tells us that $(\wh{\bm{\mu}} - \bm{\mu})^\top \bm{\Sigma}^{-1} (\wh{\bm{\mu}} - \bm{\mu}) \leq \eps^2$.
Therefore, we see that
\begin{align*}
&\; \kl(N(\wh{\bm{\mu}}, \wh{\bm{\Sigma}}), N(\bm{\mu}, \bm{\Sigma}))\\
= &\; \frac{1}{2} \cdot \left( \Tr(\bm{\Sigma}^{-1} \wh{\bm{\Sigma}}) - d + (\bm{\mu} - \wh{\bm{\mu}})^\top \bm{\Sigma}^{-1} (\bm{\mu} - \wh{\bm{\mu}}) + \ln \left( \frac{\det \bm{\Sigma}}{\det \wh{\bm{\Sigma}}} \right) \right)\\
\leq &\; \frac{1}{2} \cdot \left( (\bm{\mu} - \wh{\bm{\mu}})^\top \bm{\Sigma}^{-1} (\bm{\mu} - \wh{\bm{\mu}}) + \| \bm{\Sigma}^{-1/2} \wh{\bm{\Sigma}} \bm{\Sigma}^{-1/2} - \bI_d \|_F^2 \right) \tag{By \cref{lem:kl-known-fact}}\\
= &\; \frac{1}{2} \cdot \left( (\bm{\mu} - \wh{\bm{\mu}})^\top \bm{\Sigma}^{-1} (\bm{\mu} - \wh{\bm{\mu}}) + \| \bm{\Sigma} - \bI_d \|_F^2 \right) \tag{Since $\wh{\bm{\Sigma}} = \bI_d$}\\
\leq &\; \frac{1}{2} \cdot \left( \eps^2 + \| \bm{\Sigma} - \bI_d \|_F^2 \right) \tag{Since $(\wh{\bm{\mu}} - \bm{\mu})^\top \bm{\Sigma}^{-1} (\wh{\bm{\mu}} - \bm{\mu}) \leq \eps$, with probability at least $1 - \delta$}\\
\leq &\; \frac{1}{2} \cdot \left( \eps^2 + \alpha^2 \right) \tag{Since $\| \bm{\Sigma} - \bI_d \|_F^2 \leq \alpha^2$, with probability at least $1 - \delta$}\\
\leq &\; \frac{1}{2} \cdot \left( \eps^2 + \eps^2 \right) \tag{since $\alpha = \frac{\eps k}{d} \leq \eps$ as $k \leq d$}\\
= &\; \eps^2
\end{align*}
Thus, by symmetry of TV distance and \cref{thm:pinsker}, we see that
\[
\tv(N(\bm{\mu}, \bm{\Sigma}), N(\wh{\bm{\mu}}, \wh{\bm{\Sigma}}))
= \tv(N(\wh{\bm{\mu}}, \wh{\bm{\Sigma}}), N(\bm{\mu}, \bm{\Sigma}))
\leq \sqrt{\frac{1}{2} \kl(N(\wh{\bm{\mu}}, \wh{\bm{\Sigma}}), N(\bm{\mu}, \bm{\Sigma}))}
\leq \sqrt{\eps^2}
= \eps
\]

%% file: lower-bound.tex
\section{Lower Bounds}
\label{sec:lower-bounds}

\subsection{Learning the mean given advice}

\cref{thm:mainresultlowerbound} and \cref{thm:mainresultlowerboundcovariance} are implied by \cref{lem:implymainresultlowerbound} and \cref{lem:implymainresultlowerboundcovariance} respectively.
For the proofs of both our lower bounds, we use the following corollary of Fano's inequality.

\begin{lemma}[Lemma 6.1 of \cite{ashtiani2020gaussian}]
\label{lem:cover_fano}
Let $\kappa: \R \rightarrow \R$ be a function and let $\cF$ be a class of distributions such that, for all $\eps > 0$, there exist distributions $f_1,\ldots,f_M \in \cF$ such that
\[
\kl(f_i, f_j) \leq \kappa(\eps) \mbox{ and } \tv(f_i, f_j) > 2 \eps \,\,\forall i \neq j \in [M]
\]
Then any method that learns $\cF$ to within total variation distance $\eps$ with probability $\geq 2/3$ has sample complexity $\Omega\left(\frac{\log M}{\kappa(\eps) \log(1/\eps)}\right)$.
\end{lemma}

\begin{lemma}
\label{lem:implymainresultlowerbound}
Fix $\eps \leq \frac{1}{400}$.
Suppose we are given sample access to $N(\bm{\mu}, \bI_d)$ for some unknown $\bm{\mu} \in \R^d$, and an advice $\wt{\bm{\mu}} \in \R^d$.
Then, any algorithm that $(\eps,\frac{2}{3})$-PAC learns $N(\bm{\mu}, \bI_d)$ requires $\wt{\Omega} \left( \max \left\{ \frac{\|\bm{\mu} - \wt{\bm{\mu}}\|_1^2}{\eps^4}, \frac{d}{\eps^2} \right\} \right)$ samples.
In particular, when $\|\bm{\mu} - \wt{\bm{\mu}}\|_1 \geq \eps \sqrt{d}$, then $\wt{\Omega}(\frac{d}{\eps^2})$ samples are necessary.
\end{lemma}
\begin{proof}
Without loss of generality, we can consider $\wt{\bm{\mu}} = 0$ since we can easily sample from $N(\bm{\mu} - \wt{\bm{\mu}}, \bI_d)$ by sampling from $N(\bm{\mu},  \bI_d)$ and subtracting $\wt{\bm{\mu}}$ from each sample. Let $\widehat{\bm{\mu}}$ denote the mean-estimate produced by the learning algorithm.
Note that the TV distance between $N(\bm{\mu}, \bI_d)$ and $N(\bm{\mu}^\prime, \bI_d)$ is $\Theta(\|\bm{\mu} - \bm{\mu}^\prime\|_2)$, specifically in $\left[\frac{\|\bm{\mu} - \bm{\mu}^\prime\|_2}{200}, \frac{\|\bm{\mu} - \bm{\mu}^\prime\|_2}{2}\right]$, by Theorem 1.2 and Proposition 2.1 of \cite{devroye2018total}, as long as $\|\bm{\mu} - \bm{\mu}^\prime\|_2 \leq 1$.
Also, we have $\kl(N(\bm{\mu}, \bI_d), N(\bm{\mu}^\prime, \bI_d)) = \frac{1}{2}\|\bm{\mu} - \bm{\mu}^\prime\|_2^2$.

Now, for an arbitrary $\eps$ sufficiently small, we want to choose a large $M$ such that we can show the existence of $M$ vectors $\bm{\mu}_1,\ldots,\bm{\mu}_M \in \R^{d}$ with
\begin{equation}
\label{eqn:lb_mean_cover_cond}
\|\bm{\mu}_i - \wt{\bm{\mu}}\|_1 = \lambda  \mbox{ and } \|\bm{\mu}_i - \bm{\mu}_j\|_2 \in \left[\eps, 2\eps\right] \mbox{ for each } i \neq j \in [M].
\end{equation}
As long as $\eps \leq \frac{1}{2}$, \cref{eqn:lb_mean_cover_cond} would imply that (i) the pairwise total variation distance is at least $\frac{\eps}{200}$, and (ii) the KL divergence is at most $2 \eps^2$ (in both directions).
Suppose we take $\eps^\prime = \frac{\eps}{400}$, so that the pairwise total variation is at least $2\eps^\prime = \frac{\eps}{200}$ and the pairwise KL divergence is at most $\kappa(\eps^\prime) = 2\eps^2$ for $\kappa(x) = 2 \cdot 400^2 \cdot x^2$.
Then, \cref{lem:cover_fano} will give a sample complexity lower bound of $\Omega\left(\frac{\log M}{\kappa(\eps^\prime) \log(1/\eps^\prime)}\right) = \Omega\left(\frac{\log M}{\eps^2 \log(1/\eps)}\right)$ for learning in total variation up to $\eps^\prime$ given advice.

Our randomized construction of the covering set is as follows: 
Choose a $0 < k < d$ to be fixed later.
The first $k$ coordinates of each $\bm{\mu}_i$ are set to $\frac{\lambda}{k} \cdot \bm{v}_i$ for some $\bm{v}_i \in \{\pm 1\}^k$ and the remaining $d-k$ coordinates are set identically to $0$. 
Then, by construction, $\|\bm{\mu}_i - \wt{\bm{\mu}}\|_1 = \|\bm{\mu}_i\|_1 = k\left(\frac{\lambda}{k}\right) = \lambda$ for each $\bm{\mu}_i$, and $\|\bm{\mu}_i - \bm{\mu}_j\|_2 = \left(2\frac{\lambda}{k}\right) \sqrt{\|\bm{v}_i - \bm{v}_j\|_0}$.

By the Gilbert-Varshamov bound, for any $k > 4$, there exists a code $C \subseteq \{0,1\}^k$ with pairwise Hamming distance $\in [k/4, k]$ such that $|C| \geq \frac{2^{k-1}}{\sum_{i=0}^{k/4-1} \binom{k}{i}} \geq \frac{2^{k-1}}{\left(\frac{4ek}{k}\right)^{k/4}} \geq 2^{\Omega(k)}$ (the second inequality via Stirling's approximation). 
We can thus show the existence of our $\{\bm{v}_1,\ldots,\bm{v}_M\} \subseteq \{\pm 1\}^{k}$ by taking $M = 2^{\Omega(k)}$ to get the code $C$ as above and applying the transformation $(x_1,\ldots,x_k) \mapsto ((-1)^{x_1},\ldots,(-1)^{x_k})$ to each binary codeword in $C$.

Thus, from the above construction, we will have $\|\bm{\mu}_i - \bm{\mu}_j\|_2 \in \left[\frac{\lambda}{\sqrt{k}}, \frac{2\lambda}{\sqrt{k}}\right]$ for each $i \neq j \in [M]$.
To satisfy \cref{eqn:lb_mean_cover_cond}, we can choose $k = \left\lceil\frac{\lambda^2}{\eps^2}\right\rceil$.
By the above discussion, this gives us a sample complexity lower bound of $\Omega\left(\frac{\lambda^2}{\eps^4\log(1/\eps)}\right)$ for learning Gaussian means given advice $\wt{\bm{\mu}}$ with $\|\bm{\mu} - \wt{\bm{\mu}}\|_1 = \lambda$.
\end{proof}

\begin{lemma}
\label{lem:implymainresultlowerboundcovariance}
Suppose we are given advice $\wt{\bm{\Sigma}} \in \R^{d \times d}$ which is symmetric and positive-definite, and sample access to $N(\bm{0}, \bm{\Sigma})$ for some unknown symmetric positive-definite $\bm{\Sigma} \in \R^{d \times d}$, with only the constraint that $\|\vec\left(\wt{\bm{\Sigma}}^{-\frac{1}{2}} \bm{\Sigma} \wt{\bm{\Sigma}}^{-\frac{1}{2}} -\bI_d\right)\|_1 \leq \Delta$. Then, any algorithm that $(\eps,\frac{2}{3})$-PAC learns $N(\bm{0}, \bm{\Sigma})$ in total variation requires {$\widetilde{\Omega}\left(\min\left(\frac{d^2}{\eps^2}, \frac{\Delta^2}{\eps^4}\right)\right)$} samples.
\end{lemma}
\begin{proof}
Without loss of generality, we can assume $\wt{\bm{\Sigma}} =\bI_d$ since, we can transform the input samples from $N(\bm{0}, \bm{\Sigma})$ as $\bm{x} \mapsto \wt{\bm{\Sigma}}^{-\frac{1}{2}} \bm{x}$ to get samples from $N\left(\bm{0}, \wt{\bm{\Sigma}}^{-\frac{1}{2}} \bm{\Sigma} \wt{\bm{\Sigma}}^{-\frac{1}{2}}\right)$, so that the advice quality in the transformed space (with advice taken to be $\bI_d$) would be $\|\vec\left(\bI_d \left(\wt{\bm{\Sigma}}^{-\frac{1}{2}} \bm{\Sigma} \wt{\bm{\Sigma}}^{-\frac{1}{2}}\right)\bI_d -\bI_d\right)\|_1$, which is equal to the original advice quality $\|\vec\left(\wt{\bm{\Sigma}}^{-\frac{1}{2}} \bm{\Sigma} \wt{\bm{\Sigma}}^{-\frac{1}{2}} -\bI_d\right)\|_1$.

To use \cref{lem:cover_fano}, we need to construct a set of $M$ distributions $f_1,\ldots,f_M$ with $f_i \triangleq N(\bm{0},\bm{\Sigma}_i)$ such that
\begin{enumerate}[(i)]
\item Advice quality $\|\vec\left(\bm{\Sigma}_i -\bI_d\right)\|_1 \leq \Delta$ for each $i \in [M]$,
\item the pairwise KL divergence $\kl(f_i \| f_j) \leq  \cO(\eps^2)$,
\item the the pairwise TV distance $\tv(f_i, f_j) \geq \Omega(\eps)$, and
\item $\log M \geq \Omega\left(\min\left(d^2 ,\frac{\Delta^2}{\eps^2}\right)\right)$.
\end{enumerate}

If we can construct such a family, \cref{lem:cover_fano} would give us a sample complexity lower bound of
\[
\Omega\left(\min\left(\frac{d^2}{\eps^2\log(1/\eps)}, \frac{\Delta^2}{\eps^4\log(1/\eps)}\right)\right)
\]
to $(\eps, 2/3)$-PAC learn the true disitribution, even given advice with quality $\leq \Delta$.

The following claim is a Gilbert-Varshamov like bound on the existence of large sets of $s$-tuples of $[N]$ with pairwise distance $\geq (1-\frac{1}{40})s$.

\begin{lemma}\label{lem:gilbert_varshamov_type_sets}
For any $N \geq 200$ and $s > 0$, there exists $A = \{A_1,\ldots,A_M\} \subseteq [N]^s$ with $M \geq N^{\Omega(s)}$ such that for all pairs $i \neq j \in [M]$, $A_i$ and $A_j$ agree on $\leq s/40$ coordinates. 
\end{lemma}
And the following claim follows from \cite{ashtiani2020gaussian}, Lemma 6.4.

\begin{lemma}\label{lem:good_orthonormal_matrix_slices}For $p \geq 10$, there exist $N \geq 2^{\Omega(p^2)}$ matrices $\bU_1,\ldots,\bU_N \in \R^{p \times (p/10)}$ such that the columns of each $\bU_i$ are the first $p \times 10$ columns of a $p \times p$ orthogonal matrix, and for each pair $i \neq j \in [N]$, $\|\bU_i^\top \bU_j\|_F^2 \leq p/20$.
\end{lemma}

Let $d$ be a positive integer such that $d$ is a multiple of $10$, and either $d^2$ is a multiple of $10\left\lceil \frac{\Delta^2}{\eps^2} \right\rceil$ or $d^2 < 10\left\lceil \frac{\Delta^2}{\eps^2} \right\rceil$.  For every $\eps > 0$ and $\Delta \geq \eps$, there exist infinitely many choices of $d$ that satisfy these criteria. Take $p = \min\left(d, \frac{10}{d}\left\lceil \frac{\Delta^2}{\eps^2} \right\rceil\right)$. Then, we will have $d = s \cdot p$  for some integer $s \geq 1$, and $p$ will be a multiple of $10$. Also take $\mu = \frac{\Delta}{d} \sqrt\frac{10}{p} \lesssim \eps/\sqrt{d}$ (using $p \leq (10/d)\lceil\Delta^2/\eps^2\rceil$).

 Let $\bU_1, \ldots, \bU_N \in \R^{p \times (p/10)}$ be the $N \geq 2^{\Omega(p^2)}$ matrices as in \Cref{lem:good_orthonormal_matrix_slices}.

Also let $A_1,\ldots,A_M$ denote the $M \geq 2^{\Omega(p^2 s)} = 2^{\Omega\left(\min\left(d^2, \Delta^2/\eps^2\right)\right)}$ tuples in $[N]^s$ which agree pairwise only on $\leq s/40$ coordinates as guaranteed by \Cref{lem:gilbert_varshamov_type_sets}.

Then, we use the construction in Theorem 6.3 of \cite{ashtiani2020gaussian} block-wise to construct each covariance matrix $\bm{\Sigma}_i, i \in [M]$. We construct each $\bm{\Sigma}_i = \begin{bmatrix}\bm{\Sigma}_{i,1} & 0 & \cdots & 0\\ 0 & \bm{\Sigma}_{i,2} & \cdots & 0\\0 & 0 & \cdots & \bm{\Sigma}_{i,s}\end{bmatrix} \in \R^{d \times d}$, where each $\bm{\Sigma}_{i,j} = \bI_p + \mu \bU_{A_i(j)} \bU_{A_i(j)}^\top \in \R^{p \times p}$.

By \Cref{lem:good_orthonormal_matrix_slices}, each $\bm{\Sigma}_{i,j} - \bI_p = \mu \bU_{A_i(j)} \bU_{A_i(j)}^\top$ has $p/10$ eigenvalues which are equal to $\mu$ and the remaining $p-p/10$ eigenvalues equal to $0$. Thus, we have $\|\bm{\Sigma}_i - \bI_d\|_1 = \sum_{j=1}^{s} \|\bm{\Sigma}_{i,j} - I_p\|_1$ (decomposing the sum in the $\ell_1$ norm definition)
$\leq \sum_{j=1}^{s} p \cdot \|\bm{\Sigma}_{i,j} - \bI_p\|_{\textrm{F}}$ (by Cauchy-Schwarz)
$\leq s \cdot p \cdot \sqrt{\frac{p}{10} \mu^2}$(since Frobenius norm = Schatten-$2$ norm)
$\leq d \mu \sqrt{p/10} \leq \Delta$ (substituting $s p  = d$ and $\mu = (\Delta/d) \sqrt{10/p}$).

We have $\bm{\Sigma}_{i,j}^{-1} =\bI_p - \frac{\mu}{1 + \mu} \bU_{A_i(j)} \bU_{A_i(j)}^\top$ by construction of $\bU_1,\ldots,\bU_N$.
By a similar calculation as in Theorem 6.3 of \cite{ashtiani2020gaussian}, we have $\kl(f_i, f_j) = \frac{1}{2} \Tr(\bm{\Sigma}_i^{-1} \bm{\Sigma}_j - \bI_d) = \sum_{r=1}^{s} \frac{1}{2}\Tr(\bm{\Sigma}_{i,r}^{-1} \bm{\Sigma}_{j,r} - \bI_p) \leq s \mu^{2} \frac{p}{10} \leq \frac{d}{10} \mu^2 \leq \cO(\eps^2)$ (using $\mu \lesssim \eps/\sqrt{d}$).

By using a similar argument as in Lemma 6.6 of \cite{ashtiani2020gaussian}, we can lower bound the pairwise TV distance. By Theorem 1.1 in \cite{devroye2018total}, we have $\tv(f_i, f_j) \geq \Theta\left(\min\{1, \|\bm{\Sigma}_i^{-1/2} \bm{\Sigma}_j \bm{\Sigma}_i^{-1/2} - \bI_d\|_{\rm F}\}\right)$. Since $\sigma_{\rm min}(\bm{\Sigma}_i^{-1/2}) = (1 + \mu)^{-1/2} = \Theta(1)$ when $\eps \leq \sqrt{d}$, we have $\tv(f_i, f_j) \geq \Omega(\eps)$ when $\|\bm{\Sigma}_i - \bm{\Sigma}_j\|_{\rm F} \geq \Omega(\eps)$. We then have
\begin{equation*}
\begin{aligned}
\|\bm{\Sigma}_i - \bm{\Sigma}_j\|_{\rm F}^2 &= \sum_{r=1}^{s} \|\bm{\Sigma}_{i,r} - \bm{\Sigma}_{j,r}\|_{\rm F}^2 = \sum_{r=1}^{s} \mu^2 \|\bU_{A_i(r)} \bU_{A_i(r)}^\top - \bU_{A_j(r)} \bU_{A_j(r)}^\top\|_{\rm F}^2\\
&= \sum_{r=1}^{s} \mu^2 \Tr\left(\left(\bU_{A_i(r)} \bU_{A_i(r)}^\top - \bU_{A_j(r)} \bU_{A_j(r)}^\top\right) \left(\bU_{A_i(r)} \bU_{A_i(r)}^\top - \bU_{A_j(r)} \bU_{A_j(r)}^\top\right)\right)\\
&= \sum_{r=1}^{s} \mu^2 \left(\Tr(\bU_{A_i(r)} U_{A_i(r)}^\top) + \Tr(\bU_{A_j(r)} U_{A_j(r)}^\top) - 2 \|\bU_{A_i(r)}^\top \bU_{A_j(r)}\|_{\rm F}^2\right)\\
&\hspace*{10pt}(\mbox{using $\bU_{A_i(r)}^\top \bU_{A_i(r)} = \bI_{p/10}$, cyclic property of trace, and $\|A\|_{\rm F}^2 = \Tr(A^\top A)$})\\
&= \cdot \frac{2 \mu^2 d}{10} - 2 \mu^2 \sum_{r=1}^{s}  \|\bU_{A_i(r)}^\top \bU_{A_j(r)}\|_{\rm F}^2\,\,(\mbox{using $\Tr(\bU_n \bU_n^\top) = \tfrac{p}{10} \,\forall\,n \in [N],$ $d = sp$})\\
&\geq \frac{2\mu^2 d}{10} - 2\mu^2 \left(\#\{A_i(r) = A_j(r)\} \frac{p}{10} + \#\{A_i(r) \neq A_j(r)\} \frac{p}{20}\right)\\
&\hspace*{10pt}\mbox{(using $\bU_n^\top \bU_n = \bI_{p/10}$ and $\|\bU_m^\top \bU_n\|_{\rm F}^2 \leq p/20$ for $m \neq n$ by \Cref{lem:good_orthonormal_matrix_slices})}\\
&\geq  \frac{2\mu^2 d}{10} - 2\mu^2\left(\frac{sp}{40} - \frac{sp}{20}\right) \geq \frac{9\mu^2d}{40} \geq \Omega(\eps^2)\,\,(\mbox{using \Cref{lem:gilbert_varshamov_type_sets}}).
\end{aligned}
\end{equation*}
\end{proof}

%% file: experiments.tex
\section{Experiments}
\label{sec:experiments}

Here, we explore the sample complexity gains in the identity covariance setting when one is given high quality advice, specifically the benefits of performing the optimization in line 6 of \cref{alg:testandoptimizemean} versus returning 
the empirical mean as in line 9.
As such, we do \emph{not} invoke \textsc{ApproxL1} but instead explore how to $\|\bm{\mu} - \wh{\bm{\mu}}_{\textsc{ALG}}\|_2$ behaves as a function of $\|\bm{\mu} - \wh{\bm{\mu}}\|_1$ and number of samples, where \textsc{ALG} is either our \textsc{TestAndOptimize} approach or simply computing the empirical mean.
Our simple script is given in \cref{sec:appendix-python-code}.

We perform two experiments on multivariate Gaussians of dimension $d = 500$ while varying two parameters: sparsity $s \in [d]$ and advice quality $q \in \R_{\geq 0}$.
In both experiments, the difference vector $\bm{\mu} - \wt{\bm{\mu}} \in \R^d$ is generated with random $\pm q/s$ values in the first $s$ coordinates and zeros in the remaining $d-s$ coordinates.
In the first experiment (see \cref{fig:exp1}), we fix $q = 50$ and vary $s \in \{100, 200, 300\}$.
In the second experiment (see \cref{fig:exp2}), we fix $s = 100$ and vary $q \in \{0.1, 20, 30\}$.
In both experiments, we see that \textsc{TestAndOptimize} beats the empirical mean estimate in terms of incurred $\ell_2$ error (which translate directly to $\tv$), with the diminishing benefits as $q$ or $s$ increases.
While running our experiments, we observed an interesting phenomenon: the rate of improvement does not worsen as $\ell_1$ increases if we fixed the $\ell_0$ sparsity; see \cref{fig:exp3}.
As such, it would be interesting to show theoretical guarantees with advice error in the $\ell_0$-norm.

For computational efficiency, we solve the LASSO optimization in its Lagrangian form
\[
\wh{\bm{\mu}} = \argmin_{\bm{\beta} \in \R^d} \frac{1}{n} \sum_{i=1}^n \| \by_i - \bm{\beta} \|_2^2 + \lambda \|\bm{\beta}\|_1
\]
using the \texttt{LassoLarsCV} method in \texttt{scikit-learn}, instead of the equivalent penalized form.
The value of the hyperparameter $\lambda$ is chosen using 5-fold cross-validation.

\begin{figure}[htb]
\centering
\begin{subfigure}[b]{0.3\textwidth}
    \centering
    \includegraphics[width=\textwidth]{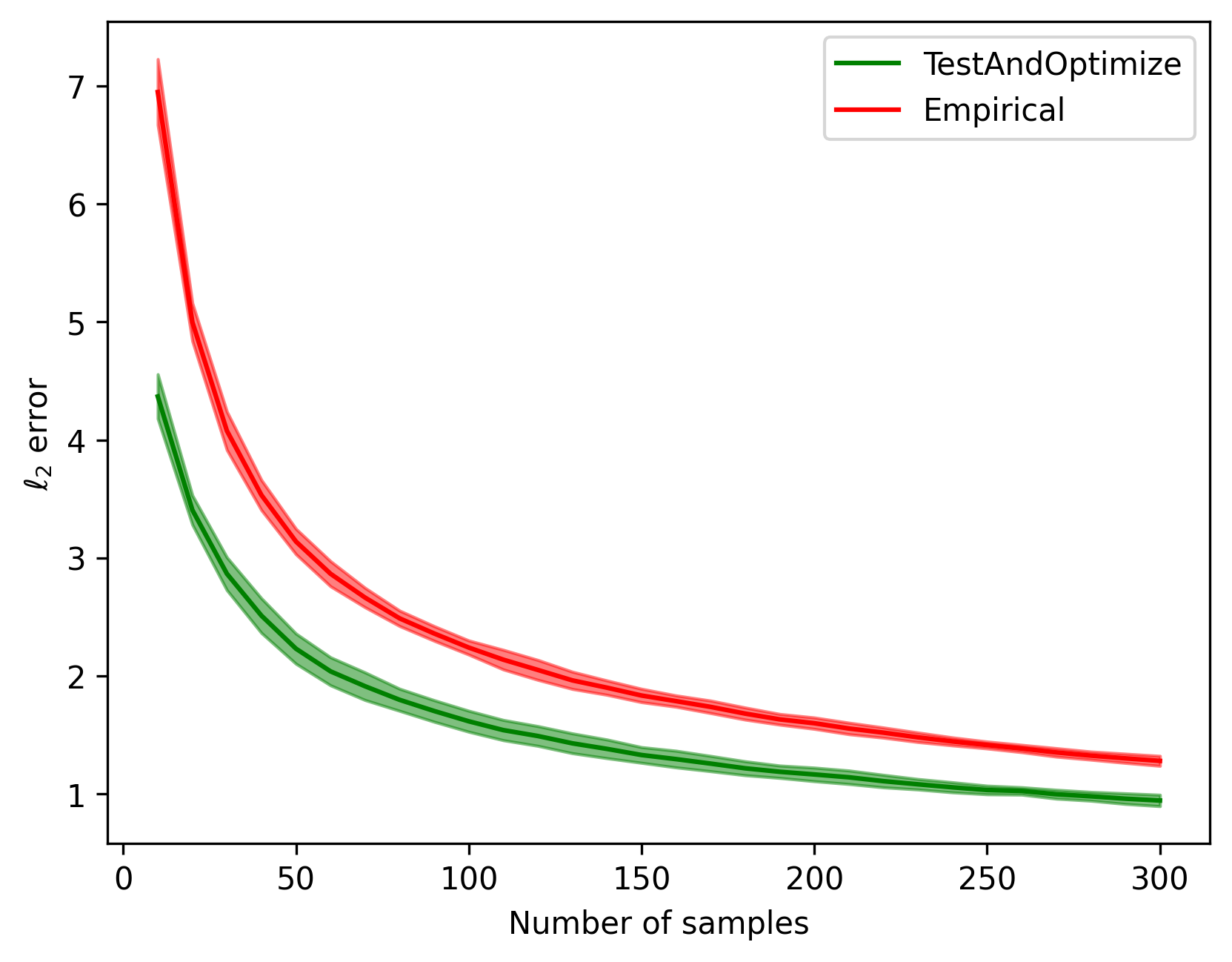}
\end{subfigure}
\hfill
\begin{subfigure}[b]{0.3\textwidth}
    \centering
    \includegraphics[width=\textwidth]{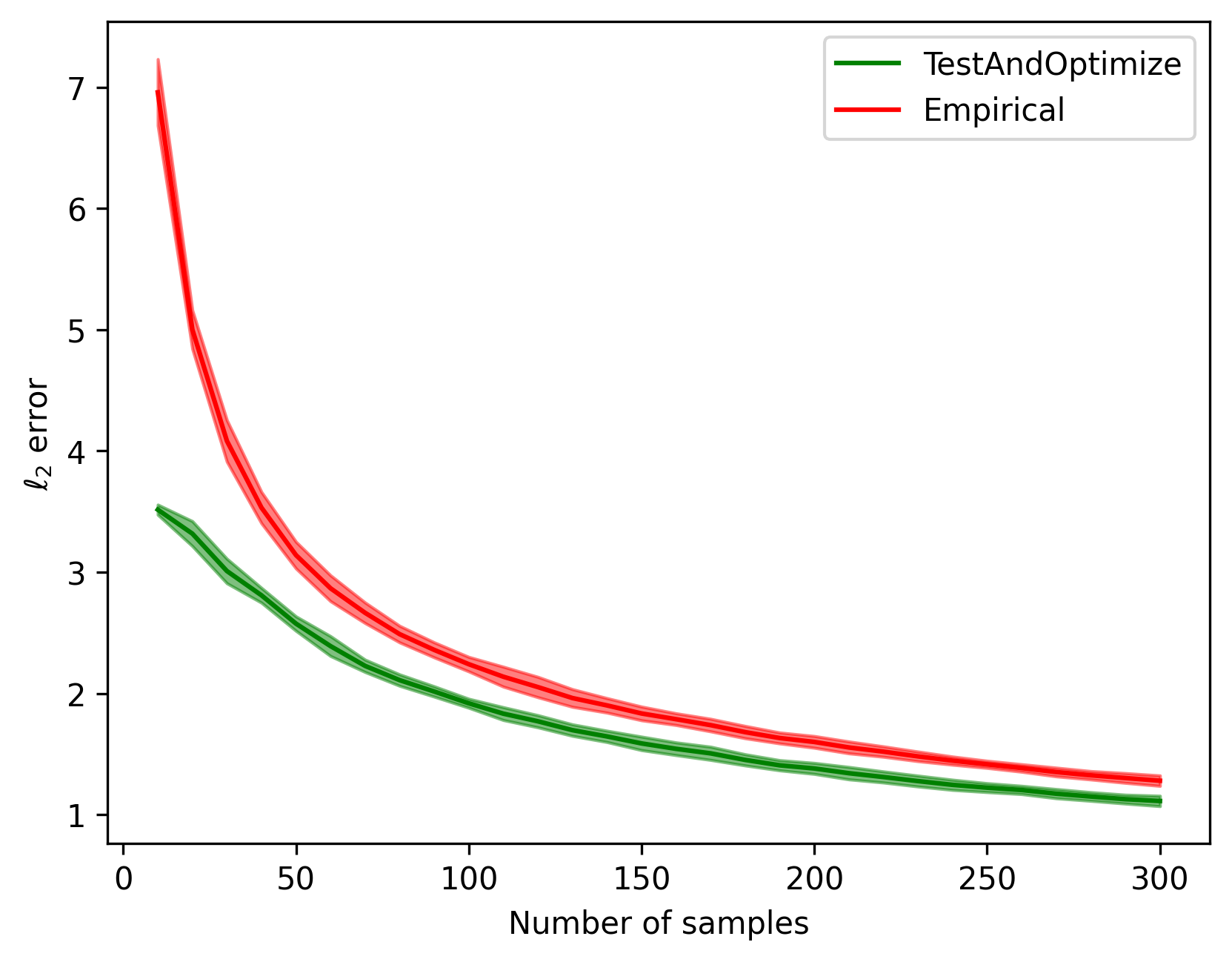}
\end{subfigure}
\hfill
\begin{subfigure}[b]{0.3\textwidth}
    \centering
    \includegraphics[width=\textwidth]{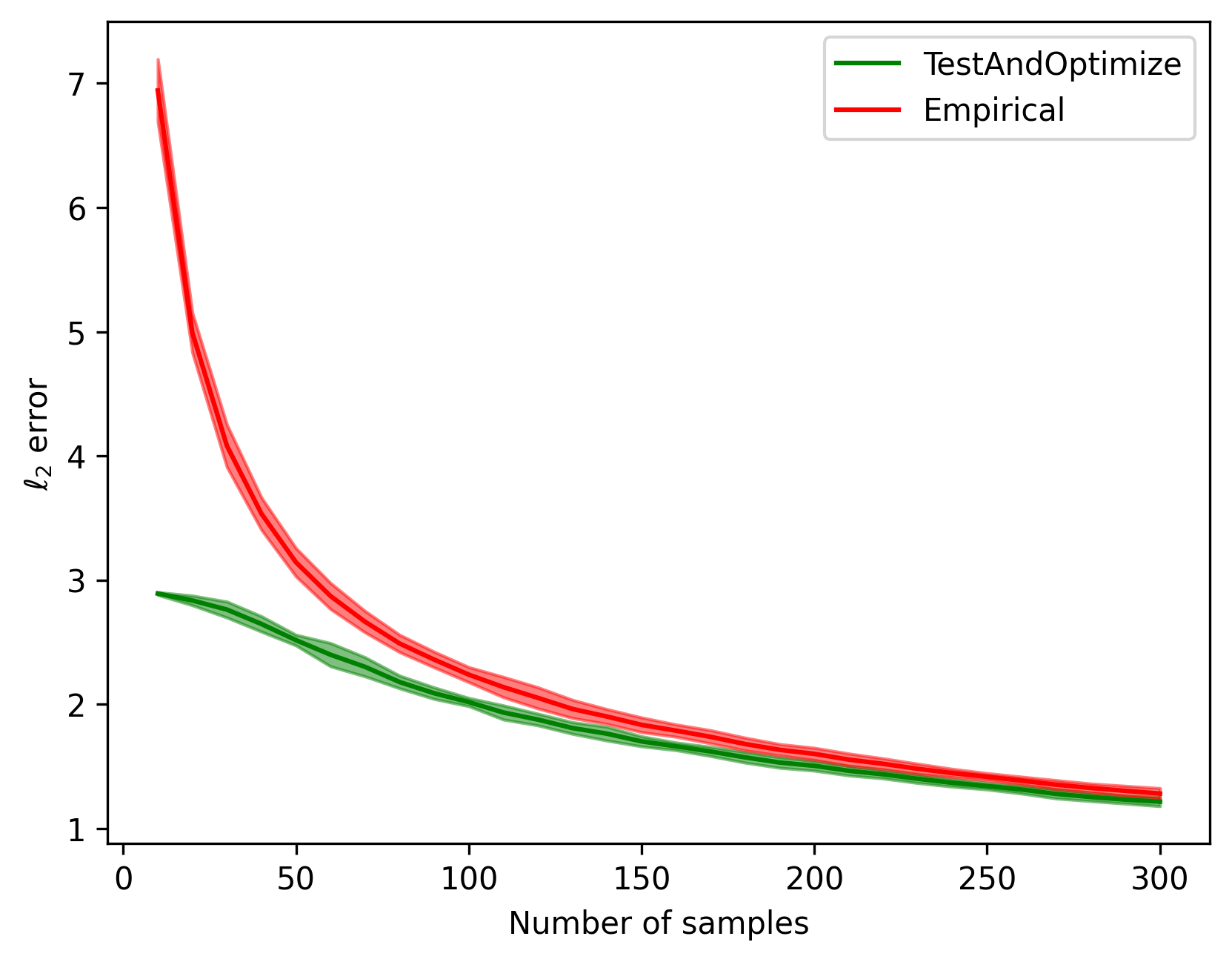}
\end{subfigure}
\caption{Here, $d = 500$, $s = \{100, 200, 300\}$, and $q = \|\bm{\mu} - \wt{\bm{\mu}} \|_1 = 50$. Error bars show standard deviation over $10$ runs.}
\label{fig:exp1}
\end{figure}

\begin{figure}[htb]
\centering
\begin{subfigure}[b]{0.3\textwidth}
    \centering
    \includegraphics[width=\textwidth]{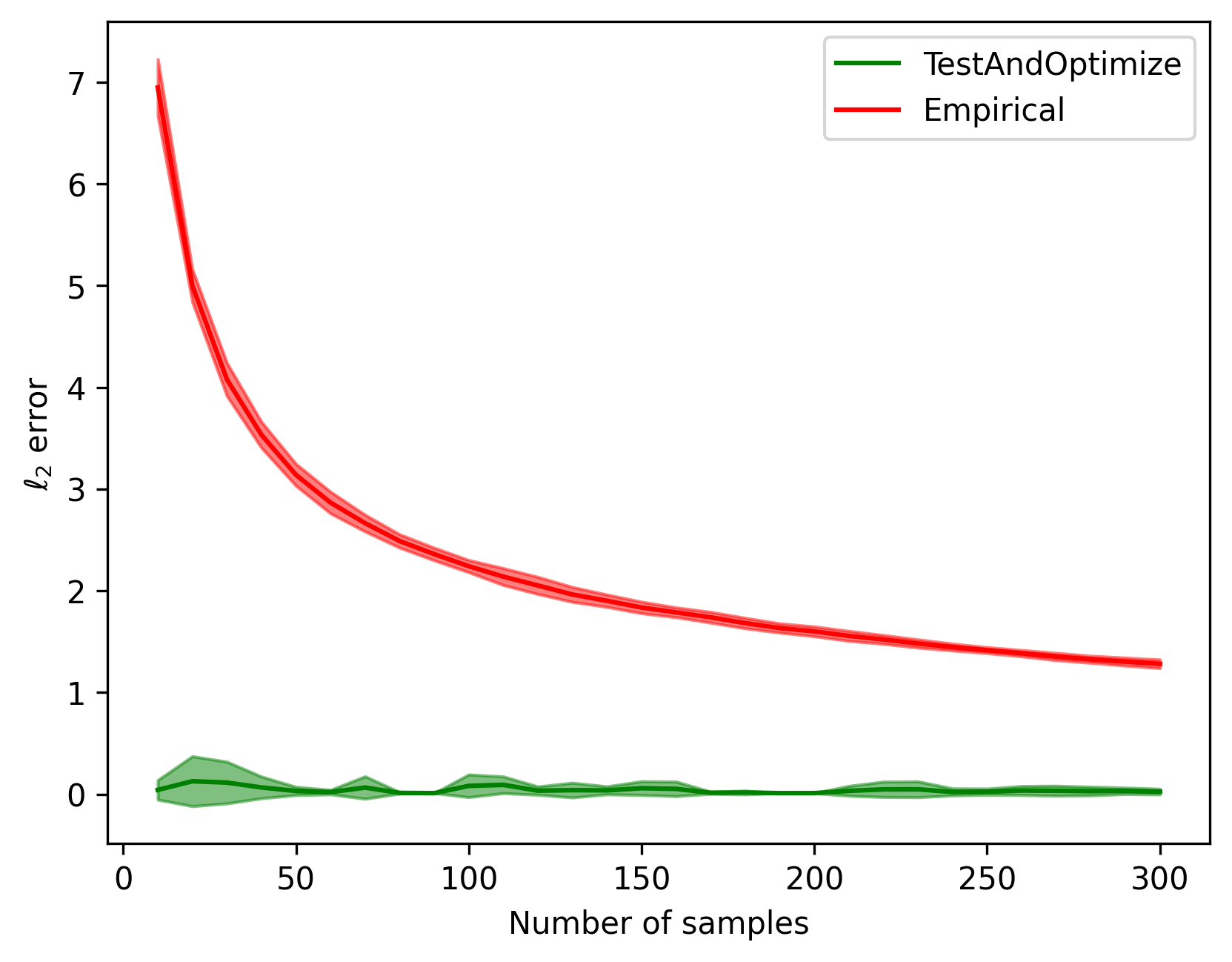}
\end{subfigure}
\hfill
\begin{subfigure}[b]{0.3\textwidth}
    \centering
    \includegraphics[width=\textwidth]{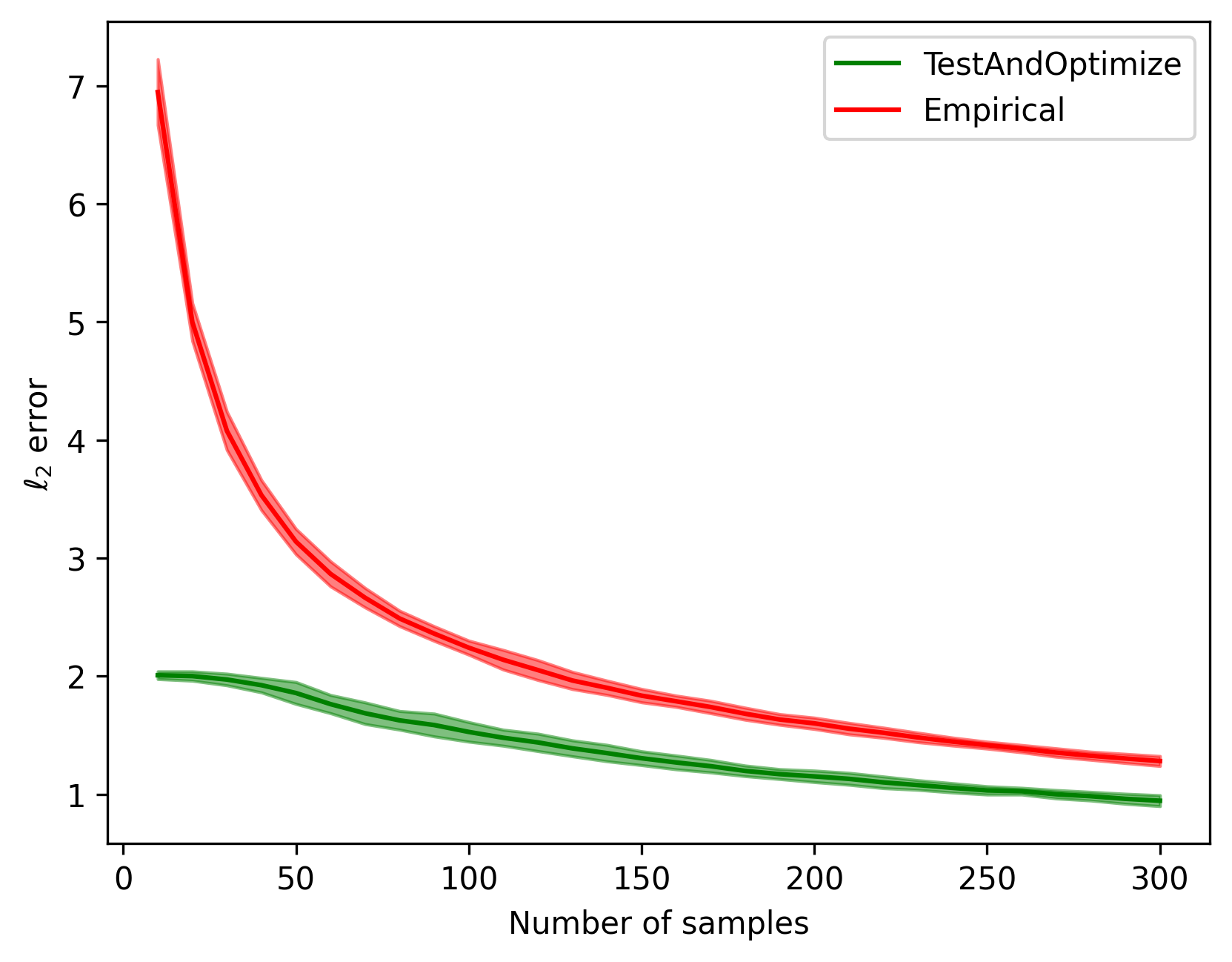}
\end{subfigure}
\hfill
\begin{subfigure}[b]{0.3\textwidth}
    \centering
    \includegraphics[width=\textwidth]{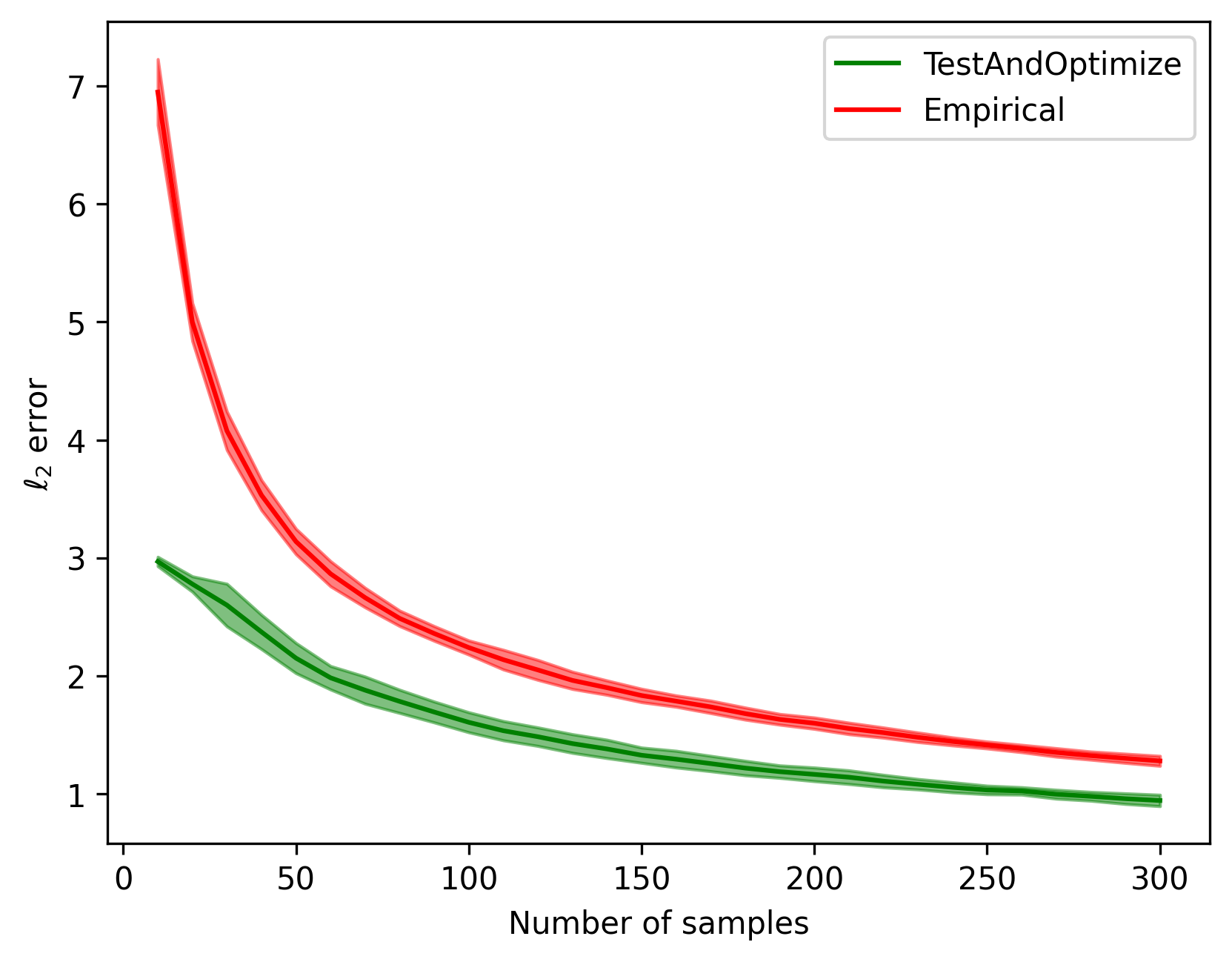}
\end{subfigure}
\caption{Here, $d = 500$, $s = 100$, and $q = \|\bm{\mu} - \wt{\bm{\mu}} \|_1 \in \{0.1, 20, 30\}$. Error bars show standard deviation over $10$ runs.}
\label{fig:exp2}
\end{figure}

\begin{figure}[htb]
\centering
\begin{subfigure}[b]{0.3\textwidth}
    \centering
    \includegraphics[width=\textwidth]{plots/plot_d500_sparsity100_L1norm0.1_Nmax=300_runs=10.png}
\end{subfigure}
\hfill
\begin{subfigure}[b]{0.3\textwidth}
    \centering
    \includegraphics[width=\textwidth]{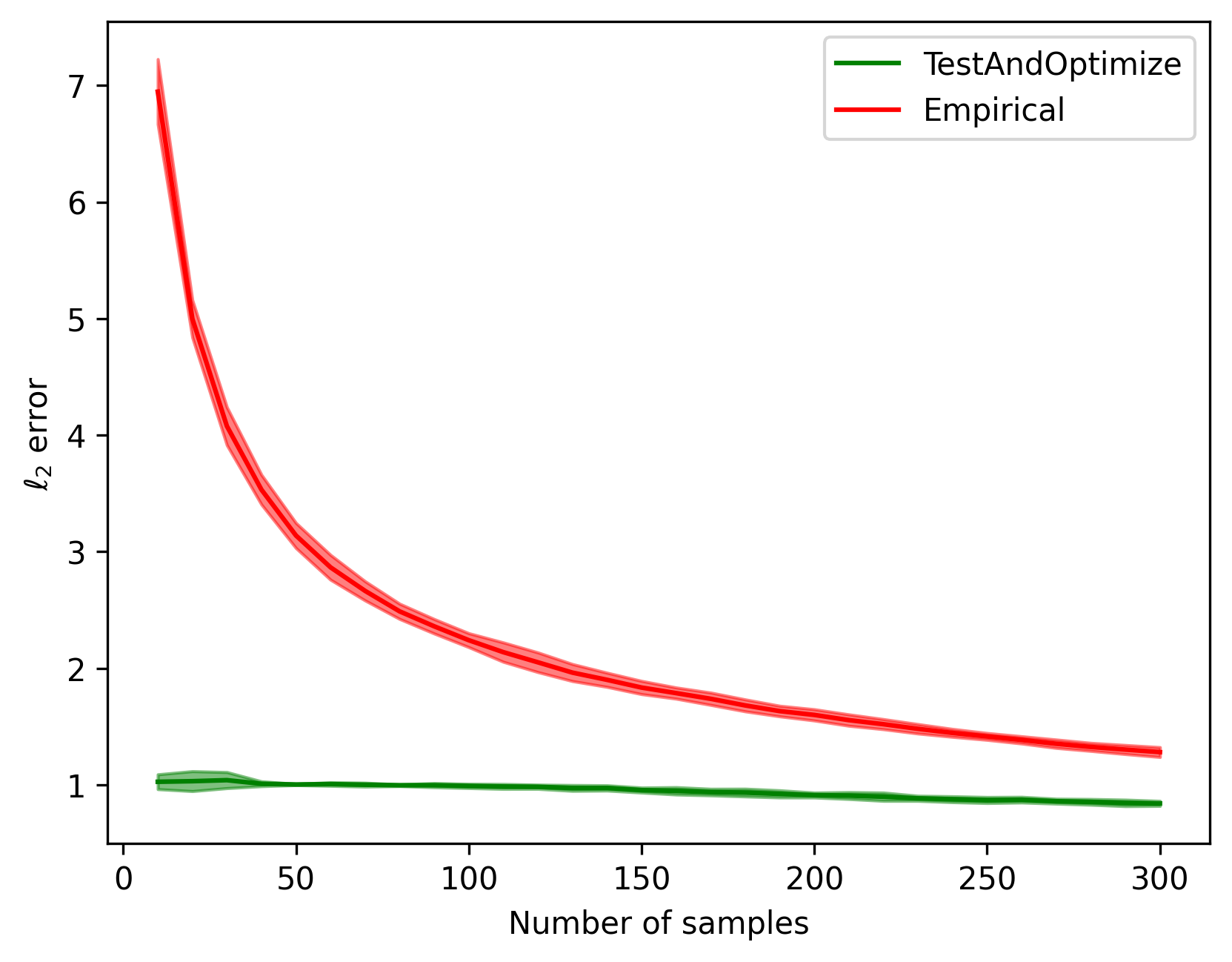}
\end{subfigure}
\hfill
\begin{subfigure}[b]{0.3\textwidth}
    \centering
    \includegraphics[width=\textwidth]{plots/plot_d500_sparsity100_L1norm20_Nmax=300_runs=10.png}
\end{subfigure}
\hfill
\begin{subfigure}[b]{0.3\textwidth}
    \centering
    \includegraphics[width=\textwidth]{plots/plot_d500_sparsity100_L1norm30_Nmax=300_runs=10.png}
\end{subfigure}
\hfill
\begin{subfigure}[b]{0.3\textwidth}
    \centering
    \includegraphics[width=\textwidth]{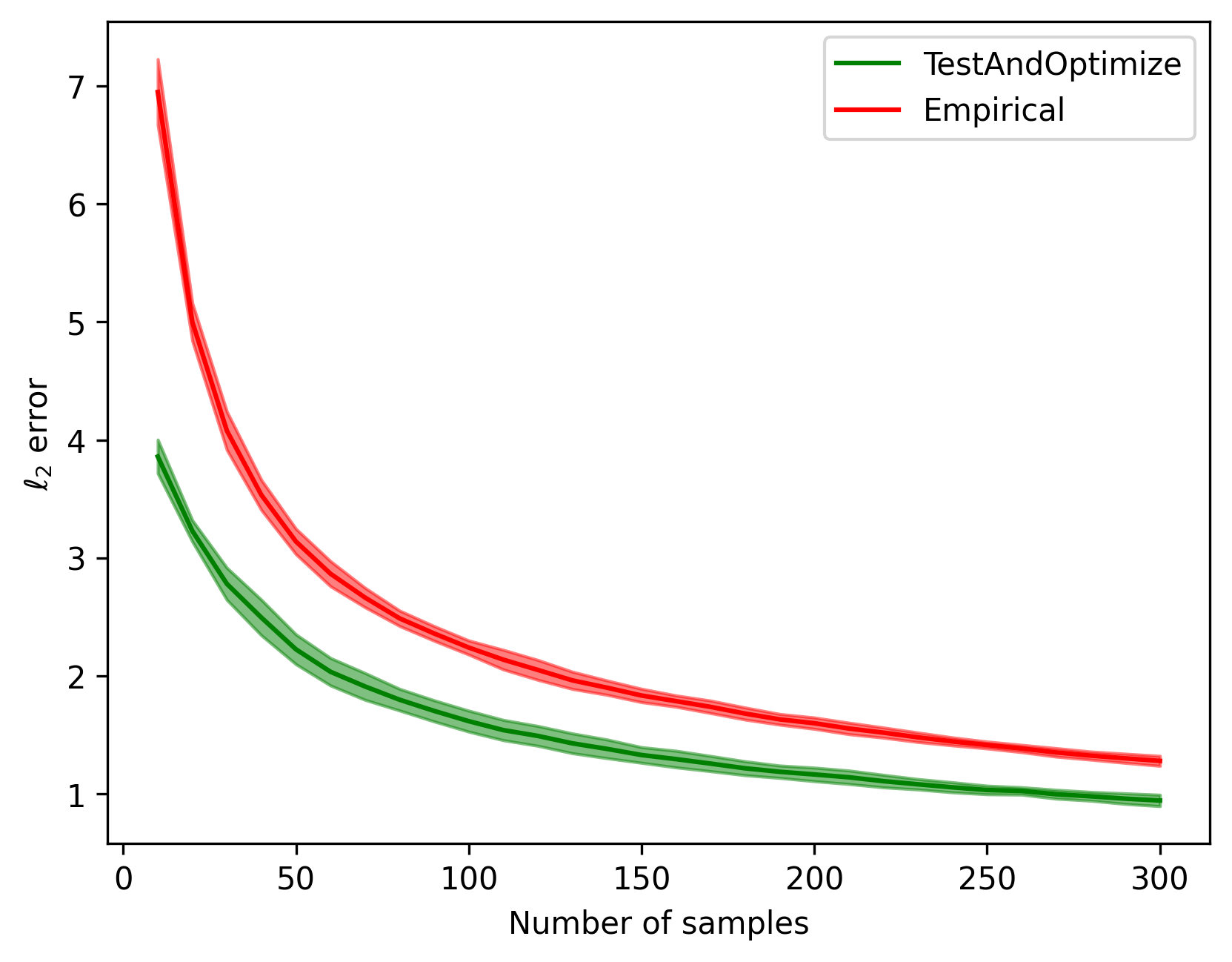}
\end{subfigure}
\hfill
\begin{subfigure}[b]{0.3\textwidth}
    \centering
    \includegraphics[width=\textwidth]{plots/plot_d500_sparsity100_L1norm50_Nmax=300_runs=10.png}
\end{subfigure}
\hfill
\begin{subfigure}[b]{0.3\textwidth}
    \centering
    \includegraphics[width=\textwidth]{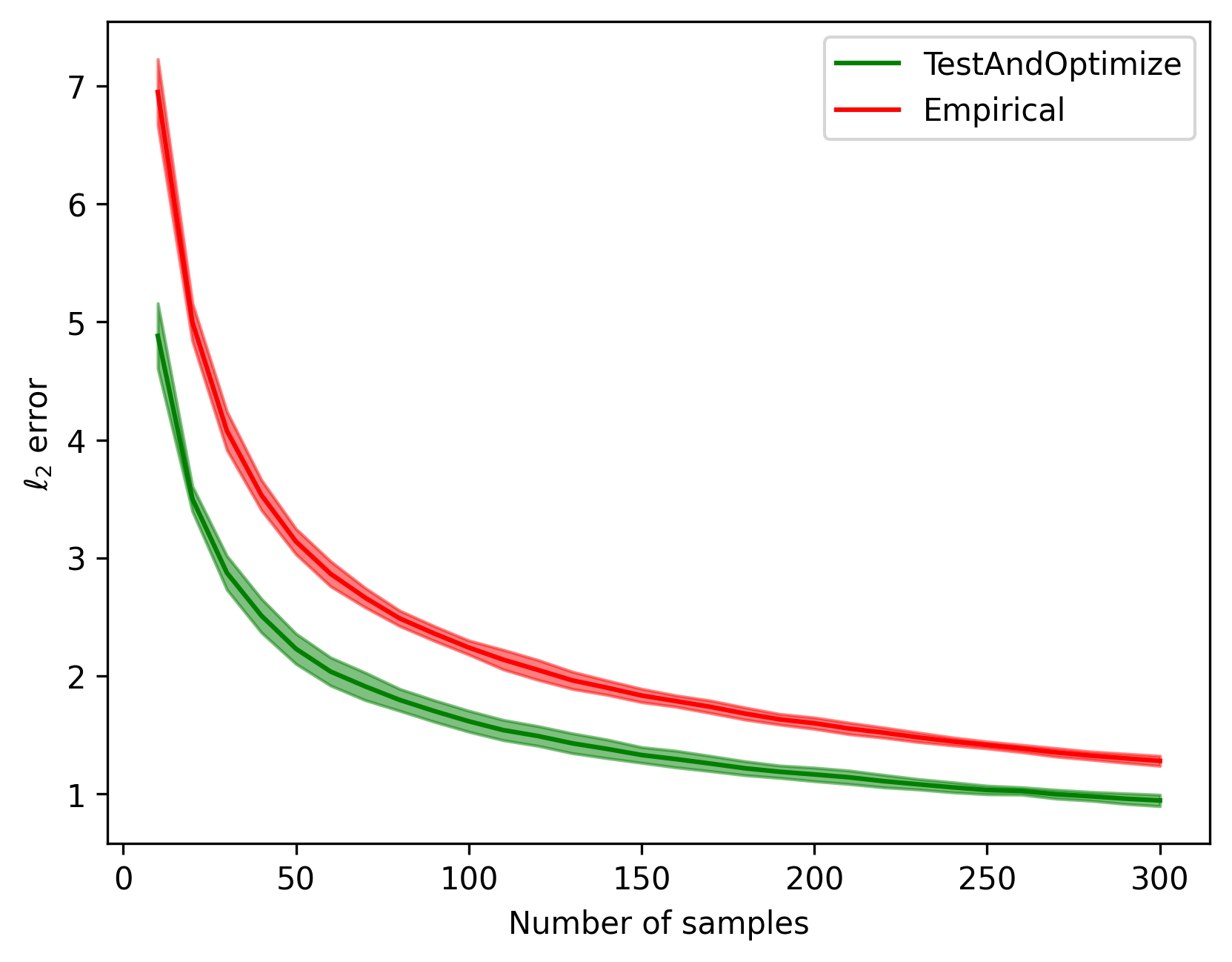}
\end{subfigure}
\hfill
\begin{subfigure}[b]{0.3\textwidth}
    \centering
    \includegraphics[width=\textwidth]{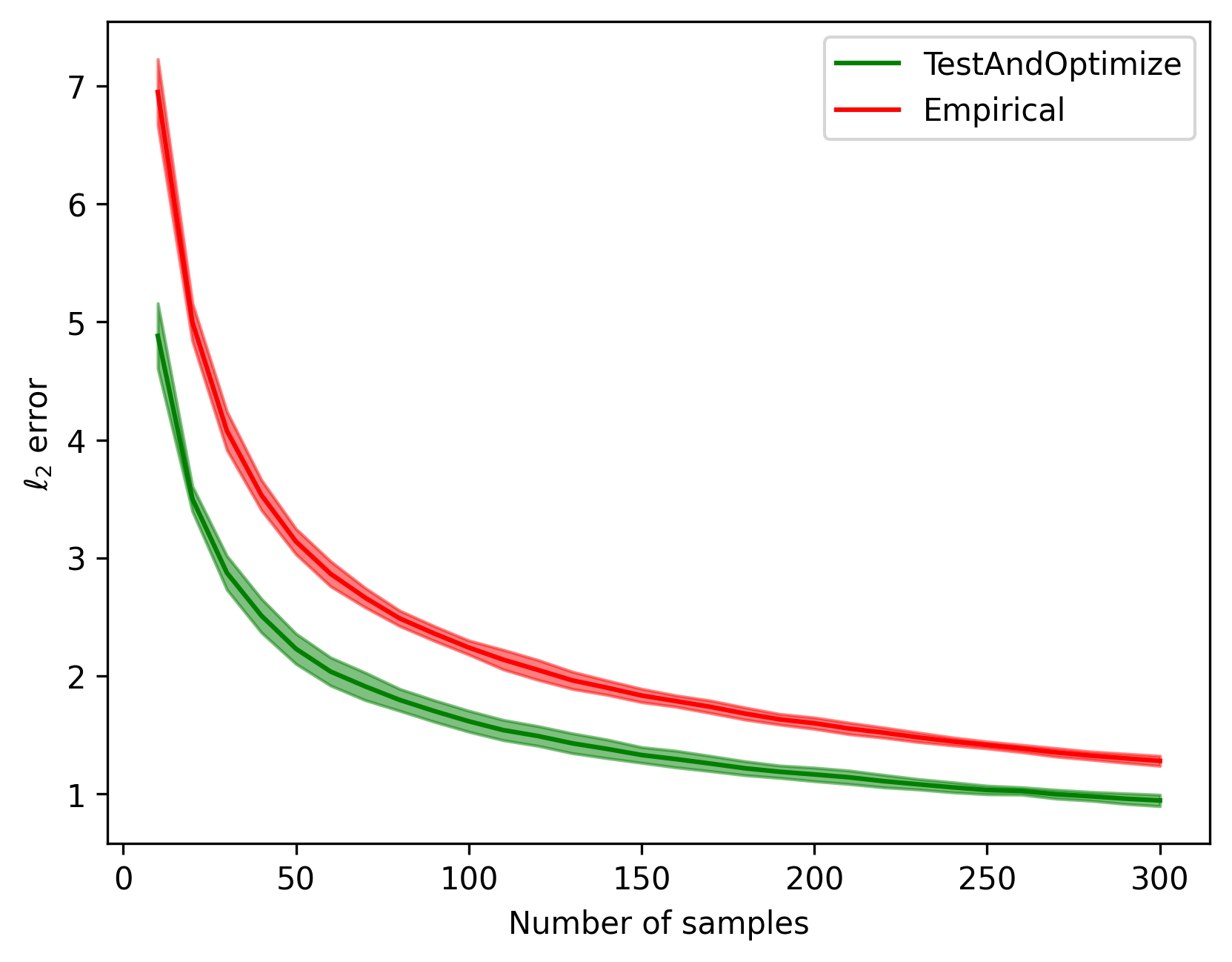}
\end{subfigure}
\hfill
\begin{subfigure}[b]{0.3\textwidth}
    \centering
    \includegraphics[width=\textwidth]{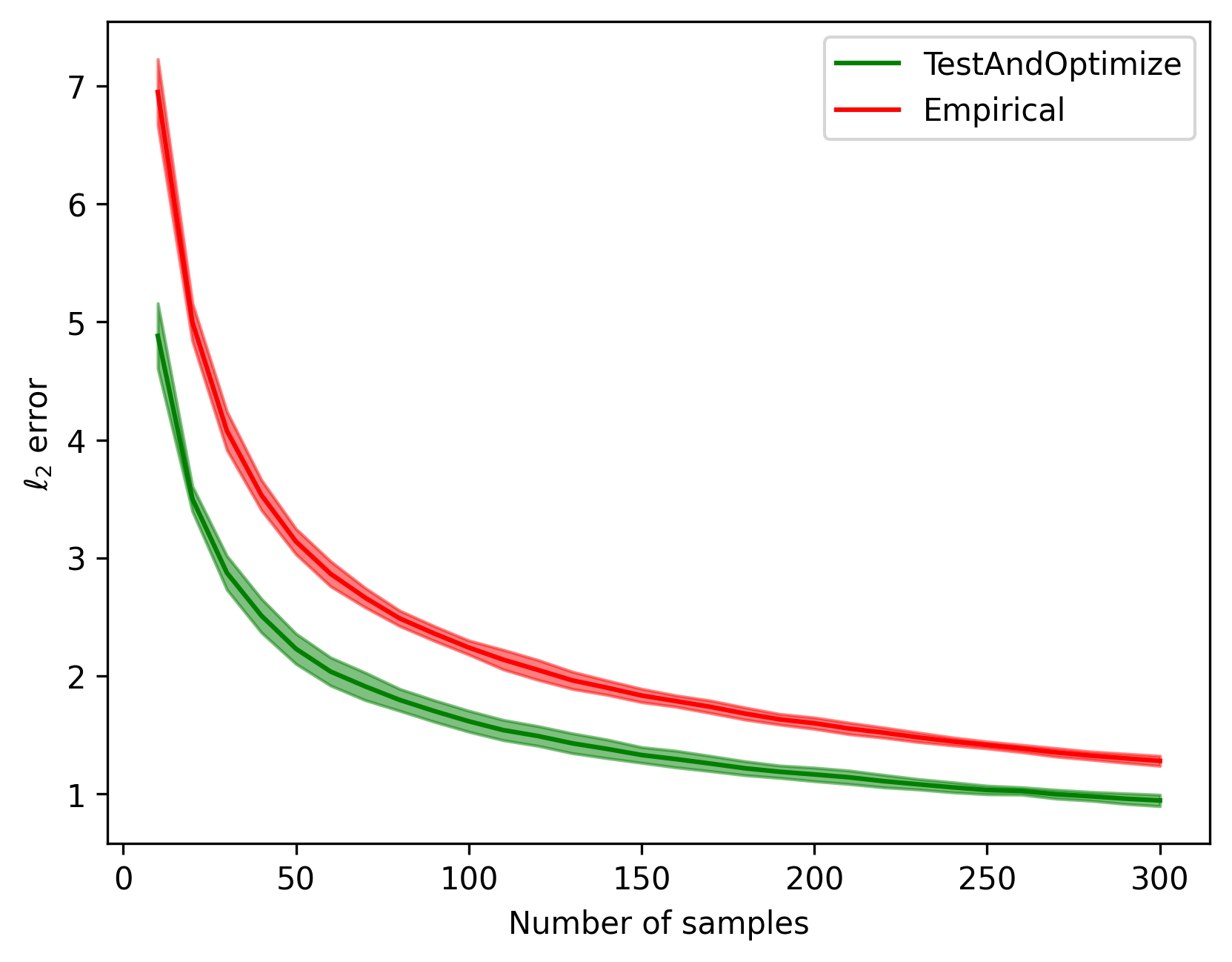}
\end{subfigure}
\caption{Here, $d = 500$, $s = 100$, and $q = \|\bm{\mu} - \wt{\bm{\mu}} \|_1 \in \{0.1, 10, 20, 30, 40, 50, 1000, 10000, 100000\}$. Error bars show standard deviation over $10$ runs. Observe that the slope of the green line looks the same for all $q \geq 1000$ instances.}
\label{fig:exp3}
\end{figure}

%% file: appendix-additional-results.tex
\section{Additional results}

\subsection{Tolerant testing}
\label{sec:appendix-tolerant-testing}

In this section, we present an algorithm for testing whether an unknown distribution is close to a standard normal distribution. More specifically, we first describe a tolerant tester for the property that the mean of an isotropic Gaussian distribution equals zero. Subsequently, we present a tolerant tester for the property that the covariance matrix equals the identity matrix.    

\subsubsection{Tolerant testing for mean}

The definition of a tolerant tester for the mean of an isotropic Gaussian distribution is given below.

\begin{definition}[Tolerant testing of isotropic Gaussian mean]
Fix $m \geq 1$, $d \geq 1$, $\eps_2 > \eps_1 > 0$, and $\delta > 0$.
Suppose $\bm{\mu} \in \R^d$ is a hidden mean vector and we draw $m$ samples $\bx_1, \ldots, \bx_m \sim N(\bm{\mu}, \bI_d)$.
An algorithm $\textsc{ALG}$ is said to be a $(\eps_1, \eps_2, \delta)$-tolerant isotropic Gaussian mean tester if it satisfies the following two conditions:
\begin{enumerate}
    \item If $\| \bm{\mu} \|_2 \leq \eps_1$, then $\textsc{ALG}$ should \emph{Accept} with probability at least $1 - \delta$
    \item If $\| \bm{\mu} \|_2 \geq \eps_2$, then $\textsc{ALG}$ should \emph{Reject} with probability at least $1 - \delta$.
\end{enumerate}
$\textsc{ALG}$ is allowed to decide arbitrarily when $\eps_1 < \| \bm{\mu} \|_2 < \eps_2$.
\end{definition}

It is known that the test statistic $y_n = \left\|\frac{1}{\sqrt{n}} \sum_{i=1}^{n} \bx_i \right\|_2^2$ can be used for \emph{non-tolerant} isotropic Gaussian mean testing with an appropriate threshold; see \cite[Appendix C]{DiakonikolasKaneStewart2017}.
With the following lemma we show that $y_n$ can also be used for \emph{tolerant} isotropic Gaussian mean testing.

\begin{algorithm}[htb]
\begin{algorithmic}[1]
\caption{The \textsc{TolerantIGMT} algorithm.}
\label{alg:isotropic-gaussian-mean-testing}
\Statex \textbf{Input}: $\eps_2 > \eps_1 > 0$, $\delta \in (0, 1)$, $m$ i.i.d.\ samples of $N(\bm{\mu}, \bI_d)$, where $\bm{\mu} \in \R^d$
\Statex \textbf{Output}: $\Fail$ (too little samples), $\Accept$ ($\| \bm{\mu} \|_2 \leq \eps_1$), or $\Reject$ ($\| \bm{\mu} \|_2 \geq \eps_2$).

\State Define sample batch size $n = \lceil \frac{16\sqrt{d}}{\eps_2^2 - \eps_1^2} \rceil$

\State Define number of rounds $r = \left\lceil \log(\frac{12}{\delta}) \right\rceil$ if $\left\lceil \log(\frac{12}{\delta}) \right\rceil$ is odd, otherwise define $r = 1 + \left\lceil \log(\frac{12}{\delta}) \right\rceil$

\State Define testing threshold $\tau = d + \frac{n(\eps_1^2 + \eps_2^2)}{2}$

\If{$m < nr$}
    \State \Return $\Fail$
\Else
    \For{$i \in \{1, \ldots, r\}$}
        \State Use an unused batch of $n$ i.i.d.\ samples $\bx^{(i)}_1, \ldots, \bx^{(i)}_n \sim N(\bm{\mu}, \bI_d)$
        
        \State Compute test statistic $y^{(i)}_n = \left\| \frac{1}{\sqrt{n}} \sum_{i=1}^{n} \bx^{(i)}_i \right\|_2^2$ for the $i^{th}$ test
        
        \State Define $i^{th}$ outcome $\texttt{R}^{(i)}$ as $\Accept$ if $y^{(i)}_n \leq \tau$, and $\Reject$ otherwise
    \EndFor
    \State \Return $\maj(\texttt{R}^{(1)}, \ldots, \texttt{R}^{(r)})$
\EndIf
\end{algorithmic}
\end{algorithm}

\begin{lemma}
\label{lem:tolerant-isotropic-GMT}
Fix $m \geq 1$, $d \geq 1$, $\eps_2 > \eps_1 > 0$, and $\delta > 0$.
Suppose $\bm{\mu} \in \R^d$ is a hidden mean vector and we draw $m$ i.i.d.\ samples $\bx_1, \ldots, \bx_m \sim N(\bm{\mu}, \bI_d)$.
When $d \geq \left( \frac{16 \eps_2^2}{\eps_2^2 - \eps_1^2} \right)^2$ and $m \in \cO \left(\frac{\sqrt{d}}{\eps_2^2 - \eps_1^2} \log\left( \frac{1}{\delta} \right) \right)$, \textsc{TolerantIGMT} (\cref{alg:isotropic-gaussian-mean-testing}) is a $(\eps_1, \eps_2, \delta)$-tolerant isotropic Gaussian mean tester.
\end{lemma}
\begin{proof}
The total number of samples $m$ required is
$
nr \in \cO \left(\frac{\sqrt{d}}{\eps_2^2 - \eps_1^2} \log\left( \frac{1}{\delta} \right) \right)
$
since \textsc{TolerantIGMT} uses $n = \frac{16 \sqrt{d}}{\eps_2^2 - \eps_1^2}$ i.i.d.\ samples in each of the $r \in \cO(\log(\frac{1}{\delta}))$ rounds.

For correctness, we will prove that each round $i \in \{1, \ldots, r\}$ succeeds with probability at least $2/3$.
Then, by Chernoff bound, the majority outcome out of $r \geq \log(\frac{12}{\delta})$ independent tests will be correct with probability at least $1 - \delta$.
    
Now, fix an arbitrary round $i \in \{1, \ldots, r\}$.
\textsc{TolerantIGMT} uses $n = \frac{16 \sqrt{d}}{\eps_2^2 - \eps_1^2} \geq 1$ i.i.d.\ samples to form a statistic $y^{(i)}_n$ and tests against the threshold $\tau = d + \frac{n (\eps_1^2 + \eps_2^2)}{2}$.
From \cref{lem:non-central-chi-square} (first item), we know that $y^{(i)}_n \sim \chi^{\prime^2}_d(\lambda)$ is a non-central chi-square random variable with $\lambda = n \| \bm{\mu}\|_2^2$.
Let us define $t = \frac{n (\eps_2^2 - \eps_1^2)}{2} > 0$.
Observe that we can rewrite the testing threshold $\tau$ in two different ways: $\tau = d + \frac{n(\eps_1^2 + \eps_2^2)}{2} = d + n \eps_1^2 + t = d + n \eps_2^2 - t$.

\textbf{Case 1}: $\| \bm{\mu} \|_2 \leq \eps_1$

In this case, we have $\lambda = n \| \bm{\mu}\|_2^2 \leq n \eps_1^2$ and $\tau = d + n \eps_1^2 + t$.
So,
\begin{align*}
\Pr(y^{(i)}_n > \tau)
&= \Pr(y^{(i)}_n > d + n \eps_1^2 + t) \tag{since $\tau = d + n \eps_1^2 + t$}\\
&\leq \Pr(y^{(i)}_n > d + \lambda + t) \tag{since $\lambda \leq n \eps_1^2$}\\
&\leq \exp\left(-\frac{dt^2}{4 (d + 2\lambda)(d + 2\lambda + t)}\right) \tag{apply \cref{lem:non-central-chi-square} (second item) with $t > 0$}\\
&\leq \exp\left(-\frac{dt^2}{4 (d + 2n\eps_1^2)(d + 2n\eps_1^2 + t)}\right) \tag{since $\lambda \leq n \eps_1^2$}\\
&\leq \exp\left(-\frac{dn^2(\eps_2^2 - \eps_1^2)^2}{16(d + 2n\eps_1^2)(d + 2n\eps_2^2)}\right) \tag{since $t = \frac{n (\eps_2^2 - \eps_1^2)}{2} 
\leq 2n (\eps_2^2 - \eps_1^2)$}\\
&= \exp\left(-\frac{16^2 d^2}{16(d + 2n\eps_1^2)(d + 2n\eps_2^2)}\right) \tag{since $n = \frac{16 \sqrt{d}}{\eps_2^2 - \eps_1^2}$}\\
&= \exp\left(-\frac{16}{\left(1 + \frac{2n\eps_1^2}{d}\right) \left(1 + \frac{2n\eps_2^2}{d}\right)}\right) \tag{dividing both numerator and denominator by $16 d^2$}\\
&= \exp\left(-\frac{16}{\left(1 + \frac{32 \eps_1^2}{\sqrt{d}(\eps_2^2 - \eps_1^2)}\right)\left(1 + \frac{32 \eps_2^2}{\sqrt{d}(\eps_2^2 - \eps_1^2)}\right)}\right) \tag{since $n = \frac{16 \sqrt{d}}{\eps_2^2 - \eps_1^2}$}\\
&= \exp\left(-\frac{16}{(1+2)(1+2)}\right) \tag{since $d \geq \left( \frac{16 \eps_2^2}{\eps_2^2 - \eps_1^2} \right)^2 \geq \left( \frac{16 \eps_1^2}{\eps_2^2 - \eps_1^2} \right)^2$}\\
&= \exp\left(-\frac{16}{9} \right) < \frac{1}{3}
\end{align*}

Thus, when $\| \bm{\mu} \|_2 \leq \eps_1$, we have $\Pr(y^{(i)}_n \leq \tau) \geq 2/3$ and the $i^{th}$ test outcome will be correctly an $\Accept$ with probability at least $2/3$.

\textbf{Case 2}: $\| \bm{\mu} \|_2 \geq \eps_2$

In this case, we have $\lambda = n \| \bm{\mu}\|_2^2 \geq n \eps_2^2 > n \eps_1^2$ and $\tau = d + n \eps_2^2 - t$.
We first observe the following inequalities:
\begin{itemize}
    \item Since $n \geq 1$, $d \geq 1$, $\lambda \geq n \eps_2^2$, and $\eps_2 > \eps_1 > 0$, we see that
    \begin{equation}
    \label{eq:tolerant-isotropic-GMT-inequality1}
    \left( 2 - \frac{n \eps_1^2}{\lambda} - \frac{n \eps_2^2}{\lambda} \right)^2 \geq \left( 1 - \frac{\eps_1^2}{\eps_2^2} \right)^2
    \quad \text{and} \quad
    \left( \frac{d}{\lambda} + 2 \right)^2 \leq \left( \frac{d}{n \eps_2^2} + 2 \right)^2
    \end{equation}
    \item Since $n = \frac{16\sqrt{d}}{\eps_2^2 - \eps_1^2} \geq 1$ and $d \geq \left( \frac{16 \eps_2^2}{\eps_2^2 - \eps_1^2} \right)^2 \geq 1$, we see that
    \begin{equation}
    \label{eq:tolerant-isotropic-GMT-inequality2}
    \left( 1 + \frac{2 n \eps_2^2}{d} \right)^2 \leq 3^2
    \end{equation}
\end{itemize}
So,
\begin{align*}
\Pr(y^{(i)}_n < \tau)
&= \Pr(y^{(i)}_n < d + n \eps_2^2 - t) \tag{since $\tau = d + n \eps_2^2 - t$}\\
&= \Pr(y^{(i)}_n < d + \lambda  - (\lambda + t - n \eps_2^2)) \tag{Rewriting}\\
&\leq \exp\left( -\frac{d (\lambda + t - n\eps_2^2)^2}{4(d+2\lambda)^2} \right) \tag{apply \cref{lem:non-central-chi-square} (third item) with $0 < \lambda + t - n \eps_2^2 < d + \lambda$}\\
&= \exp \left(-\frac{d \left( \lambda - \frac{n}{2} \eps_1^2 - \frac{n}{2} \eps_2^2 \right)^2}{4 (d+2\lambda)^2} \right) \tag{since $t = \frac{n (\eps_2^2 - \eps_1^2)}{2}$}\\
&= \exp \left(-\frac{d \left(2 - \frac{n \eps_1^2}{ \lambda} - \frac{n \eps_2^2}{ \lambda} \right)^2}{16 \left( \frac{d}{\lambda} + 2 \right)^2} \right) \tag{Pulling out the factor of $\frac{\lambda}{2}$ from numerator}\\
&\leq \exp \left(-\frac{d \left( 1 - \frac{\eps_1^2}{\eps_2^2} \right)^2}{16 \left( \frac{d}{n \eps_2^2} + 2 \right)^2} \right) \tag{by \cref{eq:tolerant-isotropic-GMT-inequality1}}\\
&\leq \exp \left(-\frac{n^2 \left( \eps_2^2 - \eps_1^2 \right)^2}{16 d \left( 1 + \frac{n \eps_2^2}{d} \right)^2} \right) \tag{Pulling out factors of $n$, $d$, and $\eps_2^2$}\\
&= \exp \left(-\frac{16}{\left( 1 + \frac{n \eps_2^2}{d} \right)^2} \right) \tag{since $n = \frac{16\sqrt{d}}{\eps_2^2 - \eps_1^2}$}\\
&= \exp \left(-\frac{16}{3^2} \right) = \exp\left(-\frac{16}{9} \right) < \frac{1}{3} \tag{by \cref{eq:tolerant-isotropic-GMT-inequality2}}
\end{align*}

Thus, when $\| \bm{\mu} \|_2 \geq \eps_2$, we have $\Pr(y^{(i)}_n \geq \tau) \geq 2/3$ and the $i^{th}$ test outcome will be correctly a $\Reject$ with probability at least $2/3$.
\end{proof}

We are now ready to state the main theorem below.

\tolerantmeantester*
\begin{proof}
Use the guarantee of \cref{lem:tolerant-isotropic-GMT} on \textsc{TolerantIGMT} (\cref{alg:isotropic-gaussian-mean-testing}) with parameters $\eps_1 = \eps$ and $\eps_2 = 2 \eps$.
\end{proof}

\subsubsection{Tolerant testing for covariance matrix}

We now give the definition of a tolerant tester for the unknown covariance matrix being equal to identity. 
\begin{definition}[Tolerant testing of zero-mean Gaussian covariance matrix]
Fix $m \geq 1$, $d \geq 1$, $\eps_2 > \eps_1 > 0$, and $\delta > 0$.
Suppose $\bm{\Sigma} \in \R^{d \times d}$ is a hidden full rank covariance matrix and we draw $m$ samples $\bx_1, \ldots, \bx_m \sim N(\bm{0}, \bm{\Sigma})$.
An algorithm $\textsc{ALG}$ is said to be a $(\eps_1, \eps_2, \delta)$-tolerant zero-mean Gaussian covariance tester if it satisfies the following two conditions:
\begin{enumerate}
    \item If $\| \bm{\Sigma} - \bI_d \|_F \leq \eps_1$, then $\textsc{ALG}$ should \emph{Accept} with probability at least $1 - \delta$
    \item If $\| \bm{\Sigma} - \bI_d \|_F \geq \eps_2$, then $\textsc{ALG}$ should \emph{Reject} with probability at least $1 - \delta$.
\end{enumerate}
$\textsc{ALG}$ is allowed to decide arbitrarily when $\eps_1 < \| \bm{\Sigma} - \bI_d \|_2 < \eps_2$.
\end{definition}

\begin{definition}[Test statistic 
$\texttt{T}_n$]
\label{def:test-statistic-Tn}
Let $x_1, \ldots, x_n$ be $n$ i.i.d.\ samples from $\sim N(\bm{0}, \bm{\Sigma})$ for an unknown $\bm{\Sigma} \in \R^{d \times d}$.
For $i \neq j$, we define $h(x_i, x_j) = (x_i^\top x_j)^2 - (x_i^\top x_i + x_j^\top x_j) + d$.
Then, we define $\texttt{T}_n$ as
\[
\texttt{T}_n = \frac{2}{n (n-1)} \sum_{1 \leq i < j \leq n} h(x_i, x_j)
\]    
\end{definition}

It is known that the test statistic $\texttt{T}_n$ (\cref{def:test-statistic-Tn}) can be used for \emph{non-tolerant} zero-mean Gaussian covariance testing with an appropriate threshold; see \cite{cai2013optimal}.
With the following lemma, we show that $\texttt{T}_n$ can also be used for \emph{tolerant} zero-mean Gaussian covariance testing.

\begin{algorithm}[htb]
\begin{algorithmic}[1]
\caption{\textsc{TolerantZMGCT}.}
\label{alg:zero-mean-covariance-testing}
\Statex \textbf{Input}: $\eps_2 > \eps_1 > 0$, $\delta \in (0,1)$, $m$ i.i.d.\ samples of $N(\bm{0}, \bm{\Sigma})$, where $\bm{\Sigma} \in \R^{d \times d}$
\Statex \textbf{Output}: $\Fail$ (too little samples), $\Accept$ ($\| \bm{\Sigma} - \bI_d \|_F^2 \leq \eps_1^2$), or $\Reject$ ($\| \bm{\Sigma} - \bI_d \|_F^2 \geq \eps_2^2$)

\State Define sample batch size $n = \left\lceil 3200 \cdot d \cdot \max \left\{ \frac{1}{\eps_1^2}, \left( \frac{\eps_1^2}{\eps_2^2 - \eps_1^2} \right)^2, 2 \left( \frac{\eps_2}{\eps_2^2 - \eps_1^2} \right)^2\right\} \right\rceil$

\State Define number of rounds $r = \left\lceil \log(\frac{12}{\delta}) \right\rceil$ if $\left\lceil \log(\frac{12}{\delta}) \right\rceil$ is odd, otherwise define $r = 1 + \left\lceil \log(\frac{12}{\delta}) \right\rceil$

\State Define testing threshold $\tau = \frac{\eps_2^2 + \eps_1^2}{2}$

\If{$m < nr$}
    \State \Return $\Fail$
\Else
    \For{$i \in \{1, \ldots, r\}$} 
        \State Use an unused batch of $n$ i.i.d.\ samples $\bx^{(i)}_1, \ldots, \bx^{(i)}_n \sim N(\bm{0}, \bm{\Sigma})$
        
        \State Compute test statistic $T^{(i)}_n$ according to \cref{def:test-statistic-Tn} for the $i^{th}$ test
        
        \State Define $i^{th}$ outcome $R^{(i)}$ as $\Accept$ if $T^{(i)}_n \leq \tau$, and $\Reject$ otherwise
    \EndFor
    \State \Return $\maj(R^{(1)}, \ldots, R^{(r)})$
\EndIf
\end{algorithmic}
\end{algorithm}

\begin{lemma}
\label{lem:tolerant-ZMGCT}
Fix $m \geq 1$, $d \geq 1$, $\eps_2 > \eps_1 > 0$, and $\delta > 0$.
Suppose $\bm{\Sigma} \in \R^{d \times d}$ is a hidden full rank covariance matrix and we draw $m$ i.i.d.\ samples $\bx_1, \ldots, \bx_m \sim N(\bm{0}, \bm{\Sigma})$.
When $d \geq \eps_2^2$ and
\[
m \geq \cO \left( d \cdot \max \left\{ \frac{1}{\eps_1^2}, \left( \frac{\eps_1^2}{\eps_2^2 - \eps_1^2} \right)^2, \left( \frac{\eps_2}{\eps_2^2 - \eps_1^2} \right)^2 \right\} \cdot \log\left( \frac{1}{\delta} \right) \right) \;,
\]
\textsc{TolerantZMGCT} (\cref{alg:zero-mean-covariance-testing}) is a $(\eps_1, \eps_2, \delta)$-tolerant zero-mean Gaussian covariance tester.
\end{lemma}

To prove \cref{lem:tolerant-ZMGCT}, we first state the expectation and variance of $\texttt{T}_n$ known from \cite{cai2013optimal}, and give an upper bound on the variance that will be useful for subsequent analysis.

\begin{lemma}[\cite{cai2013optimal}]
\label{lem:expectation-and-variance-of-test-statistic}
For the test statistic $\texttt{T}_n$ defined in \cref{def:test-statistic-Tn}, we have $\E(\texttt{T}_n) = \| \bm{\Sigma} - \bI_d \|_F^2$ and $\sigma^2(\texttt{T}_n) = \frac{4}{n (n-1)} \left[ \Tr^2(\bm{\Sigma}^2) + \Tr(\bm{\Sigma}^4) \right] + \frac{8}{n} \Tr(\bm{\Sigma}^2 (\bm{\Sigma} - \bI_d)^2)$.
\end{lemma}

\begin{lemma}
\label{lem:X-to-X+I-bound}
Fix $d, n \geq 1$, $\bm{\Sigma} \in \R^{d \times d}$, and $b \geq 0$.
If $\| \bm{\Sigma} - \bI_d \|_F^2 = \frac{b^2 d}{n}$, then $\| \bm{\Sigma} \|_F^2 \leq d \cdot \left( 1 + \frac{b}{\sqrt{n}} \right)^2$.
\end{lemma}
\begin{proof}
Since the matrices can be treated as vectors in $\mathbb{R}^{d^2}$ and then the Frobenius norm corresponds to the $\ell_2$ norm, we see that
\begin{align*}
\| \bm{\Sigma} \|_F
&\leq \| \bm{\Sigma} - \bI_d \|_F + \| \bI_d \|_F \tag{Triangle inequality}\\
&= b \cdot \sqrt{\frac{d}{n}} + \sqrt{d} \tag{Since $\| \bm{\Sigma} - \bI_d \|_F^2 = \frac{b^2 d}{n}$ and $\| \bI_d \|_F^2 = d$}\\
&=\sqrt{d}\left(1+\frac{b}{\sqrt{n}}\right)
\end{align*}
Therefore, $\| \bm{\Sigma} \|_F^2 \leq d \cdot \left( 1 + \frac{b}{\sqrt{n}} \right)^2$ as desired.
\end{proof}

\begin{restatable}{lemma}{teststatisticvarianceupperbound}
\label{lem:test-statistic-variance-upper-bound}
Fix $d \geq 1$, $n \geq 2$, $\bm{\Sigma} \in \R^{d \times d}$, and $b \geq 0$.
If $\| \bm{\Sigma} - \bI_d \|_F^2 = \frac{b^2 d}{n}$, then for the test statistic $\texttt{T}_n$ defined in \cref{def:test-statistic-Tn}, we have
\[
\sigma^2(\texttt{T}_n)
\leq \frac{64 d^2}{n^2} \cdot \left( 1 + \frac{b^2}{n} \right) \cdot \left( 1 + \frac{b^2}{n} + b^2 \right)
\]
\end{restatable}
\begin{proof}
We begin by observing two simple upper bounds for $\Tr(\bm{\Sigma}^4)$ and $\Tr(\bm{\Sigma}^2 (\bm{\Sigma} - \bI_d)^2)$.

\begin{equation}
\label{eq:upper-bound-second-term}
\Tr(\bm{\Sigma}^4)
= \| \bm{\Sigma}^2 \|_F^2
\leq \| \bm{\Sigma} \|_F^2 \cdot \| \bm{\Sigma} \|_F^2
= \| \bm{\Sigma} \|_F^4
= \Tr^2(\bm{\Sigma}^2)
\end{equation}

Since $\bm{\Sigma} (\bm{\Sigma} - \bI_d) = \bm{\Sigma}^2 - \bm{\Sigma} = (\bm{\Sigma} - \bI_d) \bm{\Sigma}$, i.e.\ $\bm{\Sigma}$ and $\bm{\Sigma} - \bI_d$ commute, we have
\begin{equation}
\label{eq:upper-bound-third-term}
\Tr(\bm{\Sigma}^2 (\bm{\Sigma} - \bI_d)^2)
= \Tr((\bm{\Sigma} (\bm{\Sigma} - \bI_d))^2)
= \| \bm{\Sigma} (\bm{\Sigma} - \bI_d) \|_F^2
\leq \| \bm{\Sigma} \|_F^2 \cdot \| \bm{\Sigma} - \bI_d \|_F^2
= \Tr(\bm{\Sigma}^2) \cdot \Tr((\bm{\Sigma} - \bI_d)^2)
\end{equation}

\begin{align*}
&\; \bm{\Sigma}^2(\texttt{T}_n)\\
= &\; \frac{4}{n (n-1)} \left[ \Tr^2(\bm{\Sigma}^2) + \Tr(\bm{\Sigma}^4) \right] + \frac{8}{n} 
\Tr(\bm{\Sigma}^2 (\bm{\Sigma} - \bI_d)^2) && \tag{By \cref{lem:expectation-and-variance-of-test-statistic}}\\
\leq &\; \frac{8}{n (n-1)} \left[ \Tr^2(\bm{\Sigma}^2) + (n-1) \cdot 
\Tr(\bm{\Sigma}^2 (\bm{\Sigma} - \bI_d)^2) \right] && \tag{By \cref{eq:upper-bound-second-term}}\\
\leq &\; \frac{8}{n (n-1)} \left[ \Tr^2(\bm{\Sigma}^2) + (n-1) \cdot \Tr(\bm{\Sigma}^2) \cdot \Tr((\bm{\Sigma} - \bI_d)^2) \right] && \tag{By \cref{eq:upper-bound-third-term}}\\
= &\; \frac{8}{n (n-1)} \cdot \Tr(\bm{\Sigma}^2) \cdot \left[ \Tr(\bm{\Sigma}^2) + (n-1) \cdot \Tr((\bm{\Sigma} - \bI_d)^2) \right]\\
\leq &\; \frac{8}{n (n-1)} \cdot \Tr(\bm{\Sigma}^2) \cdot \left[ \Tr(\bm{\Sigma}^2) + n \cdot \Tr((\bm{\Sigma} - \bI_d)^2) \right] && \tag{Since $\Tr((\bm{\Sigma} - \bI_d)^2) \geq 0$}\\
\leq &\; \frac{8}{n (n-1)} \cdot d \cdot \left( 1 + \frac{b}{\sqrt{n}} \right)^2 \cdot \left( d \cdot \left( 1 + \frac{b}{\sqrt{n}} \right)^2 + n \cdot \Tr((\bm{\Sigma} - \bI_d)^2)
\right) && \tag{Since $\Tr(\bm{\Sigma}^2) = \| \bm{\Sigma} \|_F^2$ and by \cref{lem:X-to-X+I-bound}}\\
= &\; \frac{8}{n (n-1)} \cdot d \cdot \left( 1 + \frac{b}{\sqrt{n}} \right)^2 \cdot \left( d \cdot \left( 1 + \frac{b}{\sqrt{n}} \right)^2 + b^2 \cdot d \right) && \tag{Since $\Tr((\bm{\Sigma} - \bI_d)^2) = \| \bm{\Sigma} - \bI_d \|_F^2 = \frac{b^2 d}{n}$}\\
= &\; \frac{8 d^2}{n (n-1)} \cdot \left( 1 + \frac{b}{\sqrt{n}} \right)^2 \cdot \left( \left( 1 + \frac{b}{\sqrt{n}} \right)^2 + b^2 \right)\\
\leq &\; \frac{16 d^2}{n^2} \cdot \left( 1 + \frac{b}{\sqrt{n}} \right)^2 \cdot \left( \left( 1 + \frac{b}{\sqrt{n}} \right)^2 + b^2 \right) && \tag{Since $n \geq 2$}\\
\leq &\; \frac{64 d^2}{n^2} \cdot \left( 1 + \frac{b^2}{n} \right) \cdot \left( 1 + \frac{b^2}{n} + b^2 \right) && \tag{Since $(a+b)^2 \leq 2a^2 + 2b^2$}
\end{align*}
\end{proof}

\begin{proof}[Proof of \cref{lem:tolerant-ZMGCT}]
Let us define $\Delta_{\eps_1, \eps_2} = \max \left\{ \frac{1}{\eps_1^2}, \left( \frac{\eps_1^2}{\eps_2^2 - \eps_1^2} \right)^2, 2 \left( \frac{\eps_2}{\eps_2^2 - \eps_1^2} \right)^2\right \} > 0$ and suppose $\| \bm{\Sigma} - \bI_d \|_F^2 = \frac{b^2 d}{n}$ for some $b \geq 0$.

The total number of samples $m$ required is
$
nr \in \cO \left( d \cdot \Delta_{\eps_1, \eps_2} \cdot \log\left( \frac{1}{\delta} \right) \right)
$
since \textsc{TolerantZMGCT} uses $n = 3200 \cdot d \cdot \Delta_{\eps_1, \eps_2}$ i.i.d.\ samples in each of the $r \in \cO(\log(\frac{1}{\delta}))$ rounds.

For correctness, we will prove that each round $i \in \{1, \ldots, r\}$ succeeds with probability at least $2/3$.
Then, by Chernoff bound, the majority outcome out of $r \geq \log(\frac{12}{\delta})$ independent tests will be correct with probability at least $1 - \delta$.
    
Now, fix an arbitrary round $i \in \{1, \ldots, r\}$.
\textsc{TolerantZMGCT} uses $n = 3200 \cdot d \cdot \Delta_{\eps_1, \eps_2}$ i.i.d.\ samples to form a statistic $T^{(i)}_n$ (\cref{def:test-statistic-Tn}) and tests against the threshold $\tau = \frac{\eps_2^2 + \eps_1^2}{4}$.

\textbf{Case 1}: $\| \bm{\Sigma} - \bI_d \|_F^2 \leq \eps_1^2$

We see that
\begin{align*}
b^2
&= \frac{n}{d} \cdot \| \bm{\Sigma} - \bI_d \|_F^2 \tag{Since $\| \bm{\Sigma} - \bI_d \|_F^2 = \frac{b^2 d}{n}$}\\
&= 3200 \cdot \Delta_{\eps_1, \eps_2} \cdot \| \bm{\Sigma} - \bI_d \|_F^2 \tag{Since $n = 3200 \cdot d \cdot \Delta_{\eps_1, \eps_2}$}\\
&\leq 3200 \cdot \Delta_{\eps_1, \eps_2} \cdot \eps_1^2 \tag{Since $\| \bm{\Sigma} - \bI_d \|_F^2 \leq \eps_1^2$}
\end{align*}
and
\begin{align*}
1 + \frac{b^2}{n}
&= 1 + \frac{\| \bm{\Sigma} - \bI_d \|_F^2}{d} \tag{Since $\| \bm{\Sigma} - \bI_d \|_F^2 = \frac{b^2 d}{n}$}\\
&\leq 1 + \frac{\eps_1^2}{d} \tag{Since $\| \bm{\Sigma} - \bI_d \|_F^2 \leq \eps_1^2$}\\
&\leq 2 \tag{Since $d \geq \eps_2^2 > \eps_1^2$}
\end{align*}

So,
\begin{align*}
\sigma^2(\texttt{T}_n)
& \leq \frac{64 d^2}{n^2} \cdot \left( 1 + \frac{b^2}{n} \right) \cdot \left( 1 + \frac{b^2}{n} + b^2 \right) \tag{By \cref{lem:test-statistic-variance-upper-bound}}\\
&\leq \frac{64 d^2}{n^2} \cdot 2 \cdot \left( 2 + 3200 \cdot \Delta_{\eps_1, \eps_2} \cdot \eps_1^2 \right) \tag{From above}\\
&= \frac{64 \cdot 2}{3200^2} \cdot \frac{1}{\Delta_{\eps_1, \eps_2}^2} \cdot \left( 2 + 3200 \cdot \Delta_{\eps_1, \eps_2} \cdot \eps_1^2 \right) \tag{Since $n = 3200 \cdot d \cdot \Delta_{\eps_1, \eps_2}$}\\
&\leq \frac{64 \cdot 2}{3200^2} \cdot \frac{1}{\Delta_{\eps_1, \eps_2}^2} \cdot 3202 \cdot \Delta_{\eps_1, \eps_2} \cdot \eps_1^2 \tag{Since $\Delta_{\eps_1, \eps_2} \eps_1^2 \geq 1$}\\
&\leq \frac{64 \cdot 2 \cdot 3202}{3200^2} \cdot (\eps_2^2 - \eps_1^2)^2 \tag{Since $\left( \frac{\eps_1^2}{\eps_2^2 - \eps_1^2} \right)^2 \leq \Delta_{\eps_1, \eps_2}$}
\end{align*}

Chebyshev's inequality then tells us that
\begin{align*}
\Pr\left( \texttt{T}_n > \tau \right)
&= \Pr\left( \texttt{T}_n > \eps_1^2 + \frac{\eps_2^2 - \eps_1^2}{2} \right) \tag{Since $\tau = \frac{\eps_2^2 + \eps_1^2}{2} = \eps_1^2 + \frac{\eps_2^2 - \eps_1^2}{2}$}\\
&\leq \Pr\left( \texttt{T}_n > \| \bm{\Sigma} - \bI_d \|_F^2 + \frac{\eps_2^2 - \eps_1^2}{2} \right) \tag{Since $\| \bm{\Sigma} - \bI_d \|_F^2 \leq \eps_1^2$}\\
&= \Pr\left( \texttt{T}_n > \E[\texttt{T}_n] + \frac{\eps_2^2 - \eps_1^2}{2} \right) \tag{By \cref{lem:expectation-and-variance-of-test-statistic}}\\
&\leq \Pr\left( \left| \texttt{T}_n - \E[\texttt{T}_n] \right| > \frac{\eps_2^2 - \eps_1^2}{2} \right) \tag{Adding absolute sign}\\
&\leq \sigma^2(\texttt{T}_n) \cdot \left( \frac{2}{\eps_2^2 - \eps_1^2} \right)^2 \tag{Chebyshev's inequality}\\
&\leq \frac{64 \cdot 2 \cdot 3202}{3200^2} \cdot (\eps_2^2 - \eps_1^2)^2 \cdot \frac{4}{(\eps_2^2 - \eps_1^2)^2} \tag{From above}\\
&< \frac{1}{3}
\end{align*}
Thus, when $\| \bm{\Sigma} - \bI_d \|_F^2 \leq \eps_1^2$, we have $\Pr\left( \texttt{T}_n < \tau \right) \geq 2/3$ and the $i^{th}$ test outcome will be correctly an $\Accept$ with probability at least $2/3$.

\textbf{Case 2}: $\| \bm{\Sigma} - \bI_d \|_F^2 \geq \eps_2^2$

We can lower bound $b^2$ as follows:
\begin{align*}
{\color{blue}b^2}
&= \frac{n}{d} \cdot \| \bm{\Sigma} - \bI_d \|_F^2 \tag{Since $\| \bm{\Sigma} - \bI_d \|_F^2 = \frac{b^2 d}{n}$}\\
&= 3200 \cdot \Delta_{\eps_1, \eps_2} \cdot \| \bm{\Sigma} - \bI_d \|_F^2 \tag{Since $n = 3200 \cdot d \cdot \Delta_{\eps_1, \eps_2}$}\\
&\geq {\color{blue}3200 \cdot \Delta_{\eps_1, \eps_2} \cdot \eps_2^2} \tag{Since $\| \bm{\Sigma} - \bI_d \|_F^2 \geq \eps_2^2$}
\end{align*}

Meanwhile, we can lower bound $n$ as follows:
\begin{align*}
{\color{red}n}
&= 3200 \cdot d \cdot \Delta_{\eps_1, \eps_2} \tag{Since $n = 3200 \cdot d \cdot \Delta_{\eps_1, \eps_2}$}\\
&\geq 3200 \cdot \eps_2^2 \cdot \Delta_{\eps_1, \eps_2} \tag{Since $d \geq \eps_2^2$}\\
&\geq {\color{red}\frac{3200 \cdot \eps_2^2 \cdot \Delta_{\eps_1, \eps_2}}{\Delta_{\eps_1, \eps_2} \cdot \left( \frac{\eps_2^2 - \eps_1^2}{\eps_2} \right)^2 - 1}} \tag{Since $\Delta_{\eps_1, \eps_2} \geq 2 \left( \frac{\eps_2}{\eps_2^2 - \eps_1^2} \right)^2$}  
\end{align*}

Using these lower bounds on $b^2$ and $n$ (which we color for convenience), we can conclude that $1 + \frac{b^2}{n} \leq \frac{b^2}{3200} \cdot \left( \frac{\eps_2^2 - \eps_1^2}{\eps_2^2} \right)^2$ via the following two equivalences:
\[
1 + \frac{b^2}{n} \leq \frac{b^2}{3200} \cdot \left( \frac{\eps_2^2 - \eps_1^2}{\eps_2^2} \right)^2 \iff {\color{blue}b^2} \geq \frac{n}{\frac{n}{3200} \cdot \left( \frac{\eps_2^2 - \eps_1^2}{\eps_2^2} \right)^2 - 1}
\]
and
\[
{\color{blue}3200 \cdot \Delta_{\eps_1, \eps_2} \cdot \eps_2^2}
\geq \frac{n}{\frac{n}{3200} \cdot \left( \frac{\eps_2^2 - \eps_1^2}{\eps_2^2} \right)^2 - 1}
\iff
{\color{red}n}
\geq \frac{3200 \cdot \Delta_{\eps_1, \eps_2} \cdot \eps_2^2}{\Delta_{\eps_1, \eps_2} \cdot \eps_2^2 \cdot \left( \frac{\eps_2^2 - \eps_1^2}{\eps_2^2} \right)^2 - 1}
= {\color{red}\frac{3200 \cdot \eps_2^2 \cdot \Delta_{\eps_1, \eps_2}}{\Delta_{\eps_1, \eps_2} \cdot \left( \frac{\eps_2^2 - \eps_1^2}{\eps_2} \right)^2 - 1}}
\]

So,
\begin{align*}
\sigma^2(\texttt{T}_n)
& \leq \frac{64 d^2}{n^2} \cdot \left( 1 + \frac{b^2}{n} \right) \cdot \left( 1 + \frac{b^2}{n} + b^2 \right) \tag{By \cref{lem:test-statistic-variance-upper-bound}}\\
&\leq 64 \cdot 2 \cdot \frac{d^2}{n^2} \cdot \left( \frac{b^2}{3200} \cdot \left( \frac{\eps_2^2 - \eps_1^2}{\eps_2^2} \right)^2 \right) \cdot \left( \frac{b^2}{3200} \cdot \left( \frac{\eps_2^2 - \eps_1^2}{\eps_2^2} \right)^2 + b^2 \right) \tag{Since $1 + \frac{b^2}{n} \leq \frac{b^2}{3200} \cdot \left( \frac{\eps_2^2 - \eps_1^2}{\eps_2^2} \right)^2$}\\
&= \frac{64 \cdot 2 \cdot 2}{3200} \cdot \left( \frac{\eps_2^2 - \eps_1^2}{\eps_2^2} \right)^2 \cdot \frac{d^2}{n^2} \cdot b^4 \tag{Since $\frac{1}{3200} \left( \frac{\eps_2^2 - \eps_1^2}{\eps_2^2} \right)^2 \leq 1$}\\
&= \frac{64 \cdot 2 \cdot 2}{3200} \cdot \left( \frac{\eps_2^2 - \eps_1^2}{\eps_2^2} \right)^2 \cdot \| \bm{\Sigma} - \bI_d \|_F^4 \tag{Since $\| \bm{\Sigma} - \bI_d \|_F^2 = \frac{b^2 d}{n}$}
\end{align*}

Chebyshev's inequality then tells us that
\begin{align*}
\Pr\left( \texttt{T}_n < \tau \right)
&= \Pr\left( \texttt{T}_n < \eps_2^2 \cdot \left( 1 - \frac{\eps_2^2 - \eps_1^2}{2 \eps_2^2} \right) \right) \tag{Since $\tau = \frac{\eps_2^2 + \eps_1^2}{2} = \eps_2^2 - \frac{\eps_2^2 - \eps_1^2}{2} = \eps_2^2 \cdot \left( 1 - \frac{\eps_2^2 - \eps_1^2}{2 \eps_2^2} \right)$}\\
&\leq \Pr\left( \texttt{T}_n < \| \bm{\Sigma} - \bI_d \|_F^2 \cdot \left( 1 - \frac{\eps_2^2 - \eps_1^2}{2 \eps_2^2} \right) \right) \tag{Since $\| \bm{\Sigma} - \bI_d \|_F^2 \geq \eps_2^2$}\\
&= \Pr\left( \| \bm{\Sigma} - \bI_d \|_F^2 - \texttt{T}_n > \| \bm{\Sigma} - \bI_d \|_F^2 \cdot  \frac{\eps_2^2 - \eps_1^2}{2 \eps_2^2} \right) \tag{Rearranging}\\
&= \Pr\left( \E[\texttt{T}_n] - \texttt{T}_n > \| \bm{\Sigma} - \bI_d \|_F^2 \cdot  \frac{\eps_2^2 - \eps_1^2}{2 \eps_2^2} \right) \tag{By \cref{lem:expectation-and-variance-of-test-statistic}}\\
&\leq \Pr\left( \left| \E[\texttt{T}_n] - \texttt{T}_n \right| > \| \bm{\Sigma} - \bI_d \|_F^2 \cdot  \frac{\eps_2^2 - \eps_1^2}{2 \eps_2^2} \right) \tag{Adding absolute sign}\\
&\leq \sigma^2(\texttt{T}_n) \cdot \left( \frac{1}{\| \bm{\Sigma} - \bI_d \|_F^2} \cdot \frac{2 \eps_2^2}{\eps_2^2 - \eps_1^2} \right)^2 \tag{Chebyshev's inequality}\\
&\leq \frac{64 \cdot 2 \cdot 2}{3200} \cdot \left( \frac{\eps_2^2 - \eps_1^2}{\eps_2^2} \right)^2 \cdot \| \bm{\Sigma} - \bI_d \|_F^4 \cdot \left( \frac{1}{\| \bm{\Sigma} - \bI_d \|_F^2} \cdot \frac{2 \eps_2^2}{\eps_2^2 - \eps_1^2} \right)^2 \tag{From above}\\
&= \frac{64 \cdot 2 \cdot 2 \cdot 4}{3200}\\
&< \frac{1}{3}
\end{align*}
Thus, when $\| \bm{\Sigma} - \bI_d \|_F^2 \geq \eps_2^2$, we have $\Pr\left( \texttt{T}_n > \tau \right) \geq 2/3$ and the $i^{th}$ test outcome will be correctly an $\Reject$ with probability at least $2/3$.
\end{proof}

\tolerantcovariancetester*
\begin{proof}
Use the guarantee of \cref{lem:tolerant-ZMGCT} on \textsc{TolerantZMGCT} (\cref{alg:zero-mean-covariance-testing}) with parameters $\eps_1^2 = \eps^2$ and $\eps_2^2 = 2 \eps^2$.
\end{proof}

\subsection{Basic results from \texorpdfstring{\cref{sec:preliminaries}}{Section 2}}
\label{sec:appendix-preliminaries-proofs}

\rotatingnorm*
\begin{proof}
Exercise 5.6.P58(b) of \cite{horn2012matrix} tells us that $\| \bA \bB \| = \| \bB \bA \|$ when $\bA$ normal and $\bB$ is Hermitian.
Since normal matrices are invertible and every real matrix is Hermitian, the claim follows.
\end{proof}

\applyingrotatingnorm*
\begin{proof}
$
\| \bA^{-1/2} \bB \bA^{-1/2} - \bI_d \|
= \| (\bA^{-1/2} \bB - \bA^{1/2}) \bA^{-1/2}\|
= \| \bA^{-1/2} (\bA^{-1/2} \bB - \bA^{1/2}) \|
= \| \bA^{-1} \bB - \bI_d \|.
$
\end{proof}

\traceinequality*
\begin{proof}
Let $\lambda_1(\bM), \ldots, \lambda_d(\bM)$ denote the eigenvalues of a matrix $\bM \in \R^{d \times d}$.
\begin{align*}
\Tr(\bA \bB \bC)
&\leq \sum_i \lambda_i(\bA \bB) \cdot \lambda_i(\bC) \tag{by von Neumann trace inequality}\\
&= \sum_i \lambda_i(\bB \bA) \cdot \lambda_i(\bC) \tag{e.g.\ see Theorem 1.3.22 of \cite{horn2012matrix}}\\
&\leq \sum_i | \lambda_i(\bB \bA) \cdot \lambda_i(\bC) |\\
&\leq \left\| \begin{pmatrix}\lambda_1(\bB \bA)\\ \vdots\\ \lambda_d(\bB \bA)\end{pmatrix} \right\|_1 \cdot \left\| \begin{pmatrix}\lambda_1(\bC)\\ \vdots\\ \lambda_d(\bC)\end{pmatrix} \right\|_\infty \tag{H\"{o}lder's inequality}\\
&= \sum_i |\lambda_i(\bB \bA)| \cdot \max_i \lambda_i(\bC) \tag{Definitions of vector $\ell_1$ and $\ell_\infty$ norms}\\
&\leq \sum_i |\lambda_i(\bB \bA)| \cdot \| \bC \|_2 \tag{Definition of matrix spectral norm}
\end{align*}

It remains to argue that $\sum_i |\lambda_i(\bB \bA)| \leq \| \vec(\bB \bA) \|_1$.
To this end, consider the singular value decomposition (SVD) of $\bB \bA = \bU \bm{\Sigma} \bV^\top$ with unitary matrices $\bU, \bV$ and diagonal matrix $\bm{\Sigma} = \textrm{diag}(\sigma_1, \ldots, \sigma_d)$.
Let us denote the eigenvalues of $\bB \bA$ by $\sigma_1, \ldots, \sigma_d$ and the columns of $\bB \bA$ by $\bz_1, \ldots, \bz_d \in \R^d$.
Then,
\begin{align*}
\sum_i |\lambda_i(\bB \bA)|
&\leq \sum_i \sigma_i \tag{e.g.\ see Equation (7.3.17) in \cite{horn2012matrix}}\\
&= \Tr(\bm{\Sigma}) \tag{By definition of $\bm{\Sigma}$}\\
&= \Tr(\bV^\top \bV \bU^\top \bU \bm{\Sigma}) \tag{Since $\bU$ and $\bV$ are unitary matrices}\\
&= \Tr(\bV \bU^\top \bU \bm{\Sigma} \bV^\top) \tag{By cyclic property of trace}\\
&= \Tr(\bV \bU^\top \bB \bA) \tag{By SVD of $\bB \bA$}\\
&= \sum_{i=1}^d (\bV \bU^\top \bz_i)_i \tag{By definition of trace}\\
&\leq \sum_{i=1}^d \| \bV \bU^\top \bz_i \|_2 \tag{Since $(\bV \bU^\top \bz_i)_i^2$ is just one term in summation of $\| \bV \bU^\top \bz_i \|_2^2$}\\
&= \sum_{i=1}^d \| \bz_i \|_2 \tag{Since $\bU$ and $\bV$ are unitary matrices}\\
&\leq \sum_{i=1}^d \| \bz_i \|_1 \tag{Since $\ell_2 \leq \ell_1$}\\
&= \sum_{i=1}^d \sum_{j=1}^d |(\bB \bA)_{i,j}| \tag{By definition of vector $\ell_1$ norm}\\
&= \| \vec(\bB \bA) \|_1 \tag{By definition of $\| \vec(\bB \bA) \|_1$}
\end{align*}
Putting together, we get $\Tr(\bA \bB \bC) \leq \sum_i |\lambda_i(\bB \bA)| \cdot \| \bC \|_2 \leq \| \vec(\bB \bA) \|_1 \cdot \| \bC \|_2$ as desired.
\end{proof}

\vectorizedinequalities*
\begin{proof}
To see $\| \vec(\bA + \bB) \|_1 \leq \| \vec(\bA) \|_1 + \| \vec(\bB) \|_1$, observe that
\[
\| \vec(\bA + \bB) \|_1
= \sum_{i=1}^d \sum_{j=1}^d | \bA_{ij} + \bB_{ij} |
\leq \sum_{i=1}^d \sum_{j=1}^d | \bA_{ij} | + \sum_{i=1}^d \sum_{j=1}^d | \bB_{ij} |
= \| \vec(\bA) \|_1 + \| \vec(\bB) \|_1
\]
To see $\| \vec(\bA \bB) \|_1 \leq \| \vec(\bA) \|_1 \cdot \| \vec(\bB) \|_1$, observe that
\[
\| \vec(\bA \bB) \|_1
= \sum_{i=1}^d \sum_{j=1}^d \sum_{k=1}^d | \bA_{ij} \bB_{jk} |
\leq \left( \sum_{i=1}^d \sum_{j=1}^d | \bA_{ij}  | \right) \cdot \left( \sum_{j=1}^d \sum_{k=1}^d | \bB_{jk} | \right)
= \| \vec(\bA) \|_1 \cdot \| \vec(\bB) \|_1
\]
\end{proof}

\klknownfact*
\begin{proof}
Let $\cP \sim N(\bm{\mu}_{\cP}, \bm{\Sigma}_{\cP})$ and $\cQ \sim N(\bm{\mu}_{\cQ}, \bm{\Sigma}_{\cQ})$ be two $d$-dimensional multivariate Gaussian distributions where $\bm{\Sigma}_{\cP}$ and $\bm{\Sigma}_{\cQ}$ are full rank invertible covariance matrices.

By definition, the KL divergence between $\cP$ and $\cQ$ is
\begin{equation}
\label{eq:KL-definition}
\kl(\cP, \cQ)
= \frac{1}{2} \cdot \left( \Tr(\bm{\Sigma}_{\cQ}^{-1} \bm{\Sigma}_{\cP}) - d + (\bm{\mu}_{\cQ} - \bm{\mu}_{\cP})^\top \bm{\Sigma}_{\cQ}^{-1} (\bm{\mu}_{\cQ} - \bm{\mu}_{\cP}) + \ln \left( \frac{\det \bm{\Sigma}_{\cQ}}{\det \bm{\Sigma}_{\cP}} \right) \right)
\end{equation}

Let us define the matrix $\bX = \bm{\Sigma}_{\cQ}^{-1/2} \bm{\Sigma}_{\cP} \bm{\Sigma}_{\cQ}^{-1/2} - \bI_d$ with eigenvalues $\lambda_1, \ldots, \lambda_d$.
Note that $\bX$ is invertible because $\bm{\Sigma}_{\cP}$ and $\bm{\Sigma}_{\cQ}$ are invertible, so $\lambda_1, \ldots, \lambda_d > 0$.
Then, \cref{eq:KL-definition} can be upper bounded as
\begin{multline}
\label{eq:KL-definition-equal-means}
\kl(\cP, \cQ)
= \frac{1}{2} \cdot \left( \Tr(\bm{\Sigma}_{\cQ}^{-1} \bm{\Sigma}_{\cP}) - d + (\bm{\mu}_{\cQ} - \bm{\mu}_{\cP})^\top \bm{\Sigma}_{\cQ}^{-1} (\bm{\mu}_{\cQ} - \bm{\mu}_{\cP}) + \ln \left( \frac{\det \bm{\Sigma}_{\cQ}}{\det \bm{\Sigma}_{\cP}} \right) \right)\\
\leq \frac{1}{2} \left( (\bm{\mu}_{\cQ} - \bm{\mu}_{\cP})^\top \bm{\Sigma}_{\cQ}^{-1} (\bm{\mu}_{\cQ} - \bm{\mu}_{\cP}) + \| \bX \|_F^2 \right)
\end{multline}

This is because $\Tr(\bm{\Sigma}_{\cQ}^{-1} \bm{\Sigma}_{\cP}) = \Tr(\bm{\Sigma}_{\cQ}^{-1/2} \bm{\Sigma}_{\cP} \bm{\Sigma}_{\cQ}^{-1/2}) = \Tr(\bX + \bI_d) = \Tr(\bX) + d$ and
\begin{multline*}
-\ln \left( \frac{\det \bm{\Sigma}_{\cQ}}{\det \bm{\Sigma}_{\cP}} \right)
= \ln \det \left( \bm{\Sigma}_{\cQ}^{-1} \bm{\Sigma}_{\cP} \right)
= \ln \det (\bX + \bI_d)
= \ln \prod_{i=1}^d (1 + \lambda_i)\\
= \sum_{i=1}^d \ln (1 + \lambda_i)
\geq \sum_{i=1}^d (\lambda_i - \lambda_i^2)
= \Tr(\bX) - \sum_{i=1}^d \lambda_i^2
= \Tr(\bX) - \| \bX \|_F^2
\end{multline*}
where the inequality holds due to $\lambda_1, \ldots, \lambda_d > 0$.

When $\bm{\Sigma}_{\cP} = \bm{\Sigma}_{\cQ} = \bI_d$, \cref{eq:KL-definition} reduces to $\kl(\cP, \cQ) = \frac{1}{2} \| \bm{\mu}_{\cQ} - \bm{\mu}_{\cP} \|_2^2$.
Meanwhile, when $\bm{\mu}_{\cP} = \bm{\mu}_{\cQ}$, \cref{eq:KL-definition-equal-means} reduces to $\kl(\cP, \cQ) \leq \frac{1}{2} \left( \| \bX \|_F^2 \right)$.
\end{proof}

\knownpropertiesofempiricalcovariance*
\begin{proof}
For item 1, let $1 \leq r \leq d$ be the rank of $\bm{\Sigma}$.
We consider the case of the $d$-dimensional Gaussian with zero mean and covariance $\bm{\Gamma}_r = \begin{bmatrix}\bI_r & \mathbf{0}\\ \mathbf{0} & \mathbf{0}\end{bmatrix}$, where $\bI_r$ denotes the $r$-dimensional identity matrix and the zero-padding is added when $r < d$.
Note that there is an invertible transformation between samples from $N(\bm{0}, \bm{\Gamma}_r)$ and $N(\bm{0}, \bm{\Sigma})$ with samples from $N(0, \bm{\Gamma}_r)$ having the $r+1, \ldots, d$ coordinates be fixed to $0$.
Now, let us denote the $i$-th standard basis vector by $\bm{e}_i$ and apply an induction argument on $r$ from $1$ to $d$.
The base case ($r = 1$) is obviously true since a single sample $\bm{x}_1$ will span $\{\bm{e}_1\}$ unless $\bm{x}_1 = \bm{0}$, which will happen with probability $0$.
When $r > 1$, by strong induction, $r$ samples $\bm{x}_1, \ldots, \bm{x}_r$ will not span $\{\bm{e}_1, \ldots, \bm{e}_r\}$ only if the $r$-th sample $\bm{x}_r$ lies in the subspace spanned by $\bm{x}_1, \ldots, \bm{x}_{r-1}$.
This is a measure $0$ event under the $N(\bm{0},\bm{\Gamma}_r)$ measure.

For item 2, see Fact 3.4 of \cite{kamath2019privately}.
\end{proof}

\noncentralchisquare*
\begin{proof}
The first item follows from the definition of the non-central chi-squared distribution, noting that the random vector $\bz_n$ is distributed as $N(\sqrt{n} \cdot \bm{\mu}, \bI_d)$.
The second and third items follow from Theorems 3 and 4 of \cite{Ghosh2021} respectively.
\end{proof}

\gaussianmaxconcentration*
\begin{proof}
Since $\bg_1, \ldots, \bg_n \sim N(0, \bI_d)$, we see that $\by = \bg_1 + \ldots + \bg_n \sim N(0, n \bI_d)$.
Furthermore, each coordinate $i \in [d]$ of $\by_i = (y_1, \ldots, y_d)$ is distributed according to $N(0, n)$.
By standard Gaussian tail bounds, we know that $\Pr(|y_i| \geq t) \leq 2 \exp \left(- \frac{t^2}{2n} \right)$ for any $i \in [d]$ and $t > 0$.
So,
\begin{align*}
\Pr \left( \left\| \sum_{i=1}^n \bg_i \right\|_\infty \geq \sqrt{2n \log \left( \frac{2d}{\delta} \right)} \right)
&= \Pr \left( \left\| \by \right\|_\infty \geq \sqrt{2n \log \left( \frac{2d}{\delta} \right)} \right)\\
&= \Pr \left( \max_{i \in [d]} \left\| y_i \right\| \geq \sqrt{2n \log \left( \frac{2d}{\delta} \right)} \right)\\
&\leq \sum_{i=1}^d \Pr \left( \left\| y_i \right\| \geq \sqrt{2n \log \left( \frac{2d}{\delta} \right)} \right) \tag{Union bound over all $d$ coordinates}\\
&\leq 2d \exp \left(- \frac{2n \log \left( \frac{2d}{\delta} \right)}{2n} \right) \tag{Setting $t = 2n \log \left( \frac{2d}{\delta} \right)$}\\
&= \delta
\end{align*}
\end{proof}

%% file: appendix-identity-covariance.tex
\section{Identity covariance setting}

\subsection{Guarantees of \textsc{ApproxL1}}
\label{sec:appendix-approxl1-guarantees}

Here, we show that the guarantees of the \textsc{ApproxL1} algorithm (\cref{alg:approxl1}).

\guaranteesofapproxlone*
\begin{proof}
We begin by stating some properties of $o_1, \ldots, o_w$.
Fix an arbitrary index $j \in \{1, \ldots, w\}$ and suppose $o_j$ is \emph{not} a $\Fail$, i.e.\ the tolerant tester of \cref{lem:tolerant-mean-tester} outputs $\Accept$ for some $i^* \in \{1, 2, \ldots, \lceil \log_2 \zeta/\alpha \rceil\}$.
Note that \textsc{ApproxL1} sets $o_j = \ell_{i^*}$ and the tester outputs $\Reject$ for all smaller indices $i \in \{1, \ldots, i^* - 1\}$.
Since the tester outputs $\Accept$ for $i^*$, we have that $\| \bm{\mu}_{\bB_j} \|_2 \leq 2 \ell_{i^*} = 2 o_j$.
Meanwhile, if $i^* > 1$, then $\| \bm{\mu}_{\bB_j} \|_2 > \ell_{i^* - 1} = \ell_{i^*}/2 = o_j/2$ since the tester outputs $\Reject$ for $i^* - 1$.
Thus, we see that
\begin{itemize}
    \item When $o_j$ is not $\Fail$, we have $\| \bm{\mu}_{\bB_j} \|_2 \leq 2 o_j$.
    \item When $\| \bm{\mu}_{\bB_j} \|_2 \leq 2 \alpha$, we have $i^* = 1$ and $o_j = \ell_1 = \alpha$.
    \item When $\| \bm{\mu}_{\bB_j} \|_2 > 2 \alpha = 2 \ell_1$, we have $i^* > 1$ and so $o_j < 2 \| \bm{\mu}_{\bB_j} \|_2$.
\end{itemize}

\paragraph{Success probability.}
Fix an arbitrary index $i \in \{1, 2, \ldots, \lceil \log_2 \zeta/\alpha \rceil\}$ with $\ell_i = 2^{i-1} \alpha$, where $\ell_i \leq \ell_1 = \alpha$ for any $i$.
We invoke the tolerant tester with $\eps_2 = 2 \ell_i = 2 \eps_1$, so the $i^{th}$ invocation uses at most $n_{k, \eps} \cdot r_{\delta}$ i.i.d.\ samples to succeed with probability at least $1 - \delta$; see \cref{def:m-d-alpha-delta} and \cref{alg:isotropic-gaussian-mean-testing}.
So, with $m(k, \alpha, \delta')$ samples, \emph{any} call to the tolerant tester succeeds with probability at least $1 - \delta'$, where $\delta' = \frac{\delta}{w \cdot \lceil \log_2 \zeta/\alpha \rceil}$.
By construction, there will be at most $w \cdot \lceil \log_2 \zeta/\alpha \rceil$ calls to the tolerant tester.
Therefore, by union bound, \emph{all} calls to the tolerant tester jointly succeed with probability at least $1 - \delta$.

\paragraph{Property 1.}
When \textsc{ApproxL1} outputs $\Fail$, there exists a $\Fail$ amongst $\{o_1, \ldots, o_w\}$.
For any fixed index $j \in \{1, \ldots, w\}$, this can only happen when all calls to the tolerant tester outputs $\Reject$.
This means that $\| \bm{x}_{\bB_j} \|_2 > \eps_1 = \ell_i = 2^{i-1} \cdot \alpha$ for all $i \in \{1, 2, \ldots, \lceil \log_2 \zeta/\alpha \rceil\}$.
In particular, this means that $\| \bm{x}_{\bB_j} \|_2 > \zeta/2$.

\paragraph{Property 2.}
When \textsc{ApproxL1} outputs $\lambda = 2 \sum_{j=1}^w \sqrt{|\bB_j|} \cdot o_j \in \R$, we can lower bound $\lambda$ as follows:
\begin{align*}
\lambda
&= 2 \sum_{j=1}^w \sqrt{|\bB_j|} \cdot o_j\\
&\geq 2 \sum_{j=1}^w \sqrt{|\bB_j|} \cdot \frac{\| \bm{\mu}_{\bB_j} \|_2}{2} \tag{since $\| \bm{\mu}_{\bB_j} \|_2 \leq 2 o_j$}\\
&\geq \sum_{j=1}^w \| \bm{\mu}_{\bB_j} \|_1 \tag{since $\| \bm{\mu}_{\bB_j} \|_1 \leq \sqrt{|\bB_j|} \cdot \| \bm{\mu}_{\bB_j} \|_2$}\\
&= \| \bm{\mu} \|_1 \tag{since $\sum_{j=1}^w \| \bm{\mu}_{\bB_j} \|_1 = \| \bm{\mu}_{\bB_j} \|_1$}
\end{align*}
That is, $\lambda \geq \| \bm{\mu} \|_1$.
Meanwhile, we can also upper bound $\lambda$ as follows:
\begin{align*}
\lambda
&= 2 \sum_{j=1}^w \sqrt{|\bB_j|} \cdot o_j\\
&\leq 2 \sqrt{k} \sum_{j=1}^w o_j \tag{since $|\bB_j| \leq k$}\\
&= 2 \sqrt{k} \cdot \left( \sum_{\substack{j=1\\ \| \bm{\mu}_{\bB_j} \|_2 \leq 2 \alpha}}^w o_j + \sum_{\substack{j=1\\ \| \bm{\mu}_{\bB_j} \|_2 > 2 \alpha}}^w o_j \right) \tag{partitioning the blocks based on $\| \bm{\mu}_{\bB_j} \|_2$ versus $2 \alpha$}\\
&= 2 \sqrt{k} \cdot \left( \sum_{\substack{j=1\\ \| \bm{\mu}_{\bB_j} \|_2 \leq 2 \alpha}}^w \alpha + \sum_{\substack{j=1\\ \| \bm{\mu}_{\bB_j} \|_2 > 2 \alpha}}^w o_j \right) && \tag{since $\| \bm{\mu}_{\bB_j} \|_2 \leq 2 \alpha$ implies $o_j = \alpha$}\\
&\leq 2 \sqrt{k} \cdot \left( \sum_{\substack{j=1\\ \| \bm{\mu}_{\bB_j} \|_2 \leq 2 \alpha}}^w \alpha + \sum_{\substack{j=1\\ \| \bm{\mu}_{\bB_j} \|_2 > 2 \alpha}}^w 2 \| \bm{\mu}_{\bB_j} \|_2 \right) && \tag{since $\| \bm{\mu}_{\bB_j} \|_2 > 2 \alpha$ implies $o_j \leq 2 \| \bm{\mu}_{\bB_j} \|_2$}\\
&\leq 2 \sqrt{k} \cdot \left( \sum_{\substack{j=1\\ \| \bm{\mu}_{\bB_j} \|_2 \leq 2 \alpha}}^w \alpha + 2 \sum_{\substack{j=1\\ \| \bm{\mu}_{\bB_j} \|_2 > 2 \alpha}}^w \| \bm{\mu}_{\bB_j} \|_1 \right) && \tag{since $\| \bm{\mu}_{\bB_j} \|_2 \leq \| \bm{\mu}_{\bB_j} \|_1$}\\
&\leq 2 \sqrt{k} \cdot \left( \lceil d/k \rceil \cdot \alpha + 2 \sum_{\substack{j=1\\ \| \bm{\mu}_{\bB_j} \|_2 > 2 \alpha}}^w \| \bm{\mu}_{\bB_j} \|_1 \right) && \tag{since $|\{j \in [w]: \bm{\mu}_{\bB_j} \|_2 \leq 2 \alpha\}| \leq w$}\\
&\leq 2 \sqrt{k} \cdot \left( \lceil d/k \rceil \cdot \alpha + 2 \| \bm{\mu} \|_1 \right) && \tag{since $\sum_{\substack{j=1\\ \| \bm{\mu}_{\bB_j} \|_2 > 2 \alpha}}^w \| \bm{\mu}_{\bB_j} \|_1 \leq \sum_{j=1}^w \| \bm{\mu}_{\bB_j} \|_1 = \| \bm{\mu}_{\bB_j} \|_1$}
\end{align*}
That is, $\lambda \leq 2 \sqrt{k} \cdot \left( \lceil d/k \rceil \cdot \alpha + 2 \| \bm{\mu} \|_1 \right)$.
The property follows by putting together both bounds.
\end{proof}

\subsection{Deferred derivation}
\label{sec:appendix-identity-covariance-deferred-derivation}

Here, we show how to derive \cref{eq:mean-estimation-program-optimality-after-manipulation} from \cref{eq:mean-estimation-program-optimality}.

For any two vectors $\ba, \bb \in \R^d$, observe that $\| \ba - \bb \|_2^2 = \langle \ba - \bb, \ba - \bb \rangle = (\ba - \bb)^\top (\ba - \bb) = \ba^\top \ba - 2 \ba^\top \bb + \bb^\top \bb$, since $\ba^\top \bb = \bb^\top \ba$ is just a number.
So,
\begin{align*}
\frac{1}{n} \sum_{i=1}^n \| \by_i - \wh{\bm{\mu}} \|_2^2
&= \frac{1}{n} \sum_{i=1}^n \left( \by_i^\top \by_i - {\color{blue}2 \by_i^\top \wh{\bm{\mu}}} + {\color{blue}\wh{\bm{\mu}}^\top \wh{\bm{\mu}}} \right)\\
\frac{1}{n} \sum_{i=1}^n \| \by_i - X \bm{\mu} \|_2^2
&= \frac{1}{n} \sum_{i=1}^n \left( \by_i^\top \by_i - {\color{blue}2 \by_i^\top \bm{\mu}} + {\color{blue}\bm{\mu}^\top \bm{\mu}} \right)
\end{align*}

Therefore,
\begin{align*}
\| \wh{\bm{\mu}} - \bm{\mu} \|_2^2
& = \frac{1}{n} \sum_{i=1}^n \| \wh{\bm{\mu}} - \bm{\mu} \|_2^2\\
& = \frac{1}{n} \sum_{i=1}^n \left( {\color{blue}\wh{\bm{\mu}}^\top \wh{\bm{\mu}}} - {\color{red}2 \bm{\mu}^\top \wh{\bm{\mu}}} + {\color{red}\bm{\mu}^\top \bm{\mu}} \right)\\
& \leq \frac{1}{n} \sum_{i=1}^n \left( {\color{blue}2 \by_i^\top \wh{\bm{\mu}}} - {\color{blue}2 \by_i^\top \bm{\mu}} + {\color{blue}\bm{\mu}^\top \bm{\mu}} - {\color{red}2 \bm{\mu}^\top \wh{\bm{\mu}}} + {\color{red} \bm{\mu}^\top \bm{\mu}} \right) \tag{Since \cref{eq:mean-estimation-program-optimality-after-manipulation} tells us that $\frac{1}{n} \sum_{i=1}^n \| \by_i - \wh{\bm{\mu}} \|_2^2 \leq \frac{1}{n} \sum_{i=1}^n \| \by_i - \bm{\mu} \|_2^2$}\\
& = \frac{2}{n} \sum_{i=1}^n \left( \left( \bm{\mu} + \bg_i \right)^\top \left( \wh{\bm{\mu}} - \bm{\mu} \right) - \bm{\mu}^\top \wh{\bm{\mu}} + \bm{\mu}^\top \bm{\mu} \right) \tag{Since $\by_i = \bm{\mu} + \bg_i$}\\
& = \frac{2}{n} \sum_{i=1}^n \left( \bg_i^\top \left( \wh{\bm{\mu}} - \bm{\mu} \right) \right)\\
& = \frac{2}{n} \sum_{i=1}^n \langle \bg_i, \wh{\bm{\mu}} - \bm{\mu} \rangle\\
& = \frac{2}{n} \langle \sum_{i=1}^n \bg_i, \wh{\bm{\mu}} - \bm{\mu} \rangle \tag{Linearity of inner product}
\end{align*}
establishing \cref{eq:mean-estimation-program-optimality-after-manipulation} as desired.

%% file: appendix-general-covariance.tex
\section{General covariance setting}

\subsection{The adjustments}
\label{sec:appendix-the-adjustments}

Here, we provide the deferred proofs of \cref{lem:preconditioning-adjustment} and \cref{lem:probabilistic-partitioning-construction} from \cref{sec:the-adjustments}.

\preconditioningadjustment*
\begin{proof}
Suppose $\wh{\bm{\Sigma}} \in \R^{d \times d}$ be the empirical covariance constructed from $n = d$ i.i.d.\ samples from $N(\bm{0}, \bm{\Sigma})$.
Let $\lambda_1 \leq \ldots \leq \lambda_d$ and $\wh{\lambda}_1 \leq \ldots \leq \wh{\lambda}_d$ be the eigenvalues of $\bm{\Sigma}$ and $\wh{\bm{\Sigma}}$ respectively.
By \cref{lem:known-properties-of-empirical-covariance}, we know that:
\begin{itemize}
    \item With probability 1, we have that $\wh{\bm{\Sigma}}$ and $\bm{\Sigma}$ share the same eigenspace.
    \item With probability at least $1 - \delta$, we have $\frac{\wh{\lambda}_1}{\lambda_1} \leq 1 + c_0 \cdot \sqrt{\frac{d + \log 1/\delta}{d}}$ for some absolute constant $c_0$.
\end{itemize}

Let $\wh{\bv}_1, \ldots, \wh{\bv}_d$ be the eigenvectors corresponding to the eigenvalues $\wh{\lambda}_1, \ldots, \wh{\lambda}_d$.
Define the following terms:
\begin{itemize}
    \item $\bV_{\textrm{small}} = \{i \in [d]: \wh{\lambda}_i < 1$\} and $\bV_{\textrm{big}} = [d] \setminus \bV_{\textrm{small}}$
    \item $\bm{\Pi}_{\textrm{small}} = \sum_{i \in \bV_{\textrm{small}}} \wh{\bv}_i \wh{\bv}_i^\top$ and $\bm{\Pi}_{\textrm{big}} = \sum_{i \in \bV_{\textrm{big}}} \wh{\bv}_i \wh{\bv}_i^\top$
    \item $\bA = \sqrt{k} \bm{\Pi}_{\textrm{small}} + \bm{\Pi}_{\textrm{big}}$, where $k = \left(1 + c_0 \cdot \sqrt{\frac{d + \log 1/\delta}{n}} \right) \cdot \frac{1}{\wh{\lambda}_1}$
\end{itemize}

We first argue that the smallest eigenvalue of $\bA \bm{\Sigma} \bA$ is at least $1$, i.e.\ $\lambda_{\min}(\bA \bm{\Sigma} \bA) \geq 1$.
To show this, it suffices to show that $\bu^\top \bA \bm{\Sigma} \bA \bu \geq 1$ for any unit vector $\bu \in \R^d$.
By definition,
\[
\bu^\top \bA \bm{\Sigma} \bA \bu
= k \bu^\top \bm{\Pi}_{\textrm{small}} \bm{\Sigma} \bm{\Pi}_{\textrm{small}} \bu + \bu^\top \bm{\Pi}_{\textrm{big}} \bm{\Sigma} \bm{\Pi}_{\textrm{big}} \bu
\]
since the cross terms are zero because $\bu^\top \bm{\Pi}_{\textrm{small}} \bm{\Sigma} \bm{\Pi}_{\textrm{big}} \bu = \bu^\top \bm{\Pi}_{\textrm{big}} \bm{\Sigma} \bm{\Pi}_{\textrm{small}} \bu = 0$.

Now, observe that $\bu^\top \bm{\Pi}_{\textrm{small}} \bm{\Sigma} \bm{\Pi}_{\textrm{small}} \bu \geq \lambda_1 \cdot \| \bm{\Pi}_{\textrm{small}} \bu \|_2^2$ and $\bu^\top \bm{\Pi}_{\textrm{big}} \bm{\Sigma} \bm{\Pi}_{\textrm{big}} \bu \geq \| \bm{\Pi}_{\textrm{big}} \bu \|_2^2$.
Meanwhile, by Pythagoras theorem, we know that $\| \bm{\Pi}_{\textrm{small}} \bu \|_2^2 + \| \bm{\Pi}_{\textrm{big}} \bu \|_2^2 = 1$.
Therefore,
\begin{align*}
\bu^\top \bA \bm{\Sigma} \bA \bu 
= & k \bu^\top \bm{\Pi}_{\textrm{small}} \bm{\Sigma} \bm{\Pi}_{\textrm{small}} \bu + \bu^\top \bm{\Pi}_{\textrm{big}} \bm{\Sigma} \bm{\Pi}_{\textrm{big}} \bu\\
\geq & k \lambda_1 \cdot \| \bm{\Pi}_{\textrm{small}} \bu \|_2^2 + \| \bm{\Pi}_{\textrm{big}} \bu \|_2^2\\
\geq & \left( \| \bm{\Pi}_{\textrm{small}} \bu \|_2^2 +  \| \bm{\Pi}_{\textrm{big}} \bu \|_2^2 \right)\\
= & 1
\end{align*}
where the last inequality is because $k = \left(1 + c_0 \cdot \sqrt{\frac{d + \log 1/\delta}{n}} \right) \cdot \frac{1}{\wh{\lambda}_1} \geq \frac{1}{\lambda_1}$.

To complete the proof, note that for any full rank PSD matrix $\wt{\bm{\Sigma}} \in \R^{d \times d}$, we have
\begin{align*}
\| (\bA \wt{\bm{\Sigma}} \bA)^{-1/2} \bA \bm{\Sigma} \bA (\bA \wt{\bm{\Sigma}} \bA)^{-1/2} - \bI_d \|
&= \| (\bA \wt{\bm{\Sigma}} \bA)^{-1} \bA \bm{\Sigma} \bA  - \bI_d \|\\
&= \| \bA^{-1} \wt{\bm{\Sigma}}^{-1} \bm{\Sigma} \bA  - \bI_d \|\\
&= \| \wt{\bm{\Sigma}}^{-1} \bm{\Sigma} \bA \bA^{-1} - \bI_d \|\\
&= \| \wt{\bm{\Sigma}}^{-1} \bm{\Sigma} - \bI_d \|\\
&= \| \wt{\bm{\Sigma}}^{-1/2} \bm{\Sigma} \wt{\bm{\Sigma}}^{-1/2} - \bI_d \|
\end{align*}
\end{proof}

\probabilisticpartitioningconstruction*
\begin{proof}
By definition, we have $|\bB_1|, \ldots, |\bB_w| = k$.
Let us define $\cE_{1,i,j}$ as the event that the cell $(i,j)$ of $\bT$ \emph{never} appears in any of the submatrices $\bT_{\bB_1}, \ldots, \bT_{\bB_w}$, and $\cE_{2,i,j}$ as the event that the cell $(i,j)$ of $\bT$ appears in strictly more than $b$ submatrices.
In the rest of this proof, our goal is to show that $\Pr[\cE_1]$ and $\Pr[\cE_2]$ are small, where $\cE_1 = \cup_{(i,j) \in [d] \times [d]} \cE_{1,i,j}$ and $\cE_2 = \cup_{(i,j) \in [d] \times [d]} \cE_{2,i,j}$.

Fix any two \emph{distinct} $i,j \in [d]$.
For $\ell \in [w]$, let us define $X^{i,j}_\ell$ as the indicator event that the cell $(i,j)$ in $\bT$ appears in the $\ell^{th}$ principal submatrix $\bT_{\bB_\ell}$ when $i,j \in \bB_\ell$.
By construction,
\[
\Pr[X^{i,j}_\ell = 1] =
\begin{cases}
\frac{\binom{d-2}{k-2}}{\binom{d}{k}} = \frac{k(k-1)}{d(d-1)} & \text{if $i \neq j$}\\
\frac{\binom{d-1}{k-1}}{\binom{d}{k}} = \frac{k}{d} & \text{if $i = j$}
\end{cases}
\]

To analyze $\cE_1$, we first consider $i, j \in [d]$ where $i \neq j$.
We see that
\[
\Pr[\cE_{1,i,j}]
= \prod_{\ell = 1}^w \Pr[X^{i,j}_\ell = 0]
= \left( 1 - \frac{k(k-1)}{d(d-1)} \right)^w
\leq \exp \left( - \frac{wk(k-1)}{d(d-1)} \right)
= \exp \left( - 10 \log d \right)
= \frac{1}{d^{10}}
\]
Meanwhile, when $i = j$,
\[
\Pr[\cE_{1,i,i}]
= \prod_{\ell = 1}^w \Pr[X^{i,i}_\ell = 0]
= \left( 1 - \frac{k}{d} \right)^w
\leq \exp \left( - \frac{wk}{d} \right)
\leq \exp \left( - 10 \log d \right)
= \frac{1}{d^{10}}
\]
Taking union bound over $(i,j) \in [d] \times [d]$, we get
\[
\Pr[\cE_{1}]
\leq \sum_{(i,j) \in [d] \times [d]} \Pr[\cE_{1,i,j}]
\leq \frac{d^2}{d^{10}}
= \frac{1}{d^8}
\]

To analyze $\cE_2$, let us first define $Z^{i,j} = \sum_{\ell = 1}^w X^{i,j}_\ell$ for any $i,j \in [d]$.
Since the $X^{i,j}_\ell$ variables are indicators, linearity of expectations tells us that
\[
\E[Z^{i,j}]
= \sum_{\ell = 1}^w \E[X^{i,j}_\ell]
=
\begin{cases}
\sum_{\ell = 1}^w \frac{k(k-1)}{d(d-1)} = \frac{wk(k-1)}{d(d-1)} & \text{if $i \neq j$}\\
\sum_{\ell = 1}^w \frac{k}{d} = \frac{wk}{d}& \text{if $i = j$}
\end{cases}
\]
For $i \neq j$, applying Chernoff bound yields
\begin{multline*}
\Pr[Z^{i,j} > (1 + 2) \cdot \E[Z^{i,j}]]
\leq \exp \left( - \frac{\E[Z^{i,j}] \cdot 2^2}{2 + 2} \right)
\leq \exp \left( - \E[Z^{i,j}] \right)\\
= \exp \left( - \frac{wk(k-1)}{d(d-1)} \right)
= \exp \left( - 10 \log d \right)
= \frac{1}{d^{10}}
\end{multline*}
Meanwhile, when $i = j$,
\[
\Pr[Z^{i,i} > (1 + 2) \cdot \E[Z^{i,i}]]
\leq \exp \left( - \frac{\E[Z^{i,i}] \cdot 2^2}{2 + 2} \right)
\leq \exp \left( - \E[Z^{i,i}] \right)
= \exp \left( - \frac{wk}{d} \right)
\leq \exp \left( - 10 \log d \right)
= \frac{1}{d^{10}}
\]

By defining
\[
b
= 3 \cdot \max_{i,j \in [d]} \E[Z^{i,j}]
= \frac{3wk}{d}
= \frac{30 (d-1) \log d}{(k-1)}
\;,
\]
we see that $\Pr[E_{2,i,j}] = \Pr[Z^{i,j} > b] \leq \Pr[Z^{i,j} > (1 + 2) \cdot \E[Z^{i,j}]] \leq \frac{1}{d^{10}}$ and $\Pr[E_{2,i,i}] = \Pr[Z^{i,j} > b] \leq \Pr[Z^{i,i} > (1 + 2) \cdot \E[Z^{i,i}]] \leq \frac{1}{d^{10}}$.
Therefore, taking union bound over $(i,j) \in [d] \times [d]$, we get
\[
\Pr[\cE_{2}]
\leq \sum_{(i,j) \in [d] \times [d]} \Pr[\cE_{2,i,j}]
\leq \frac{d^2}{d^{10}}
= \frac{1}{d^8}
\]

In conclusion, this construction satisfy all 3 conditions of \cref{def:partitioning-scheme} with high probability in $d$.
\end{proof}

\subsection{Guarantees of \textsc{VectorizedApproxL1}}
\label{sec:appendix-vectorizedapproxl1-guarantees}

Here, we show that the guarantees of the \textsc{VectorizedApproxL1} algorithm (\cref{alg:vectorizedapproxl1}).

\guaranteesofvectorizedapproxlone*
\begin{proof}
We begin by stating some properties of $o_1, \ldots, o_w$.
Fix an arbitrary index $j \in \{1, \ldots, w\}$ and suppose $o_j$ is \emph{not} a $\Fail$, i.e.\ the tolerant tester of \cref{lem:tolerant-covariance-tester} outputs $\Accept$ for some $i^* \in \{1, 2, \ldots, \lceil \log_2 \zeta/\alpha \rceil\}$.
Note that \textsc{VectorizedApproxL1} sets $o_j = \ell_{i^*}$ and the tester outputs $\Reject$ for all smaller indices $i \in \{1, \ldots, i^* - 1\}$.
Since the tester outputs $\Accept$ for $i^*$, we have that $\| \bm{\Sigma}_{\bB_j} - \bI_d \|_F \leq 2 \ell_{i^*} = 2 o_j$.
Meanwhile, if $i^* > 1$, then $\| \bm{\Sigma}_{\bB_j} - \bI_d \|_F > \ell_{i^* - 1} = \ell_{i^*}/2 = o_j/2$ since the tester outputs $\Reject$ for $i^* - 1$.
Thus, we see that
\begin{itemize}
    \item When $o_j$ is not $\Fail$, we have $\| \bm{\Sigma}_{\bB_j} - \bI_d \|_F \leq 2 o_j$.
    \item When $\| \bm{\Sigma}_{\bB_j} - \bI_d \|_F \leq 2 \alpha$, we have $i^* = 1$ and $o_j = \ell_1 = \alpha$.
    \item When $\| \bm{\Sigma}_{\bB_j} - \bI_d \|_F > 2 \alpha = 2 \ell_1$, we have $i^* > 1$ and so $o_j < 2 \| \bm{\Sigma}_{\bB_j} - \bI_d \|_F$.
\end{itemize}

\paragraph{Success probability.}
Fix an arbitrary index $i \in \{1, 2, \ldots, \lceil \log_2 \zeta/\alpha \rceil\}$ with $\ell_i = 2^{i-1} \alpha$, where $\ell_i \leq \ell_1 = \alpha$ for any $i$.
We invoke the tolerant tester with $\eps_2 = 2 \ell_i = 2 \eps_1$, so the $i^{th}$ invocation uses at most $n'_{k, \eps} \cdot r_{\delta}$ i.i.d.\ samples to succeed with probability at least $1 - \delta$; see \cref{def:m-prime-d-alpha-delta} and \cref{alg:zero-mean-covariance-testing}.
So, with $m(k, \alpha, \delta')$ samples, \emph{any} call to the tolerant tester succeeds with probability at least $1 - \delta'$, where $\delta' = \frac{\delta}{w \cdot \lceil \log_2 \zeta/\alpha \rceil}$.
By construction, there will be at most $w \cdot \lceil \log_2 \zeta/\alpha \rceil$ calls to the tolerant tester.
Therefore, by union bound, \emph{all} calls to the tolerant tester jointly succeed with probability at least $1 - \delta$.

\paragraph{Property 1.}
When \textsc{VectorizedApproxL1} outputs $\Fail$, there exists a $\Fail$ amongst $\{o_1, \ldots, o_w\}$.
For any fixed index $j \in \{1, \ldots, w\}$, this can only happen when all calls to the tolerant tester outputs $\Reject$.
This means that $\| \bm{\Sigma}_{\bB_j} - \bI_d \|_F > \eps_1 = \ell_i = 2^{i-1} \cdot \alpha$ for all $i \in \{1, 2, \ldots, \lceil \log_2 \zeta/\alpha \rceil\}$.
In particular, this means that $\| \bm{\Sigma}_{\bB_j} - \bI_d \|_F > \zeta/2$.

\paragraph{Property 2.}
When \textsc{VectorizedApproxL1} outputs $\lambda = 2 \sum_{j=1}^w \sqrt{|\bB_j|} \cdot o_j \in \R$, we can lower bound $\lambda$ as follows:
\begin{align*}
\lambda
&= 2 \sum_{j=1}^w \sqrt{|\bB_j|} \cdot o_j\\
&\geq 2 \sum_{j=1}^w \sqrt{|\bB_j|} \cdot \frac{\| \bm{\Sigma}_{\bB_j} - \bI_d \|_F}{2} && \tag{since $\| \bm{\Sigma}_{\bB_j} - \bI_d \|_F \leq 2 o_j$}\\
&= \sum_{j=1}^w \sqrt{|\bB_j| \cdot \| \vec(\bm{\Sigma}_{\bB_j} - \bI_d) \|_2^2} && \tag{since $\| \bm{\Sigma}_{\bB_j} - \bI_d \|_F^2 = \| \vec(\bm{\Sigma}_{\bB_j} - \bI_d) \|_2^2$}\\
&\geq \sum_{j=1}^w \| \vec(\bm{\Sigma}_{\bB_j} - \bI_d) \|_1 && \tag{since $\| \vec(\bm{\Sigma}_{\bB_j} - \bI_d) \|_1^2 \leq |\bB_j| \cdot \| \vec(\bm{\Sigma}_{\bB_j} - \bI_d) \|_2^2$}\\
&\geq \| \vec(\bm{\Sigma} - \bI_d) \|_1 && \tag{Since each cell in $\bm{\Sigma}$ appears at least $a=1$ times across all submatrices $\bm{\Sigma}_{\bB_1}, \ldots, \bm{\Sigma}_{\bB_w}$}
\end{align*}
That is, $\lambda \geq \| \vec(\bm{\Sigma} - \bI_d) \|_1$.
Meanwhile, we can also upper bound $\lambda$ as follows:
\begin{align*}
\lambda
&= 2 \sum_{j=1}^w \sqrt{|\bB_j|} \cdot o_j\\
&\leq 2 \sqrt{k} \cdot \sum_{j=1}^w o_j && \tag{since $|\bB_j| \leq k$}\\
&= 2 \sqrt{k} \cdot \left( \sum_{\substack{j=1\\ \| \bm{\Sigma}_{\bB_j} - \bI_d \|_F \leq 2 \alpha}}^w o_j + \sum_{\substack{j=1\\ \| \bm{\Sigma}_{\bB_j} - \bI_d \|_F > 2 \alpha}}^w o_j \right) && \tag{partitioning based on $\| \bm{\Sigma}_{\bB_j} - \bI_d \|_F$ versus $2 \alpha$}\\
&= 2 \sqrt{k} \cdot \left( \sum_{\substack{j=1\\ \| \bm{\Sigma}_{\bB_j} - \bI_d \|_F \leq 2 \alpha}}^w \alpha + \sum_{\substack{j=1\\ \| \bm{\Sigma}_{\bB_j} - \bI_d \|_F > 2 \alpha}}^w o_j \right) && \tag{since $\| \bm{\Sigma}_{\bB_j} - \bI_d \|_F \leq 2 \alpha$ implies $o_j = \alpha$}\\
&\leq 2 \sqrt{k} \cdot \left( \sum_{\substack{j=1\\ \| \bm{\Sigma}_{\bB_j} - \bI_d \|_F \leq 2 \alpha}}^w \alpha + 2 \sum_{\substack{j=1\\ \| \bm{\Sigma}_{\bB_j} - \bI_d \|_F^2 \leq 2 \alpha}}^w \| \bm{\Sigma}_{\bB_j} - \bI_d \|_F \right) && \tag{since $\| \bm{\Sigma}_{\bB_j} - \bI_d \|_F > 2 \alpha$ implies $o_j \leq 2 \| \bm{\Sigma}_{\bB_j} - \bI_d \|_F$}\\
&= 2 \sqrt{k} \cdot \left( \sum_{\substack{j=1\\ \| \bm{\Sigma}_{\bB_j} - \bI_d \|_F \leq 2 \alpha}}^w \alpha + 2 \sum_{\substack{j=1\\ \| \bm{\Sigma}_{\bB_j} - \bI_d \|_F \leq 2 \alpha}}^w \| \vec(\bm{\Sigma}_{\bB_j} - \bI_d) \|_2 \right) && \tag{since $\| \bm{\Sigma}_{\bB_j} - \bI_d \|_F^2 = \| \vec(\bm{\Sigma}_{\bB_j} - \bI_d) \|_2^2$}\\
&\leq 2 \sqrt{k} \cdot \left( \sum_{\substack{j=1\\ \| \bm{\Sigma}_{\bB_j} - \bI_d \|_F \leq 2 \alpha}}^w \alpha + 2 \sum_{\substack{j=1\\ \| \bm{\Sigma}_{\bB_j} - \bI_d \|_F \leq 2 \alpha}}^w \| \vec(\bm{\Sigma}_{\bB_j} - \bI_d) \|_1 \right) && \tag{since $\| \vec(\bm{\Sigma}_{\bB_j} - \bI_d) \|_2 \leq \| \vec(\bm{\Sigma}_{\bB_j} - \bI_d) \|_1$}\\
&\leq 2 \sqrt{k} \cdot \left( w \alpha + 2 \sum_{\substack{j=1\\ \| \bm{\Sigma}_{\bB_j} - \bI_d \|_F^2 \leq 2 \alpha}}^w \| \vec(\bm{\Sigma}_{\bB_j} - \bI_d) \|_1 \right) && \tag{since $|\{j \in [w]: \| \bm{\Sigma}_{\bB_j} - \bI_d \|_F \leq 2 \alpha \}| \leq w$}\\
&\leq 2 \sqrt{k} \cdot \left( w \alpha + 2 \| \vec(\bm{\Sigma} - \bI_d) \|_1 \right) && \tag{since $\sum\limits_{\substack{j=1\\ \| \bm{\Sigma}_{\bB_j} - \bI_d \|_F \leq 2 \alpha}}^w \| \vec(\bm{\Sigma}_{\bB_j} - \bI_d) \|_1 \leq \sum_{j=1}^w \| \vec(\bm{\Sigma}_{\bB_j} - \bI_d) \|_1 = \| \vec(\bm{\Sigma} - \bI_d) \|_1$}
\end{align*}
That is, $\lambda \leq 2 \sqrt{k} \cdot \left( w \alpha + 2 \| \vec(\bm{\Sigma} - \bI_d) \|_1 \right)$, where $w = \frac{10 d(d-1) \log d}{k(k-1)}$.
The property follows by putting together both bounds.
\end{proof}

\subsection{Polynomial running time of \texorpdfstring{\cref{eq:cov-estimation-program}}{Eq. (6)}}
\label{sec:appendix-cov-estimation-program}

In this section, we show that \cref{eq:cov-estimation-program} in \cref{lem:SDP-given-l1-upper-bound} can be reformulated as a semidefinite program (SDP) that is polynomial time solvable.
Recall that we are given $n$ samples $\by_1, \ldots, \by_n \sim N(\bm{0}, \bm{\Sigma})$ under the assumption that $\| \vec(\bm{\Sigma} - \bI_d) \|_1 \leq r$ for some $r > 0$, and \cref{eq:cov-estimation-program} was defined as follows:
\[
\wh{\bm{\Sigma}} = \argmin_{\substack{\text{$\bA \in \R^{d \times d}$ is p.s.d.}\\ \| \vec(\bA - \bI_d) \|_1 \leq r\\ \lambda_{\min}(\bA) \geq 1}} \sum_{i=1}^n \| \bA - \by_i \by_i^\top \|_F^2
\]

To convert our optimization problem to the standard SDP form, we ``blow up'' the problem dimension into some integer $n' \in \poly(d)$.
Let $m$ be the number of constraints and $n'$ be the problem dimension.
For symmetric matrices $\bC, \bD_1, \ldots, \bD_m \in \R^{n' \times n'}$ and values $b_1, \ldots, b_m \in \R$, the standard form of a SDP is written as follows:
\begin{equation}
\label{eq:SDP-standard-form}
\begin{array}{crlll}
\min\limits_{\bX \in \R^{n' \times n'}} & \langle \bC, \bX \rangle\\
\text{subject to}
& \langle \bD_1, \bX \rangle & = b_1\\
& & \vdots\\
& \langle \bD_m, \bX \rangle & = b_m\\
& \bX & \succeq 0
\end{array}
\end{equation}
where the inner product between two matrices $\bA, \bB \in \R^{n' \times n'}$ is written as
\[
\langle \bA, \bB \rangle = \sum_{i=1}^{n'} \sum_{j=1}^{n'} \bA_{i,j} \bB_{i,j}
\]
For further expositions about SDPs, we refer readers to \cite{vandenberghe1996semidefinite,boyd2004convex,freund2004introduction,gartner2012approximation}.
In this section, we simply rely on the following known result to argue that our optimization problem will be polynomial time (in terms of $n$, $d$, and $r$) after showing how to frame \cref{eq:cov-estimation-program} in the standard SDP form.

\begin{theorem}[Implied by \cite{huang2022solving}]
\label{thm:sdp-solving}
Consider an SDP instance of the form \cref{eq:SDP-standard-form}.
Suppose it has an optimal solution $\bX^* \in \R^{n' \times n'}$ and any feasible solution $\bX \in \R^{n' \times n'}$ satisfies $\|\bX\|_2 \leq R$ for some $R > 0$. 
Then, there is an algorithm that produces $\wh{\bX}$ in $\cO(\poly(n, d, \log(1/\eps)))$ time such that $\langle \bC, \wh{\bX} \rangle \leq \langle \bC, \bX^* \rangle + \eps R \cdot \| \bC \|_2$.
\end{theorem}

\begin{remark}
Apart from notational changes, Theorem 8.1 of \cite{huang2022solving} actually deals with the maximization problem but here we transform it to our minimization setting.
They also guarantee additional bounds on the constraints with respect to $\wh{\bX}$, which we do not use.
\end{remark}

In the following formulation, for any indices $i$ and $j$, we define $\delta_{i,j} \in \{0,1\}$ as the indicator indicating whether $i = j$.
This will be useful for representation of the identity matrix.

\subsubsection{Re-expressing the objective function}

Observe that for any $i \in [n]$, we have
\begin{align*}
\| \bA - \by_i \by_i^\top \|_F^2
&= \Tr \left( (\bA - \by_i \by_i^\top)^\top (\bA - \by_i \by_i^\top) \right)\\
&= \Tr \left( \bA^\top \bA \right) - 2 \Tr \left( \by_i \by_i^\top \bA \right) + \Tr \left( \by_i \by_i^\top \by_i \by_i^\top \right)
\end{align*}
Since $\by_1, \ldots, \by_n \in \R^d$ are constants with respect to the optimization problem, we can ignore the $\Tr \left( \by_i \by_i^\top \by_i \by_i^\top \right)$ term and instead minimize $n \Tr \left( \bA^\top \bA \right) - 2 \sum_{i=1}^n \Tr \left( \by_i \by_i^\top \bA \right)$.
As $\bA^\top \bA$ is a quadratic expression, let us define an auxiliary matrix $\bB \in \R^{d \times d}$ which we will later enforce $\Tr(\bB) \geq \Tr(\bA^T \bA)$.
Defining a symmetric matrix $\bY = \sum_{i=1}^n \by_i \by_i^\top \in \R^{d \times d}$, the minimization objective becomes
\begin{equation}
\label{eq:sdp-objective}
n \Tr \left( \bB \right) - 2 \Tr \left( \bY \bA \right) = n \bB_{1,1} + \ldots + n \bB_{d,d} - 2 \langle \bY, \bA \rangle
\end{equation}

\subsubsection{Defining the variable matrix \texorpdfstring{$\bX$}{X}}

Let $n' = 2d^2 + 3d + 2$ and let us define the SDP variable matrix $\bX \in \R^{n' \times n'}$ as follows:
\[
\bX =
\begin{bmatrix}
\bB & \bA^\top\\
\bA & \bI_d\\
&& \bA - \bI_d\\
&&& \bU\\
&&&& \bS\\
&&&&& s_{\bU}\\
&&&&&& s_{\bB}\\
\end{bmatrix}
\in \R^{n' \times n'}
\]
where the empty parts of $\bX$ are zero matrices of appropriate sizes, $\bB \in \R^{d \times d}$ is an auxiliary matrix aiming to capture $\bA^\top \bA$, and $\bU$ and $\bS$ are diagonal matrices of size $d^2$:
\begin{align*}
\bU &= \diag(u_{1,1}, u_{1,2}, \ldots, u_{1,d}, \ldots, u_{d,1}, \ldots, u_{d,d}) \in \R^{d^2 \times d^2}\\
\bS &= \diag(s_{1,1}, s_{1,2}, \ldots, s_{1,d}, \ldots, s_{d,1}, \ldots, s_{d,d}) \in \R^{d^2 \times d^2}
\end{align*}
For convenience, we define
\[
\bM =
\begin{bmatrix}
\bB & \bA^\top\\
\bA & \bI_d\\
\end{bmatrix}
\in \R^{2d \times 2d}
\]
so we can write
\begin{equation}
\label{eq:sdp-variable-matrix}
\bX =
\begin{bmatrix}
\bM\\
& \bA - \bI_d\\
&& \bU\\
&&& \bS\\
&&&& s_{\bU}\\
&&&&& s_{\bB}
\end{bmatrix}
\in \R^{n' \times n'}
\end{equation}

In the following subsections, we explain how to ensure that submatrices in $\bX$ model the desired notions and constraints on $\bA$, $\bB$, and so on.
For instance, we will use $\bU$ to enforce $\| \vec(\bA - \bI_d) \|_1 \leq r$ in an element-wise fashion and use $\bS$ and $s_{\bU}$ for slack variables to transform inequality constraints to equality ones.
The slack variable $s_{\bB}$ is used for upper bounding the norm of $\bB$ later, so that we can argue that the feasible region is bounded.

\subsubsection{Defining the cost matrix \texorpdfstring{$\bC$}{C}}

To capture the objective function \cref{eq:sdp-objective}, let us define a symmetric cost matrix $\bC \in \R^{n' \times n'}$ as follows:
\begin{equation}
\label{eq:sdp-cost-matrix}
\bC =
\begin{bmatrix}
\diag(n, \ldots, n) & -\bY\\
-\bY & \mathbf{0}_{d \times d}\\
&& \mathbf{0}_{(2d^2 + d + 2) \times (2d^2 + d + 2)}
\end{bmatrix}
\in \R^{n' \times n'}
\end{equation}
One can check that $\langle \bC, \bX \rangle = n \bB_{1,1} + \ldots + n \bB_{d,d} - 2 \langle \bY, \bA \rangle$.

\subsubsection{Enforcing zeroes, ones, and linking \texorpdfstring{$\bA$}{A} entries with \texorpdfstring{$\bA - \bI_d$}{A-I}}

To enforce that the empty parts of $\bX$ always solves to zeroes, we can define a symmetric constraint matrix $\bD^{zero}_{i,j} \in \R^{n' \times n'}$ such that
\[
(\bD^{zero}_{i,j})_{i',j'} =
\begin{cases}
1 & \text{if $i' = i$ and $j' = j$}\\
0 & \text{otherwise}
\end{cases}
\]
and $b^{zero}_{i,j} = 0$.
Then, $\langle \bD^{zero}_{i,j}, \bX \rangle = b^{zero}_{i,j}$ resolves to $\bX_{i,j} = \langle \bD^{zero}_{i,j}, \bX \rangle = b^{zero}_{i,j} = 0$.
We can similarly enforce that the appropriate part of $\bX$ in $\bM$ resolves to $\bI_d$.

Now, to ensure that the $\bA$ submatrices within $\bM$ are appropriately linked to $\bA - \bI_d$, we can define a symmetric constraint matrix $\bD^{\bA}_{i,j} \in \R^{n' \times n'}$ such that
\[
\bD^{\bA}_{i,j} =
\begin{bmatrix}
\mathbf{0}_{d \times d} & \ast\\
\ast & \mathbf{0}_{d \times d}\\
&& \dag\\
&&& \mathbf{0}_{d^2 \times d^2}\\
&&&& \mathbf{0}_{d^2 \times d^2}\\
&&&&& 0\\
&&&&&& 0\\
\end{bmatrix}
\in \R^{n' \times n'}
\]
and $b^{\bA}_{i,j} = 0$, where $\ast$ contains $\frac{1}{4}$ at the $(i,j)$-th and $(j,i)$-th entries and $\dag$ contains $\delta_{i,j} - \frac{1}{2}$ at the $(i,j)$-th and $(j,i)$-th entries, with 0 everywhere else; if $i = j$, we double the value.
So, $\langle \bD^{\bA}_{i,j}, \bX \rangle = b^{\bA}_{i,j}$ would enforce that the $(i,j)$-th and $(j,i)$-th entries between the $\bA$ submatrices within $\bM$ and those in $\bA - \bI_d$ are appropriately linked.

\subsubsection{Modeling the \texorpdfstring{$\ell_1$}{L1} constraint}

To encode $\| \vec(\bA - \bI_d) \|_1 \leq r$ in SDP form, let us define auxiliary variables $\{u_{i,j}\}_{i,j \in [d]}$ and define the linear constraints:
\begin{itemize}
    \item $- A_{i,j} -u_{i,j} \leq - \delta_{i,j}$, for all $i,j \in [d]$
    \item $A_{i,j} - u_{i,j} \leq \delta_{i,j}$, for all $i,j \in [d]$
    \item $\sum_{i=1}^d \sum_{j=1}^d u_{i,j} \leq r$
\end{itemize}
The first two constraints effectively encode $|A_{i,j} - \delta_{i,j}| \leq u_{i,j}$ and so the third constraint captures $\| \vec(\bA - \bI_d) \|_1 \leq r$ as desired.
To convert the inequality constraint to an equality one, we use the slack variables $\{s_{i,j}\}_{i,j \in [d]}$ in $\bS$.
For instance, we can define symmetric constraint matrices $\bD^{+}_{i,j} \in \R^{n' \times n'}$, $\bD^{-}_{i,j} \in \R^{n' \times n'}$, and $\bD^{r}_{i,j} \in \R^{n' \times n'}$ with $b^{+}_{i,j} = b^{-}_{i,j} = 0$ and $b^{r} = r$ as follows:
\[
\bD^{+}_{i,j} =
\begin{bmatrix}
\mathbf{0}_{d \times d} & \ast\\
\ast & \mathbf{0}_{d \times d}\\
&& \mathbf{0}_{d \times d}\\
&&& \dag\\
&&&& \ddag\\
&&&&& 0\\
&&&&&& 0\\
\end{bmatrix}
\qquad
\bD^{-}_{i,j} =
\begin{bmatrix}
\mathbf{0}_{d \times d} & -\ast\\
-\ast & \mathbf{0}_{d \times d}\\
&& \mathbf{0}_{d \times d}\\
&&& \dag\\
&&&& \ddag\\
&&&&& 0\\
&&&&&& 0\\
\end{bmatrix}
\]
\[
\bD^{r}_{i,j} =
\begin{bmatrix}
\mathbf{0}_{2d \times 2d}\\
& \mathbf{0}_{d \times d}\\
&& \mathbf{1}_{d^2 \times d^2}\\
&&& \mathbf{0}_{d^2 \times d^2}\\
&&&& 1\\
&&&&& 0
\end{bmatrix}
\]
where $\ast$ contains $\frac{\delta_{i,j} - 1}{4}$ at the $(i,j)$-th and $(j,i)$-th entries, $\dag$ contains $-\frac{1}{2}$ at the $(i,j)$-th and $(j,i)$-th entries, and $\ddag$ contains $\frac{1}{2}$ at the $(i,j)$-th and $(j,i)$-th entries, with 0 everywhere else; if $i = j$, we double the value.
So, $\langle \bD^{+}_{i,j}, \bX \rangle = b^{+}_{i,j}$ models $\delta_{i,j} - A_{i,j} - u_{i,j} + s_{i,j} = 0$, $\langle \bD^{-}_{i,j}, \bX \rangle = b^{-}_{i,j}$ models $A_{i,j} - \delta_{i,j} - u_{i,j} + s_{i,j} = 0$, and $\langle \bD^{r}_{i,j}, \bX \rangle = b^{r}_{i,j}$ models $s_{\bS} + \sum_{i=1} \sum_{j=1} u_{i,j} = r$.

\subsubsection{Positive semidefinite constraints}

By known properties of the (generalized) Schur complement \cite[Section 1.4 and Section 1.6]{zhang2005schur}, it is known that $\bX \succeq \mathbf{0}$ if and only if the following properties hold simultaneously:
\begin{enumerate}
    \item $\bM \succeq \mathbf{0}$
    \item $\bA - \bI_d \succeq \mathbf{0} \iff \bA \succeq \bI_d \iff \lambda_{\min}(\bA) \geq 1$, which also implies that $\bA$ is psd
    \item $\bU \succeq \mathbf{0} \iff u_{1,1}, u_{1,2}, \ldots, u_{1,d}, \ldots, u_{d,1}, \ldots, u_{d,d} \geq 0$
    \item $\bS \succeq \mathbf{0} \iff s_{1,1}, s_{1,2}, \ldots, s_{1,d}, \ldots, s_{d,1}, \ldots, s_{d,d} \geq 0$
    \item $s_{\bU} \geq 0$
    \item $s_{\bB} \geq 0$
\end{enumerate}

For the first property, since $\bI_d \succ \mathbf{0}$, Schur complement tells us that
$
\bM =
\begin{bmatrix}
\bB & \bA^\top\\
\bA & \bI_d
\end{bmatrix}
\succeq 0
$
if and only if $\bB \succeq \bA^\top \bA$.
Observe that $\bB \succeq \bA^\top \bA$ implies $\Tr(\bB) \geq \Tr(\bA^\top \bA)$, which aligns with our intention of modeling $\bA^\top \bA$ by $\bB$.
Note that the objective function is $n \Tr(\bB) - 2 \Tr(\bY \bA)$ and we have that $\Tr(\bB) \geq \Tr(\bA^\top \bA)$ for all feasible matrices $\bB$.
Thus, for any pair $(\bA^*, \bB^*)$ that minimizes of the objective function, it has to be that $\Tr(\bB^*) = \Tr((\bA^*)^\top \bA^*)$, since otherwise, the pair $(\bA^*, \bB^{**} = (\bA^*)^\top \bA^*)$ would have a smaller value.

\subsubsection{Enforcing an upper bound on \texorpdfstring{$\| \bB \|_2$}{spectral norm of B}}

To apply \cref{thm:sdp-solving}, we need to argue that the feasible region of our SDP is bounded and non-empty, so that $\| \bX \|_2$ is upper bounded.
To do so, we need to enforce an upper bound on $\| \bB \|_2$.

Since $\| \vec(\bA - \bI_d) \|_1 \leq r$, by triangle inequality and standard norm inequalities, we see that
\begin{multline}
\label{eq:bounding-A-by-r}
\| \bA \|_2
\leq \| \bA - \bI_d \|_2 + \| \bI_d \|_2
\leq \| \bA - \bI_d \|_F + \| \bI_d \|_2\\
= \| \vec(\bA - \bI_d) \|_2 + d
\leq \| \vec(\bA - \bI_d) \|_1 + d
\leq r + d
\end{multline}
As $\bB$ is supposed to model $\bA^T \bA$ and is constrained only by $\bB \succeq \bA^T \bA$, it is feasible to enforce $\Tr(\bB) \leq \| \bB \|_F^2 \leq d \cdot (r + d)^4$ because
\[
\| \bA^T \bA \|_F^2
\leq d \cdot \| \bA^T \bA \|_2^2
= d \cdot \| \bA \|_2^4
\leq d \cdot (r + d)^4
\]
To this end, let us define a symmetric constraint matrix $\bD^{\bB}_{i,j} \in \R^{n' \times n'}$ such that 
\[
\bD^{\bB} =
\begin{bmatrix}
\bI_{d}\\
& \mathbf{0}_{(2d^2 + 2d + 1) \times (2d^2 + 2d + 1)}\\
&& 1
\end{bmatrix}
\in \R^{n' \times n'}
\]
and $b^{\bB} = d \cdot (r + d)^4$.
Then, $\langle \bD^{\bB}, \bX \rangle = b^{\bB}$ resolves to $\Tr(\bB) + s_{\bB} = \langle \bD^{\bB}, \bX \rangle = b^{\bB} = d \cdot (r + d)^4$.
In other words, since the slack variable $s_{\bB}$ is non-negative, i.e.\ $s_{\bB} \geq 0$, we have
\begin{equation}
\label{sec:B-norm-bound}
\| \bB \|_2
\leq \Tr(\bB)
\leq \| \bB \|_F^2
\leq d \cdot (r + d)^4
\end{equation}

\subsubsection{Bounding \texorpdfstring{$\| \bC \|_2$}{spectral norm of C} and \texorpdfstring{$\| \bX \|_2$}{spectral norm of X}}

Recalling the definition of $\bC$ in \cref{eq:sdp-cost-matrix}, we see that
\[
\| \bC \|_2
\leq
\left\|
\begin{bmatrix}
\diag(n, \ldots, n) & -\bY\\
-\bY & \mathbf{0}_{d \times d}
\end{bmatrix}
\right\|_2
\leq n + \| \bY \|_2
\]
Meanwhile, we know from \cref{lem:known-properties-of-empirical-covariance} that
\[
\| \bY \|_2
\leq \| \bm{\Sigma} \|_2 \cdot \left( 1 + \cO\left( \sqrt{\frac{d + \log 1/\delta}{n}} \right) \right)
\]
with probability at least $1 - \delta$.

Recall from \cref{alg:testandoptimizecovariance} that when we solve the optimization problem of \cref{eq:cov-estimation-program}, we have that $\| \vec(\bm{\Sigma} - \bI) \|_1 \leq r$.
So, by a similar chain of arguments as \cref{eq:bounding-A-by-r}, we see that
\[
\| \bm{\Sigma} \|_2
\leq \| \bm{\Sigma} - \bI_d \|_2 + \| \bI_d \|_2
\leq \| \bm{\Sigma} - \bI_d \|_F + \| \bI_d \|_2
= \| \vec(\bm{\Sigma} - \bI_d) \|_2 + d
\leq \| \vec(\bm{\Sigma} - \bI_d) \|_1 + d
= r + d
\]
Therefore,
\[
\| \bC \|_2
\leq n + \| \bm{\Sigma} \|_2 \cdot \left( 1 + \cO\left( \sqrt{\frac{d + \log 1/\delta}{n}} \right) \right)
\leq n + (r + d) \cdot \left( 1 + \cO\left( \sqrt{\frac{d + \log 1/\delta}{n}} \right) \right)
\in \poly(n,d,r)
\]

Meanwhile, recalling definition of $\bX$ from \cref{eq:sdp-variable-matrix}, we see that for \emph{any} feasible solution $\bX$,
\[
\| \bX \|_2
\leq \max \left\{ \| \bM \|_2, \| \bA - \bI_d \|_2, \| \bU \|_2, \| \bS \|_2, s_{\bU}, s_{\bB} \right\}
\]
By \cref{sec:B-norm-bound}, we have that $\| \bB \|_2 \leq \sqrt{d} \cdot (r + d)^2$.
So,
\[
\| \bM \|_2
\leq \| \bB \|_2 + \| \bA \|_2 + 1
\leq d \cdot (r + d)^4 + r + d + 1
\in \poly(d,r)
\]
Also, all the remaining terms are in $\poly(r,d)$ since $\| \vec(\bA - \bI_d) \|_1 \leq r$.
Therefore, $\| \bX \|_2 \in \poly(d, r)$ with probability $1-\delta$.
So, $\| \bX \|_2 \leq R$ for some $R \in \poly(d, r)$.

\subsubsection{Putting together}

Suppose we aim for an additive error of $\eps' > 0$ in \cref{eq:sdp-additive-error} when we solve \cref{eq:cov-estimation-program}.
From above, we have that $\| \bC \|_2, R \in \poly(n, d, r)$.
Let us define $\eps = \frac{\eps'}{R \cdot \| \bC \|_2}$ in \cref{thm:sdp-solving}.
Then, the algorithm of \cref{thm:sdp-solving} produces $\wh{\bX} \in \R^{n' \times n'}$ in $\poly(n, d, \log(1/\eps)) \subseteq \poly(n, d, \log(\frac{R \cdot \| \bC \|_2}{\eps'})) \subseteq \poly(n, d, r, \log(1/\eps'))$ time such that $\langle \bC, \wh{\bX} \rangle \leq \langle \bC, \bX^* \rangle + \eps R \cdot \| \bC \|_2 = \langle \bC, \bX^* \rangle + \eps'$ as desired.

%% file: appendix-python-code.tex
\section{Python code for reproducing experiments}
\label{sec:appendix-python-code}

\lstdefinestyle{mystyle}{
    basicstyle=\ttfamily\footnotesize,
    backgroundcolor=\color{lightgray!20},
    keywordstyle=\color{blue}\bfseries,
    commentstyle=\color{green!50!black},
    stringstyle=\color{red},
    showstringspaces=false,
    numbers=left,
    numberstyle=\tiny\color{gray},
    breaklines=true,
    frame=single
}

\lstset{style=mystyle}

\begin{lstlisting}[language=Python, caption={Python script for experiments}, label={lst:python_script}]
import numpy as np
import matplotlib.pyplot as plt
import pickle
import sys

from sklearn import linear_model
from tqdm import tqdm
from typing import Tuple

def estimate(samples: np.ndarray) -> Tuple[np.ndarray, np.ndarray]:
    N, d = samples.shape
    X = np.concatenate([np.identity(d) for _ in range(N)])
    y = np.concatenate(samples)
    reg = linear_model.LassoLarsCV(cv=5)
    reg.fit(X, y)
    opt_est = reg.coef_
    emp_est = 1./N * sum(y_i for y_i in samples)
    return opt_est, emp_est

def run_experiments(
        rng: np.random.Generator,
        d: int,
        s: int,
        q: float,
        Nmin: int,
        Nmax: int,
        Nstep: int,
        Nrepeats: int,
        fname: str
    ) -> None:
    assert 0 <= s and s <= d

    # Generate random ground truth mu
    mu = [0 for _ in range(d)]
    for i in range(s):
        mu[i] = q/s * rng.choice([-1, 1])
    mu = np.array(mu)

    # Run
    N_vals = np.arange(Nmin, Nmax+1, Nstep)
    opt_err = [[] for _ in range(Nrepeats)]
    emp_err = [[] for _ in range(Nrepeats)]
    for run_idx in tqdm(range(Nrepeats)):
        samples = rng.multivariate_normal(mu, np.identity(d), size=Nmax)
        for N in tqdm(N_vals):
            opt_est, emp_est = estimate(samples[:N])
            opt_err[run_idx].append(np.linalg.norm(opt_est - mu, 2))
            emp_err[run_idx].append(np.linalg.norm(emp_est - mu, 2))

    # Save results
    results = [N_vals, opt_err, emp_err]
    with open("{0}.pkl".format(fname), 'wb') as file:
        pickle.dump(results, file)

    # Generate plot
    generate_plot(fname)

def generate_plot(fname: str) -> None:
    with open("{0}.pkl".format(fname), 'rb') as file:
        results = pickle.load(file)
    N_vals, opt_err, emp_err = results

    opt_mean = np.mean(opt_err, axis=0)
    opt_std = np.std(opt_err, axis=0)
    emp_mean = np.mean(emp_err, axis=0)
    emp_std = np.std(emp_err, axis=0)
    plt.plot(N_vals, opt_mean, label="TestAndOptimize", color='g')
    plt.plot(N_vals, emp_mean, label="Empirical", color='r')
    plt.fill_between(N_vals, opt_mean - opt_std, opt_mean + opt_std, color='g', alpha=0.5)
    plt.fill_between(N_vals, emp_mean - emp_std, emp_mean + emp_std, color='r', alpha=0.5)
    plt.xlabel("Number of samples")
    plt.ylabel(r"$\ell_2$ error")
    plt.legend()
    plt.savefig("{0}.png".format(fname), dpi=300, bbox_inches='tight')

if __name__ == "__main__":
    mode = int(sys.argv[1])
    d = int(sys.argv[2])
    s = int(sys.argv[3])
    q = float(sys.argv[4])
   
    seed = 314159
    rng = np.random.default_rng(seed)
    Nmin = 10
    Nmax = 300
    Nstep = 10
    Nrepeats = 10
    fname = "plot_d{0}_sparsity{1}_L1norm{2}_Nmax={3}_runs={4}".format(d, s, q, Nmax, Nrepeats)

    if mode == 0:
        run_experiments(rng, d, s, q, Nmin, Nmax, Nstep, Nrepeats, fname)
    elif mode == 1:
        generate_plot(fname)
    else:
        raise ValueError("Invalid mode. Use '0' for full run and '1' for just plotting.")

\end{lstlisting}

%% file: main.bbl
\newcommand{\etalchar}[1]{$^{#1}$}
\begin{thebibliography}{ABDH{\etalchar{+}}20}

\bibitem[ABDH{\etalchar{+}}20]{ashtiani2020gaussian}
Hassan Ashtiani, Shai Ben-David, Nicholas J.~A. Harvey, Christopher Liaw, Abbas Mehrabian, and Yaniv Plan.
\newblock Near-optimal sample complexity bounds for robust learning of gaussian mixtures via compression schemes.
\newblock {\em J. ACM}, 67(6), oct 2020.

\bibitem[ABG{\etalchar{+}}22]{agrawal2022learning}
Priyank Agrawal, Eric Balkanski, Vasilis Gkatzelis, Tingting Ou, and Xizhi Tan.
\newblock Learning-augmented mechanism design: Leveraging predictions for facility location.
\newblock In {\em Proceedings of the 23rd ACM Conference on Economics and Computation}, pages 497--528, 2022.

\bibitem[ADJ{\etalchar{+}}20]{angelopoulos2020online}
Spyros Angelopoulos, Christoph D{\"u}rr, Shendan Jin, Shahin Kamali, and Marc Renault.
\newblock {Online Computation with Untrusted Advice}.
\newblock In {\em 11th Innovations in Theoretical Computer Science Conference (ITCS 2020)}. Schloss Dagstuhl-Leibniz-Zentrum f{\"u}r Informatik, 2020.

\bibitem[AGKK20]{antoniadis2020secretary}
Antonios Antoniadis, Themis Gouleakis, Pieter Kleer, and Pavel Kolev.
\newblock Secretary and online matching problems with machine learned advice.
\newblock {\em Advances in Neural Information Processing Systems}, 33:7933--7944, 2020.

\bibitem[AJS22]{antoniadis2022novel}
Antonios Antoniadis, Peyman Jabbarzade, and Golnoosh Shahkarami.
\newblock {A Novel Prediction Setup for Online Speed-Scaling}.
\newblock In {\em 18th Scandinavian Symposium and Workshops on Algorithm Theory (SWAT 2022)}. Schloss Dagstuhl-Leibniz-Zentrum f{\"u}r Informatik, 2022.

\bibitem[BLMS{\etalchar{+}}22]{bernardiniuniversal}
Giulia Bernardini, Alexander Lindermayr, Alberto Marchetti-Spaccamela, Nicole Megow, Leen Stougie, and Michelle Sweering.
\newblock {A Universal Error Measure for Input Predictions Applied to Online Graph Problems}.
\newblock In {\em Advances in Neural Information Processing Systems}, 2022.

\bibitem[BMRS20]{bamas2020learning}
{\'E}tienne Bamas, Andreas Maggiori, Lars Rohwedder, and Ola Svensson.
\newblock {Learning Augmented Energy Minimization via Speed Scaling}.
\newblock {\em Advances in Neural Information Processing Systems}, 33:15350--15359, 2020.

\bibitem[BMS20]{bamas2020primal}
Etienne Bamas, Andreas Maggiori, and Ola Svensson.
\newblock {The Primal-Dual method for Learning Augmented Algorithms}.
\newblock {\em Advances in Neural Information Processing Systems}, 33:20083--20094, 2020.

\bibitem[BV04]{boyd2004convex}
Stephen Boyd and Lieven Vandenberghe.
\newblock {\em Convex optimization}.
\newblock Cambridge university press, 2004.

\bibitem[CGB23]{choo2023active}
Davin Choo, Themistoklis Gouleakis, and Arnab Bhattacharyya.
\newblock Active causal structure learning with advice.
\newblock In {\em International Conference on Machine Learning}, pages 5838--5867. PMLR, 2023.

\bibitem[CGLB24]{choo2024online}
Davin Choo, Themistoklis Gouleakis, Chun~Kai Ling, and Arnab Bhattacharyya.
\newblock Online bipartite matching with imperfect advice.
\newblock In {\em International Conference on Machine Learning}. PMLR, 2024.

\bibitem[CM13]{cai2013optimal}
T~Tony Cai and Zongming Ma.
\newblock Optimal hypothesis testing for high dimensional covariance matrices.
\newblock {\em Bernoulli}, 19(5B):2359--2388, 2013.

\bibitem[CSVZ22]{chen2022faster}
Justin Chen, Sandeep Silwal, Ali Vakilian, and Fred Zhang.
\newblock Faster fundamental graph algorithms via learned predictions.
\newblock In {\em International Conference on Machine Learning}, pages 3583--3602. PMLR, 2022.

\bibitem[Dia16]{diakonikolas2016learning}
Ilias Diakonikolas.
\newblock Learning structured distributions.
\newblock {\em Handbook of Big Data}, 267:10--1201, 2016.

\bibitem[DIL{\etalchar{+}}21]{dinitz2021faster}
Michael Dinitz, Sungjin Im, Thomas Lavastida, Benjamin Moseley, and Sergei Vassilvitskii.
\newblock Faster matchings via learned duals.
\newblock {\em Advances in neural information processing systems}, 34:10393--10406, 2021.

\bibitem[DKS17]{DiakonikolasKaneStewart2017}
Ilias Diakonikolas, Daniel~M. Kane, and Alistair Stewart.
\newblock Statistical query lower bounds for robust estimation of high-dimensional gaussians and gaussian mixtures.
\newblock In {\em 2017 IEEE 58th Annual Symposium on Foundations of Computer Science (FOCS)}, pages 73--84, 2017.

\bibitem[DL01]{devroye2001combinatorial}
Luc Devroye and G{\'a}bor Lugosi.
\newblock {\em {Combinatorial methods in density estimation}}.
\newblock Springer Science \& Business Media, 2001.

\bibitem[DLPLV21]{dutting2021secretaries}
Paul D{\"u}tting, Silvio Lattanzi, Renato Paes~Leme, and Sergei Vassilvitskii.
\newblock {Secretaries with Advice}.
\newblock In {\em Proceedings of the 22nd ACM Conference on Economics and Computation}, pages 409--429, 2021.

\bibitem[DMR18]{devroye2018total}
Luc Devroye, Abbas Mehrabian, and Tommy Reddad.
\newblock The total variation distance between high-dimensional gaussians with the same mean.
\newblock {\em arXiv preprint arXiv:1810.08693}, 2018.

\bibitem[FR13]{10.5555/2526243}
Simon Foucart and Holger Rauhut.
\newblock {\em A Mathematical Introduction to Compressive Sensing}.
\newblock Birkh\"{a}user Basel, 2013.

\bibitem[Fre04]{freund2004introduction}
Robert~M. Freund.
\newblock {Introduction to Semidefinite Programming (SDP)}, 2004.
\newblock MIT OpenCourseWare.

\bibitem[Gho21]{Ghosh2021}
Malay Ghosh.
\newblock Exponential tail bounds for chisquared random variables.
\newblock {\em Journal of Statistical Theory and Practice}, 15, 2021.

\bibitem[GKST22]{gkatzelis2022improved}
Vasilis Gkatzelis, Kostas Kollias, Alkmini Sgouritsa, and Xizhi Tan.
\newblock Improved price of anarchy via predictions.
\newblock In {\em Proceedings of the 23rd ACM Conference on Economics and Computation}, pages 529--557, 2022.

\bibitem[GLS23]{gouleakis2023learning}
Themis Gouleakis, Konstantinos Lakis, and Golnoosh Shahkarami.
\newblock {Learning-Augmented Algorithms for Online TSP on the Line}.
\newblock In {\em 37th AAAI Conference on Artificial Intelligence}. AAAI, 2023.

\bibitem[GM12]{gartner2012approximation}
Bernd G{\"a}rtner and Jiri Matousek.
\newblock {\em {Approximation Algorithms and Semidefinite Programming}}.
\newblock Springer Science \& Business Media, 2012.

\bibitem[GP19]{gollapudi2019online}
Sreenivas Gollapudi and Debmalya Panigrahi.
\newblock {Online Algorithms for Rent-or-Buy with Expert Advice}.
\newblock In {\em International Conference on Machine Learning}, pages 2319--2327. PMLR, 2019.

\bibitem[HJ12]{horn2012matrix}
Roger~A. Horn and Charles~R. Johnson.
\newblock {\em {Matrix Analysis}}.
\newblock Cambridge University Press, 2012.

\bibitem[HJS{\etalchar{+}}22]{huang2022solving}
Baihe Huang, Shunhua Jiang, Zhao Song, Runzhou Tao, and Ruizhe Zhang.
\newblock Solving sdp faster: A robust ipm framework and efficient implementation.
\newblock In {\em 2022 IEEE 63rd Annual Symposium on Foundations of Computer Science (FOCS)}, pages 233--244. IEEE, 2022.

\bibitem[HTW15]{hastie2015statistical}
Trevor Hastie, Robert Tibshirani, and Martin Wainwright.
\newblock Statistical learning with sparsity.
\newblock {\em Monographs on statistics and applied probability}, 143(143):8, 2015.

\bibitem[KBC{\etalchar{+}}18]{kraska2018case}
Tim Kraska, Alex Beutel, Ed~H. Chi, Jeffrey Dean, and Neoklis Polyzotis.
\newblock {The Case for Learned Index Structures}.
\newblock In {\em Proceedings of the 2018 international conference on management of data}, pages 489--504, 2018.

\bibitem[KLSU19]{kamath2019privately}
Gautam Kamath, Jerry Li, Vikrant Singhal, and Jonathan Ullman.
\newblock Privately learning high-dimensional distributions.
\newblock In {\em Conference on Learning Theory}, pages 1853--1902. PMLR, 2019.

\bibitem[KSV24]{klivans24testable}
Adam Klivans, Konstantinos Stavropoulos, and Arsen Vasilyan.
\newblock {Testable Learning with Distribution Shift}.
\newblock In {\em Conference on Learning Theory (COLT)}, pages 2887--2943. Proceedings of Machine Learning Research (PMLR), 2024.

\bibitem[LLMV20]{lattanzi2020online}
Silvio Lattanzi, Thomas Lavastida, Benjamin Moseley, and Sergei Vassilvitskii.
\newblock {Online Scheduling via Learned Weights}.
\newblock In {\em Proceedings of the Fourteenth Annual ACM-SIAM Symposium on Discrete Algorithms}, pages 1859--1877. SIAM, 2020.

\bibitem[Mit18]{mitzenmacher2018model}
Michael Mitzenmacher.
\newblock {A Model for Learned Bloom Filters, and Optimizing by Sandwiching}.
\newblock {\em Advances in Neural Information Processing Systems}, 31, 2018.

\bibitem[PSK18]{purohit2018improving}
Manish Purohit, Zoya Svitkina, and Ravi Kumar.
\newblock {Improving Online Algorithms via ML Predictions}.
\newblock {\em Advances in Neural Information Processing Systems}, 31, 2018.

\bibitem[RV23]{rubinfeld2023testing}
Ronitt Rubinfeld and Arsen Vasilyan.
\newblock Testing distributional assumptions of learning algorithms.
\newblock In {\em Proceedings of the 55th Annual ACM Symposium on Theory of Computing}, pages 1643--1656, 2023.

\bibitem[Tib96]{tibshirani1996regression}
Robert Tibshirani.
\newblock Regression shrinkage and selection via the lasso.
\newblock {\em Journal of the Royal Statistical Society Series B: Statistical Methodology}, 58(1):267--288, 1996.

\bibitem[Tib97]{tibshirani1997lasso}
Robert Tibshirani.
\newblock The lasso method for variable selection in the cox model.
\newblock {\em Statistics in medicine}, 16(4):385--395, 1997.

\bibitem[Vas24]{vasilyan2024enhancing}
Arsen Vasilyan.
\newblock {\em Enhancing Learning Algorithms via Sublinear-Time Methods}.
\newblock PhD thesis, Massachusetts Institute of Technology, 2024.

\bibitem[VB96]{vandenberghe1996semidefinite}
Lieven Vandenberghe and Stephen Boyd.
\newblock {Semidefinite Programming}.
\newblock {\em SIAM Review}, 38(1):49--95, 1996.

\bibitem[Ver10]{vershynin2010introduction}
Roman Vershynin.
\newblock {Introduction to the non-asymptotic analysis of random matrices}.
\newblock {\em arXiv preprint arXiv:1011.3027}, 2010.

\bibitem[Ver12]{vershynin2012lectures}
Roman Vershynin.
\newblock Lectures in geometric functional analysis, 2012.

\bibitem[Ver18]{vershynin2018high}
Roman Vershynin.
\newblock {\em High-dimensional probability: An introduction with applications in data science}, volume~47.
\newblock Cambridge university press, 2018.

\bibitem[WLW20]{wang2020online}
Shufan Wang, Jian Li, and Shiqiang Wang.
\newblock {Online Algorithms for Multi-shop Ski Rental with Machine Learned Advice}.
\newblock {\em Advances in Neural Information Processing Systems}, 33:8150--8160, 2020.

\bibitem[Zha05]{zhang2005schur}
Fuzhen Zhang.
\newblock {\em {The Schur Complement and Its Applications}}.
\newblock Springer, 2005.

\end{thebibliography}
